%% file: neurips_2026_final.tex
\documentclass{article}

\PassOptionsToPackage{numbers, sort, compress}{natbib}

\usepackage[main, final]{neurips_2026}

\usepackage[utf8]{inputenc} 
\usepackage[T1]{fontenc}    
\usepackage{hyperref}       
\usepackage{url}            
\usepackage{booktabs}       
\usepackage{amsfonts}       
\usepackage{amsthm}
\usepackage{mathtools}
\usepackage{bbm}
\usepackage{nicefrac}       
\usepackage{microtype}      
\usepackage{xcolor}         
\usepackage{graphicx}
\usepackage{amsmath}
\usepackage{algorithm}
\usepackage[noEnd=true, indLines=true]{algpseudocodex}
\usepackage{microtype}
\usepackage{graphicx}
\usepackage{subcaption}
\usepackage{wrapfig}
\usepackage{tikz}
\usetikzlibrary{positioning}
\usepackage{multirow}
\usepackage[capitalize]{cleveref}
\usepackage{titletoc}
\usepackage[table]{xcolor}
\usepackage{arydshln}
\usepackage{enumitem}
\graphicspath{{figures/}{figures/icons/}}
\newcommand{\icon}[1]{\raisebox{-0.95em}{\includegraphics[height=2.4em,keepaspectratio]{#1}}}
\crefname{figure}{Figure}{Figures}
\Crefname{figure}{Figure}{Figures}

\usepackage{hyperref}
\hypersetup{colorlinks=true, citecolor=blue!50!black, linkcolor=red!50!black, urlcolor=green!50!black} 

\newcommand{\1}{\mathbbm{1}}
\newcommand{\Prob}{\mathbb{P}}

\newcommand{\MSR}{\mathrm{MSR}}

\newcommand{\Var}{\mathbb{V}}

\definecolor{proxyblue}{HTML}{1E88E5}   
\definecolor{msrorange}{HTML}{FF6F00}   
\newcommand{\cproxy}[1]{\textcolor{proxyblue}{#1}}
\newcommand{\cmsr}[1]{\textcolor{msrorange}{#1}}
\newcommand{\STATE}{\State}
\newcommand{\REQUIRE}{\Require}
\newcommand{\ENSURE}{\Ensure}
\newcommand{\FOR}{\For}

\newcommand{\ENDFOR}{\EndFor}
\newcommand{\IF}{\If}

\newcommand{\ENDIF}{\EndIf}

\newcommand{\githubrepo}{\url{https://github.com/Advueu963/ProxySHAP}}
\newtheorem{theorem}{Theorem}[section]
\newtheorem{proposition}[theorem]{Proposition}
\newtheorem{lemma}[theorem]{Lemma}
\newtheorem{corollary}[theorem]{Corollary}

\definecolor{proxyLinearMSR}{HTML}{15b01a}   
\definecolor{proxyXGBMSR}{HTML}{1e88e5}      
\definecolor{proxyXGBHPO}{HTML}{87D7CD}  
\definecolor{proxyXGBCFG}{HTML}{6266E4} 
\definecolor{proxyLinearNoMSR}{HTML}{78C56A}
\definecolor{proxyXGBNoMSR}{HTML}{73A9E6}

\newcommand{\greencheck}{\textcolor{green}{\checkmark}}
\newcommand{\redcross}{\textcolor{red}{\texttimes}}
\title{Proxy-Based Approximation of\\Shapley and Banzhaf Interactions}

\author{%
  \textbf{Santo M. A. R. Thies\textsuperscript{\,1,2,3,}\thanks{\noindent Corresponding author: \url{s.thies@lmu.de}. \hfill Part of \icon{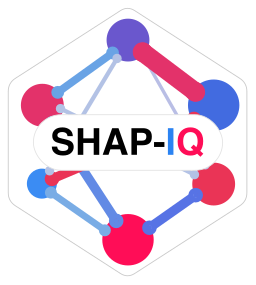} \texttt{shapiq}: \url{https://github.com/mmschlk/shapiq}.\\[-0.8em] \phantom{iiii} $^{\dagger}$Equal supervision.}}\quad
  \textbf{Hubert Baniecki\textsuperscript{\,4,5}}\quad
  \textbf{R. Teal Witter\textsuperscript{\,6}}\\
  \textbf{Eyke H\"ullermeier\textsuperscript{\,1,2,3}}\quad
  \textbf{Maximilian Muschalik\textsuperscript{\,1,2,$\dagger$}}\quad
  \textbf{Fabian Fumagalli\textsuperscript{\,1,2,7,$\dagger$}}\\[4pt]
  \textsuperscript{\textbf{1}}LMU Munich\quad
  \textsuperscript{\textbf{2}}MCML\quad
  \textsuperscript{\textbf{3}}DFKI\quad
  \textsuperscript{\textbf{4}}Centre for Credible AI, Warsaw University of Technology \\ 
  \textsuperscript{\textbf{5}}University of Warsaw\quad
  \textsuperscript{\textbf{6}}Claremont McKenna College\quad
  \textsuperscript{\textbf{7}}Bielefeld University\\[4pt]
}

\begin{document}

\maketitle

\begin{abstract}
Shapley and Banzhaf interactions capture the complex dynamics inherent in modern machine learning applications. However, current estimators for these higher-order interactions trade off between speed and accuracy. To overcome this limitation, we introduce ProxySHAP. ProxySHAP reconciles the high sample efficiency of tree-based proxy models with a situational residual correction recovering consistency.
On a theoretical level, we derive a polynomial-time generalization of interventional TreeSHAP that computes exact interaction indices for tree ensembles, thereby bypassing exponential tree-depth dependencies in prior methods. Furthermore, we formally analyze the residual adjustment strategy, characterizing the specific conditions under which Maximum Sample Reuse (MSR) corrects proxy bias without its variance scaling exponentially with interaction size. Extensive benchmarking demonstrates that ProxySHAP sets a new state-of-the-art standard for approximation quality, including in large-scale applications with thousands of features. By achieving the lowest error in both small- and large-budget regimes, ProxySHAP significantly outperforms the prior best estimators ProxySPEX and KernelSHAP-IQ, while also delivering superior performance on downstream explainability tasks.
\end{abstract}

\section{Introduction}
With the growing integration of artificial intelligence into critical decision-making processes across healthcare and finance, the demand for transparency and trustworthiness has never been higher. 
To interpret the often opaque dynamics of these models, the field has coalesced around cooperative game theory as the \emph{de facto} standard \citep{Lundberg.2017,Molnar.2022,Rozemberczki.2022,Muschalik.2024a}. 
Specifically, Shapley values \citep{Shapley.1953} and Banzhaf values \citep{Banzhaf.1964}, fundamental instances of \emph{cardinal-probabilistic values}, provide a rigorous axiomatic framework for attribution. 
These concepts are ubiquitous in modern machine learning, serving as the cornerstone for tasks such as feature attribution~\citep{Strumbelj.2010,Lundberg.2017} and data valuation~\citep{Rozemberczki.2022, Jia.2019}.

Formally, we consider a set $N = \{1, \dots, n\}$ consisting of $n$ entities, such as input features or training data points. 
The theoretical foundation of these explanations rests on a value function, $\nu: 2^N \to \mathbb{R}$, assigning a scalar score to any subset (or coalition) $S \subseteq N$. The interpretation of $\nu$ depends on the application: in feature attribution, $\nu(S)$ typically represents the model's prediction when features in $N \setminus S$ are masked or marginalized out; in data valuation, $\nu(S)$ corresponds to the utility, such as test accuracy or negative loss, of a model trained exclusively on the subset of data points $S$.

To summarize the complex dynamics of $\nu$, cardinal-probabilistic values quantify the marginal effect of each entity $i$ on the model's output. 
Formally, the attribution to $i$ is a weighted sum of its marginal contributions across all possible subsets:
\begin{align}
    \phi^p_i(\nu) := \sum_{T \subseteq N \setminus \{i\}} p_t(n) \Delta_i\nu(T),
\end{align}
where $\Delta_i\nu(T) \coloneqq \nu(T \cup \{i\}) - \nu(T)$ denotes the marginal contribution of $\nu$ at coalition $T$ with respect to $i$, and $t := |T|$ is the cardinality of $T$. Setting the weights to $p_t(n) = \tfrac{1}{n \binom{n-1}{t}}$ yields the Shapley value, while $p_t(n) = \tfrac{1}{2^{n-1}}$ yields the Banzhaf value.

While probabilistic values are useful for understanding individual contributions, they remain limited to singleton effects and often fail to capture the interactions between elements. To capture these complex dynamics, recent research has shifted toward richer explanations grounded in higher-order terms. These concepts have seen rapid adoption for quantifying feature interactions~\citep{Tsai.2022,Fumagalli.2023,Muschalik.2024}, cross-modal interactions in vision-language models~\citep{Baniecki.2025b,Jin.2025}, data valuation~\citep{Butler.2025}, language modelling~\citep{Spliethover.2025,Sengupta.2025}, and even hyperparameter optimization~\citep{Wever.2026}.
For instance, in vision-language similarity models such as SigLIP-2~\citep{tschannen2025siglip2}, interactions between image and text patches are not captured by probabilistic values~\citep{Baniecki.2025b}, as illustrated in \cref{fig:intro_example}.

Cardinal-probabilistic interaction indices naturally generalize probabilistic values to arbitrary subsets~\citep{Fujimoto.2006}. Rather than isolating the marginal contribution of a singleton, they quantify the contribution of a subset $S \subseteq N$ as it is added to a coalition $T$:
\begin{align}
\phi^p_S(\nu) := \sum_{T \subseteq N \setminus S} p_t^s(n) \Delta_S \nu(T),
\end{align}
where $\Delta_S \nu(T)$ denotes the discrete derivative of $\nu$ at $T$ with respect to the \emph{subset} $S$. 
Akin to the singleton case, specific weight choices allow us to recover the Shapley interaction index~\citep{Grabisch.1999} and the Banzhaf interaction index~\citep{Grabisch.1999}.
For $s=1$, the interaction weights reduce to $p_t(n)$.

\subsection{Estimating Cardinal-Probabilistic Interactions}

\begin{figure}[t]
    \centering
    \includegraphics[width=\textwidth]{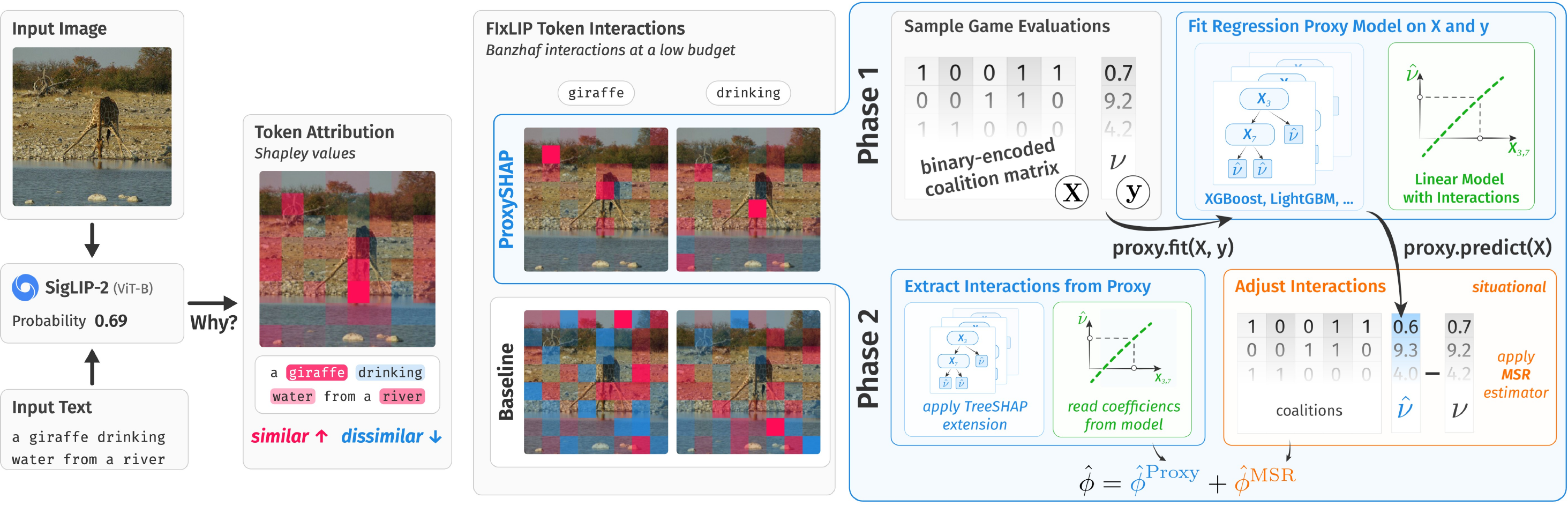}
    \caption{
    \textbf{Left:} A ProxySHAP explanation of the SigLIP-2 model using only $2048$ model calls. 
    \textbf{Right:} In \textit{Phase 1}, we fit a regression proxy model using sampled binary coalitions and game values. In \textit{Phase 2}, we extract proxy interactions and, when appropriate, adjust them using residual estimates.
    }
    \label{fig:intro_example}
\end{figure}
As computing probabilistic values and interactions requires evaluating $\nu(T)$ across all $2^n$ subsets, exact computation is infeasible for large $n$. Consequently, researchers rely on algorithms that produce estimates using a fixed evaluation budget. For probabilistic values, there are numerous estimators such as Monte Carlo sampling~\citep{Strumbelj.2014}, maximum sample reuse (MSR)~\citep{Wang.2023}, permutation sampling~\citep{Castro.2009}, and regression-based approaches like KernelSHAP~\citep{Lundberg.2017,Covert.2021,Musco.2025}.

Recently, work on estimating probabilistic \emph{interactions} has expanded to approximate the higher-order interactions inherent in modern machine learning, but current estimators are either slow or not accurate. SHAP-IQ~\citep{Fumagalli.2023} generalizes MSR to interactions, being efficiently computable, but inaccurate, as its variance scales quadratically with the magnitude of $\nu$, which we find compounds further for higher-order interactions.
Conversely, regression approaches like KernelSHAP-IQ~\citep{Fumagalli.2024} are accurate but require impractical sample budgets to adequately fit the large number of higher-order terms. Additionally, KernelSHAP-IQ remains computationally expensive, incurring quadratic time complexity in $n$.

To circumvent these bottlenecks, recent work leverages surrogate models $\hat{\nu}$ to approximate $\nu$~\citep{Butler.2025,Witter.2025}. ProxySPEX~\citep{Butler.2025} extracts Fourier coefficients from a tree surrogate, which scales exponentially with depth, necessitating aggressive, accuracy-compromising truncation. Additionally, even with the exact Fourier coefficients, ProxySPEX would return the interaction of $\hat{\nu}$ rather than the underlying function $\nu$.
RegressionMSR~\citep{Witter.2025} elegantly resolves this by directly extracting the values from the tree and employing MSR to estimate the residual values of $\nu - \hat{\nu}$, showcasing superior performance for probabilistic \textit{values}, yet remaining restricted to singleton effects.
Extending this residual-correction framework to higher-order interactions requires efficient extraction of interactions from tree proxies. 
While an algorithm for extracting Shapley interaction indices exists \citep{Zern.2023}, it is ill-suited for general cardinal-probabilistic interactions (e.g., those used in vision tasks, as shown in \cref{fig:intro_example}).

\subsection{Our Contributions}

Current interaction estimators are either fast or accurate; we propose \emph{ProxySHAP}, which is both.
First, we extend the algorithm of \citet{Zern.2023} to extract \emph{any} interaction index that exhibits a Möbius representation (\eqref{eq_möbius_conversion}), subsuming cardinal-probabilistic interaction indices, from a tree in polynomial time, orders of magnitude faster than ProxySPEX's Fourier-based extraction (see \cref{sec:tree_proxy,fig:runtime_fourier_vs_interventional}).
Second, we correct the resulting bias with the residual strategy of \citet{Witter.2025} and generalized MSR \citep{Fumagalli.2023}.
Unlike for probabilistic values, this correction does not always help, as its variance grows with the magnitude of the value function, the number of entities, and exponentially with the interaction order, so we characterize when it does (see \cref{sec:adjustment,fig:main_adjustment_benefit}).
Across \emph{54 datasets} spanning tabular, vision, language, graph, and vision--language domains, ProxySHAP attains the lowest error at both small budgets, previously held by ProxySPEX, and large budgets, previously held by KernelSHAP-IQ.
It remains applicable to thousands of entities and improves the downstream CLIP explanations (see \cref{sec:experiments_scaling,sec:CLIP}).

Our main contributions can be summarized as follows:
\begin{enumerate}[label=\textbf{(\arabic*)}, leftmargin=*, itemsep=2pt, topsep=4pt]
    \item \textbf{ProxySHAP.}
    We introduce ProxySHAP, an estimation framework for cardinal-probabilistic interaction indices (\cref{fig:intro_example}) that combines the sample efficiency of tree-based proxies with a residual correction applied only where it is most effective.

    \item \textbf{Theoretical Foundations.}
    We generalize interventional TreeSHAP to compute \emph{exact} cardinal-probabilistic interaction indices for tree ensembles in polynomial time, avoiding the exponential tree-depth dependence of Fourier-based extraction.
    We further bound the variance of the MSR correction for interactions, separating the regimes in which it removes proxy bias from those in which it dominates the error.

    \item \textbf{State-of-the-Art Performance.}
    Across $54$ datasets, ProxySHAP attains lower error than ProxySPEX and KernelSHAP-IQ, often by orders of magnitude in relative MSE, and Pareto-dominates both on CLIP explanations.
\end{enumerate}


\section{Background}

\textbf{Cardinal-Probabilistic Interaction Indices.}
\citet{Fujimoto.2006} extended semivalues to interactions between subsets $S \subseteq N$. A \emph{cardinal-probabilistic interaction index} is defined as:
\begin{equation}\label{eq:capi}
\phi^p_S(\nu) := \sum_{T \subseteq N \setminus S} p_t^s(n) \Delta_S \nu(T),\quad \mathrm{where}\quad 
\Delta_S \nu(T) := \sum_{L \subseteq  S}(-1)^{s-\ell}\nu(T \cup L).
\end{equation}
Here, $\Delta_S \nu(T)$ quantifies the interaction of $S$ in the presence of $T$. These indices satisfy the axioms of \emph{linearity (additivity)}, \emph{symmetry}, \emph{dummy partnership}, \emph{monotonicity}, and \emph{$k$-monotonicity} \citep{Fujimoto.2006}.

\textbf{Möbius Representation.}
\citet{Fujimoto.2006} have further shown the equivalent representation in terms of the \emph{Möbius transform} $m_S(\nu) := \Delta_S \nu(\emptyset)$, as
\begin{align}\label{eq_möbius_conversion}
    \phi^p_S(\nu) = \sum_{T \supseteq S} q_t^s(n) m_T(\nu),
\end{align}
where the corresponding weights $q_t^s(n)$ are computable by $p_{t}^s(n)$ for $s \in \{1,\dots,n\}$ and all $t \in \{s,\dots,n\}$.
The structure of cardinal-probabilistic interaction indices simplifies in their Möbius representation (see \cref{tab_capi_weights}), which itself is a cardinal-probabilistic interaction index.

\textbf{Maximum Sample Reuse (MSR).}
Since evaluating $\nu$ on all $2^n$ coalitions is intractable, \citet{Wang.2023} introduced maximum sample reuse (MSR), which was later extended to interactions by \citet{Fumagalli.2023}. Given a sample collection $\mathcal T$ drawn according to $\Prob_{\text{sampling}}$, MSR estimates:
\begin{align}
\hat\phi^{\MSR}_S(\nu;\mathcal T) 
:= \frac{1}{|\mathcal{T}|}\sum_{T\in \mathcal T} \nu(T) \frac{(-1)^{s-\vert S\cap T \vert} p^s_{t-\vert S \cap T \vert}(n)}{\Prob_{\text{sampling}}(T)}.
\end{align}
MSR generalizes Unbiased KernelSHAP \citep{Covert.2021} and has been refined via stratification \citep{Kolpaczki.2024b}.
However, the variance of MSR scales with $\nu(T)^2$ \citep{Witter.2025}.
To address this, \citet{Witter.2025} proposed \emph{RegressionMSR} using a proxy model $\hat\nu$. We refine this approach for cardinal-probabilistic interaction indices.

\section{Proxy-based Approximation of Shapley and Banzhaf Interactions}
\label{sec:proxyshap}

We now present \emph{ProxySHAP}, a general framework for approximating cardinal-probabilistic interaction indices via a decomposition into a proxy term and a residual correction term.
Given a budget of $m$ coalitions, we sample a collection $\mathcal{T} = \{T_1, \dots, T_m\} \subseteq 2^N$ and query the value function $\nu(T)$ for all $T \in \mathcal{T}$.
Using these samples, we fit a proxy model $\hat\nu_\mathcal{T}: 2^N \to \mathbb{R}$ to approximate the underlying game $\nu$.
The key observation underlying ProxySHAP is that, by linearity of cardinal-probabilistic interaction indices \citep{Fujimoto.2006}, the target interaction admits the decomposition
\begin{align}\label{eq:proxyshap}
    \phi^p_S(\nu)
    =
    \underbrace{\phi^p_S(\hat{\nu}_\mathcal{T})}_{\text{Exact Proxy}}
    +
    \underbrace{\phi^p_S(\nu-\hat{\nu}_\mathcal{T})}_{\text{Residual}}.
\end{align}
This separates the overall approximation problem into two components:
\emph{(i)} a modeling problem, namely how well the proxy $\hat\nu_{\mathcal T}$ approximates $\nu$, and
\emph{(ii)} a correction problem, namely, how accurately the residual interactions can be estimated from the available coalition evaluations.

This perspective makes ProxySHAP a \emph{framework} rather than a single fixed estimator.
At a conceptual level, the only requirement for the proxy class is that its cardinal-probabilistic interactions can be efficiently extracted.
Learning the proxy itself is a standard supervised regression problem, where the sampled coalitions are represented as binary inputs and the corresponding game values $\nu(T)$ serve as targets.
Hence, in principle, any machine learning model can serve as a proxy, provided its interactions remain tractable.
Moreover, this viewpoint naturally enables hyperparameter optimization to improve proxy quality, as shown in \cref{fig:main_scaling_figure}.

In this work, we study two proxy classes.
First, we consider a linear proxy with interaction features, which provides a simple and often competitive baseline.
Second, we consider tree-based proxies, which form our main instantiation of ProxySHAP.
Their appeal is that they are substantially more expressive than linear surrogates, while---as we show in the next subsection---still admitting exact polynomial-time cardinal-probabilistic interaction extraction.
This exact extraction is the key to obtaining a substantial runtime improvement over Fourier-based proxy methods.

\textbf{Linear Proxy.}
We first consider a linear model with interaction features.
Let $\mathcal S \subseteq 2^N$ be all interactions of interest, e.g.\ interactions up to order $k = 1,\dots,n$.
We define the linear proxy as
\begin{align}
    \hat\nu_{\text{linear}}(T) := \sum_{S \in \mathcal S} \beta_S \cdot \1[S \subseteq T]
    \qquad \text{with } \beta \in \mathbb{R}^{|\mathcal S|}.
\end{align}
The coefficients $\beta \in \mathbb{R}^{|\mathcal S|}$ are determined using standard linear regression.
A key advantage of the linear proxy in our setting is that its interactions are directly accessible through the Möbius representation.
Each coefficient $\beta_S$ corresponds to a Möbius coefficient $m_S(\hat\nu_{\text{linear}})$, and therefore the proxy interactions can be computed directly via \cref{eq_möbius_conversion}.

\begin{proposition}
    The Möbius transform of the linear proxy is given by
    $
        m_S(\hat\nu_{\text{linear}}) = \beta_S \1[S \in \mathcal S].
    $
    Hence,
    \begin{align*}
        \phi_S^p(\hat\nu_{\text{linear}})
        =
        \sum_{T \in \mathcal S: T \supseteq S} q_t^s(n)\beta_T.
    \end{align*}
\end{proposition}

Thus, linear proxies provide a simple and effective instantiation of ProxySHAP whenever the interaction basis remains sufficiently small to be fitted reliably.
However, they become restrictive when the value function exhibits a strong nonlinearity, and the number of their parameters grows combinatorially with the interaction order.
This makes them increasingly impractical for large $n$ or high-order interactions.
Tree-based proxies offer a compelling alternative: they capture complex non-linear relationships while remaining applicable to large feature sets and high-order interactions.
Motivated by their strong empirical performance for cardinal-probabilistic values \citep{Witter.2025}, we make tree-based proxies applicable to interaction estimation by deriving an exact polynomial-time extraction procedure for interactions which exhibit a Möbius representation as in \cref{eq_möbius_conversion}.

\subsection{Exact Proxy Interactions for Tree-Based Models}
\label{sec:tree_proxy}

Our main proxy classes are tree-based models such as XGBoost \citep{Chen.2016} and LightGBM \citep{Ke.2017}. We define such tree-based proxies as the piecewise constant functions induced by decision tree leaf predictions:
\begin{align}\label{eq_tree_proxy}
    \hat\nu_{\text{tree}}(T) := \sum_{j \in \mathcal L} c_j \cdot \1[R_j \subseteq T \subseteq N \setminus L_j].
\end{align}

Here, $\mathcal L$ is the set of leaves, $c_j \in \mathbb{R}$ is the prediction of leaf $j$, and $L_j$ and $R_j$ are the features that split left and right along the path to $j$.
That is, a coalition $T$ reaches leaf $j$ if and only if it contains every feature in $R_j$ and none in $L_j$.
Each leaf term is thus an indicator game \citep{Zern.2023}, whose Möbius transform has a closed form (see Lemma \ref{lemma:moebius_intervalgame}).
Combining it with \cref{eq_möbius_conversion} and summing over the leaves yields the interaction index of the tree proxy.
\begin{wrapfigure}[17]{r}{0.38\textwidth}
    \centering
    \includegraphics[width=\linewidth]{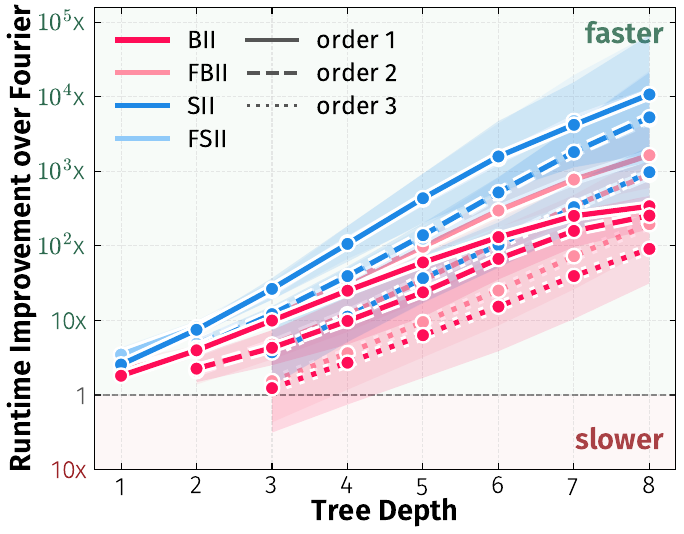}
\caption{Runtime improvement of extracting interactions using our \cref{alg:tree_capi} over Fourier-based extraction. 
Per-dataset speedups and the effect of tree depth on approximation quality are shown in \cref{fig:appendix:interventional_speedup,fig:tree_depth_on_approximation_quality}.}
    \label{fig:runtime_fourier_vs_interventional}
\end{wrapfigure}
\begin{proposition}\label{prop_tree_capi}
For $S \subseteq N$, we have
\begin{align*}
    \phi^p_S(\hat\nu_{\textnormal{tree}})
    =
    \sum_{j \in \mathcal L: S \subseteq L_j \cup R_j}
    c_j \cdot
    \lambda\left(|L_j|, |R_j|, |S \cap L_j|, |S|\right),
\end{align*} 
where $\lambda(\ell,r,u,s) := \sum_{i=0}^{\ell-u} (-1)^{i+u} \binom{\ell-u}{i} q_{i+u+r}^s(n).$
\end{proposition}
Proposition~\ref{prop_tree_capi} shows that the interactions of the tree proxy can be computed exactly by aggregating leaf-wise contributions, requiring only a single pass over the tree leaves for a fixed interaction $S$ together with evaluation of the closed-form weight $\lambda$.
This yields \cref{alg:tree_capi} with a runtime of $\mathcal O(n_{\text{nodes}})$ per interaction and $\mathcal O(n_{\text{nodes}}\cdot |\mathcal S|)$ for a collection $\mathcal S$ of interactions (see Appendix~\ref{appendix:generalization_tree_shap} for details).
Crucially, by the linearity of cardinal-probabilistic interaction indices \citep{Fujimoto.2006}, the result extends to ensembles of trees, with runtime scaling linearly in the number of trees. 
This exact computation is a key computational advantage of tree-based ProxySHAP.

In contrast to ProxySPEX, which obtains interactions via Fourier coefficients and exhibits worst-case complexity exponential in the tree depth, $\mathcal{O}(4^d)$ \citep{Butler.2025,Gorji.2025}, our method remains efficient for deep trees.
As shown in \cref{fig:runtime_fourier_vs_interventional}, interventional tree-based extraction is orders of magnitude faster than Fourier-based extraction (see Appendix~\ref{appendix:fourier_vs_interventional} for details), 
making tree models particularly attractive proxies for ProxySHAP. Further comparison can be found in Appendix~\ref{appendix:detailed_comparison_proxyshap_proxyspex}.

Additionally, our tree-based ProxySHAP supports computing up to 10 indices, whereas previous methods support at most 4 (see \cref{tab:extraction_algorithms}), yielding practical gains for explaining text-image similarity models such as CLIP (see Figure~\ref{fig:clip_experiment}). Furthermore, Proposition~\ref{prop_tree_capi} filters the number of non-zero interactions, as $\phi_S^p(\hat \nu_{tree}) = 0$ whenever the features of $S$ do not co-occur on a common root-leaf path. As such, for a tree of depth $d$ and interactions of size $s$, at most $\mathcal{O}(n_{leaves} \binom{d}{s})$ can be non-zero, being often way smaller than the total number of interactions $\binom{n}{s}$.

\subsection{Residual Adjustment and its Practical Limits}
\label{sec:adjustment}

While the tree proxy alone already yields a powerful estimator, it is not consistent in general, since the fitted proxy need not perfectly match the underlying value function $\nu$ (see, e.g., \cref{fig:mse_quality}).
To obtain consistency, meaning that we recover the exact index once all $2^n$ coalitions are observed, we correct for proxy bias by estimating the interactions of the residual game $\nu - \hat{\nu}_{\mathcal T}$.
As previously shown by \citet{Witter.2025}, MSR adjustment reliably improves approximation quality for probabilistic values, which our results in \cref{fig:main_adjustment_benefit} confirm.
This motivates the adjusted ProxySHAP estimator:
\[
\hat\phi^{\mathrm{ProxySHAP}}_S(\nu;\mathcal T)
=
\phi^p_S(\hat\nu_{\mathcal T})
+
\hat\phi^{\MSR}_S(\nu-\hat\nu_{\mathcal T};\mathcal T).
\]

While MSR often improves singleton estimates, this \emph{does not} directly translate to improved interaction approximation, as evidenced in \cref{fig:main_adjustment_benefit}. This is explained by the variance of the MSR estimator.\\
\begingroup\interlinepenalty=1000
\begin{theorem}[Variance growth of MSR under size-uniform sampling for Shapley interactions]
\label{th:var_msr}
Suppose coalitions are sampled i.i.d. uniformly over subset size,
$
\Prob_{\mathrm{sampling}}(T)=\frac{1}{(n+1)\binom{n}{|T|}}.
$\\
Then, for any interaction $S \subseteq N$, the MSR estimator of the Shapley interaction index satisfies
\[
\Var\!\left[\hat\phi^{\MSR}_S(\nu;\mathcal T)\right]
\le
\begin{cases}
\mathcal{O}\!\left(\dfrac{\|\nu\|_\infty^2 \log n}{|\mathcal{T}|}\right), & |S|=1,\\[8pt]
\mathcal{O}\!\left(\dfrac{\|\nu\|_\infty^2 n^{|S|-1}}{|\mathcal{T}|}\right), & |S|\geq 2.
\end{cases}
\]
\end{theorem}
\endgroup
\begin{figure}
    \centering
    \includegraphics[width=\textwidth]{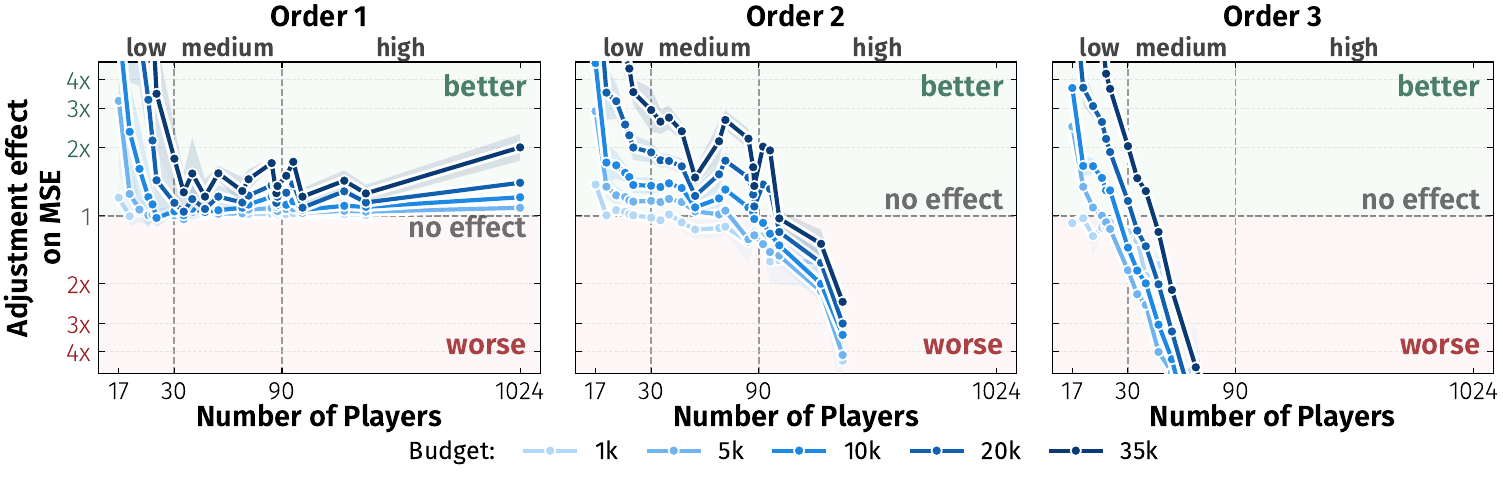}
    \caption{
    Comparison of ProxySHAP with and without MSR adjustment, measured by the MSE ratio. While MSR improves Shapley value approximation, it can degrade higher-order interaction estimates, as its variance is bounded by $n^{k-1}/|\mathcal{T}|$ for interactions of order $k$ (see \cref{th:var_msr}).
    }
    \label{fig:main_adjustment_benefit}
\end{figure}

\cref{th:var_msr} shows that, for Shapley interactions with uniform over-sizes, the variance upper bound grows rapidly with the interaction order.
In particular, when estimating all Shapley interactions up to order $k$, the dominant variance term scales as
$
    \mathcal{O}\!\left(\frac{\|\nu\|_\infty^2 n^{k-1}}{|\mathcal{T}|}\right).
$
Thus, the variance of MSR is governed by three quantities: the number of sampled coalitions $|\mathcal T|$, the number of players $n$, and the maximal interaction order $k$.
We provide matching lower bounds and a related order-dependent result for Banzhaf interactions in Appendix~\ref{appendix:proof:var_msr}.

We evaluate the quality of MSR adjustment across all $32$ benchmarked \textit{TabArena}~\citep{Erickson.2025} datasets in \cref{fig:main_adjustment_benefit}; further details are provided in Appendix~\ref{sec:appendix:adjustment}. Our findings are threefold. First, in the pairwise interaction approximation, MSR substantially reduces MSE as the evaluation budget increases. Second, this benefit diminishes as the number of players grows, making the adjustment impractical for high-dimensional games. Third, for higher-order interactions, adjustments can already harm performance in medium-sized games, even with large budgets, in contrast to the more robust gains observed for pairwise interactions. 
For Banzhaf interactions, similar insights apply, with adjustment being even less beneficial as player numbers increase (see Figure~\ref{fig:appendix:adjustment_benefit_bii}).

Consequently, while MSR adjustment remains effective for first-order interactions, its benefit for higher-order interactions is more nuanced. It substantially improves approximation quality for second-order interactions in small- to medium-sized games, but beyond pairwise interactions, it is mainly beneficial for smaller games. As a practical rule of thumb, we recommend applying MSR adjustment primarily in low-dimensional settings with $n < 30$. For larger games, adjustment should only be used when the sampling budget is sufficiently large relative to the dominant variance term, i.e., when $|\mathcal T| \gg n^{k-1}$ for interactions up to order $k$, which in our experiments occurs for medium-sized games only at budgets above $10{,}000$.

\begin{wrapfigure}[17]{r}{0.55\textwidth}
\begin{minipage}{0.55\textwidth}
\vspace{-2.2em}
\begin{algorithm}[H]
\caption{\texttt{ProxySHAP(XGBoost, MSR)}}
\label{alg:proxyshap}
\small
\begin{algorithmic}[1]
\Require value function $\nu$, weight $p_t^s(n)$, sampling distribution
  $\Prob_{\mathrm{sampling}}$, interactions of interest $\mathcal S \subseteq 2^N$
\Ensure $\hat\phi^p_S$ for all $S \in \mathcal S$
\State $\mathcal T \gets$ sample according to $\Prob_{\mathrm{sampling}}$
\BeginBox[draw=black!15, fill=black!3, rounded corners=2pt]
\LComment{\textbf{Phase 1}: fit \cproxy{proxy $\hat\nu$} and \cmsr{residual game $r$}}
\State $\cproxy{\hat\nu_{\mathcal T}} \gets$ train XGBoost on $\{(T, \nu(T)) : T \in \mathcal T\}$
\State $\cmsr{r}(T) \gets \nu(T) - \cproxy{\hat\nu_{\mathcal T}}(T)$ for all $T \in \mathcal T$
\EndBox
\BeginBox[draw=black!15, fill=black!3, rounded corners=2pt]
\LComment{\textbf{Phase 2}: extract interactions}
\For{$S \in \mathcal S$}
  \State $\cproxy{\hat\phi^{\mathrm{Proxy}}_S} \gets \phi^p_S(\cproxy{\hat\nu_{\mathcal T}})$
         \Comment{exact, \cref{alg:tree_capi}}
  \State $\cmsr{\hat\phi^{\MSR}_S} \gets \hat\phi^{\MSR}_S(\cmsr{r};\mathcal T)$
         \Comment{situational (\cref{sec:adjustment})}
  \State $\hat\phi^p_S \gets \cproxy{\hat\phi^{\mathrm{Proxy}}_S} + \cmsr{\hat\phi^{\MSR}_S}$
\EndFor
\EndBox
\State \Return $\hat\phi^p_S$ for all $S \in \mathcal S$
\end{algorithmic}
\end{algorithm}
\end{minipage}
\vspace{-1.0em}
\end{wrapfigure}

\subsection{The ProxySHAP Algorithm}
We summarize ProxySHAP in \cref{alg:proxyshap}, using \cref{alg:tree_capi} to compute the exact proxy interactions.
Following \citet{Witter.2025}, we use the same sampled coalitions $\mathcal T$ for both proxy fitting and residual estimation.
This makes the residual game depend on the coalitions used to correct it, so \cref{th:var_msr} does not apply verbatim, but we find that reusing coalitions substantially improves the approximation quality for interactions as well (\cref{appendix:shared_disjoint_subsets}).
Unless stated otherwise, we sample coalitions uniformly over subset sizes without replacement \citep{Musco.2025}, which can only reduce the variance relative to i.i.d.\ sampling \citep{Hoeffding.1963}.
ProxySHAP also works with other sampling schemes \citep{Lundberg.2017,Fumagalli.2024}, though these yield similar results (\cref{appendix:abalation}).

\textbf{Computational Complexity.} 
Fitting the XGBoost proxy $\hat{\nu}$ with $m$ binary-encoded coalitions takes roughly $O(m \log m)$ time \citep{Chen.2016}, and extracting one interaction with \cref{alg:tree_capi} takes $O(n_{\text{trees}} \ell_{\max})$ time, where $n_{\text{trees}}$ is the number of trees and $\ell_{\max}$ the maximum number of leaves per tree.
Tree proxies are moreover inherently sparse.
By Proposition~\ref{prop_tree_capi}, $\phi^p_S(\hat\nu_{\text{tree}}) = 0$ unless $S \subseteq L_j \cup R_j$ for some leaf $j$, so only interactions co-occurring on a root-to-leaf path are non-zero.
Iterating over leaves, unadjusted ProxySHAP returns all of them in $O(n_{\text{trees}} \ell_{\max} 2^{d})$ time for maximum tree depth $d$, against $O(n^k)$ interactions in total.
The MSR adjustment is dense, since $\hat\phi^{\MSR}_S(r;\mathcal T)$ is generally non-zero for every $S \in \mathcal S$, raising the cost to $\mathcal{O}\left(|\mathcal{S}| \cdot (n_\text{trees} \ell_{\max}+m)\right)$.
ProxySHAP thus offers a sparse but biased variant that scales beyond $1000$ features (see \cref{fig:main_scaling_figure}), and a consistent one that costs $O(n^km)$ but pays off at lower player counts (see \cref{fig:main_adjustment_benefit}).
For a linear proxy, the number of fitted parameters equals the number of target interactions, so the demand grows as $O(m n^{2k})$ using ordinary least squares.

\section{Experiments}
\label{sec:experiments}
We empirically evaluate ProxySHAP across a range of experimental settings and systematically compare it against standard baselines, including KernelSHAP-IQ \citep{Fumagalli.2024}, ProxySPEX \citep{Butler.2025}, SHAP-IQ~\citep[i.e., MSR for interactions,][]{Fumagalli.2023}, SVARM-IQ~\citep{Kolpaczki.2024b}, and traditional permutation sampling~\citep{Sundararajan.2020,Tsai.2022}.
All experiments are based on \texttt{shapiq} \citep{Muschalik.2024a}, and code can be found in \githubrepo.

\textbf{Games.}
We evaluate ProxySHAP on local-explanation games across tabular benchmarks~\citep{Strumbelj.2010,Muschalik.2024a}, including TabArena~\citep{Erickson.2025}, and established synthetic, vision, language, graph, and vision--language settings~\citep{Muschalik.2024a,Muschalik.2025,Baniecki.2025b}. 
The underlying models include TabPFN~\citep{Hollmann.2025}, XGBoost~\citep{Chen.2016}, LightGBM~\citep{Ke.2017}, vision transformers~\citep{dosovitskiy2021an}, language models~\citep{Sanh.2019}, CLIP~\citep{Radford.2021}, DINOv2~\citep{Oquab.2024} and GNNs~\citep{Xu.2019}; \cref{tab:dataset_summary} summarizes all $54$ datasets. 
We use exhaustive evaluation only for games with $n \leq 25$, since exact Shapley and Banzhaf interactions scale exponentially in $n$. 
For larger tabular games, tree models allow efficient ground-truth extraction via \cref{alg:tree_capi}, for graphs, we use \texttt{GraphSHAP-IQ}~\citep{Muschalik.2025}.
The synthetic SOUM games are built by randomly sampling Möbius coefficients, enabling ground-truth values. We hereby follow the procedure described in \citep{Tsai.2022}. 
For CLIP, ground-truth interactions are unavailable, so we use task-specific faithfulness metrics. 
Further details are given in \cref{appendix:experiment_details}.

\textbf{Metrics.}
We measure approximation quality using relative mean squared error (Relative MSE; lower is better), defined as the sum of squared errors normalized by the sum of squared ground-truth interaction values. 
This normalization makes errors comparable across games with different interaction magnitudes. 
Unless stated otherwise, we report the mean and standard error of the mean (SEM) over $30$ explained instances per dataset. 
We additionally report computational efficiency in terms of model evaluations and runtime (see Appendix~\ref{appendix:runtime}).
To broaden the scope of evaluation, for models such as CLIP, we evaluate explanation quality using $R^2$ faithfulness and the area under the insertion--deletion curve (see Appendix~\ref{appendix:clip_experiments}; \citealp{Baniecki.2025b}).

\subsection{Approximation Quality}
\label{sec:approximation_quality}
\begin{figure*}[tb]
    \centering
    \includegraphics[width=\textwidth]{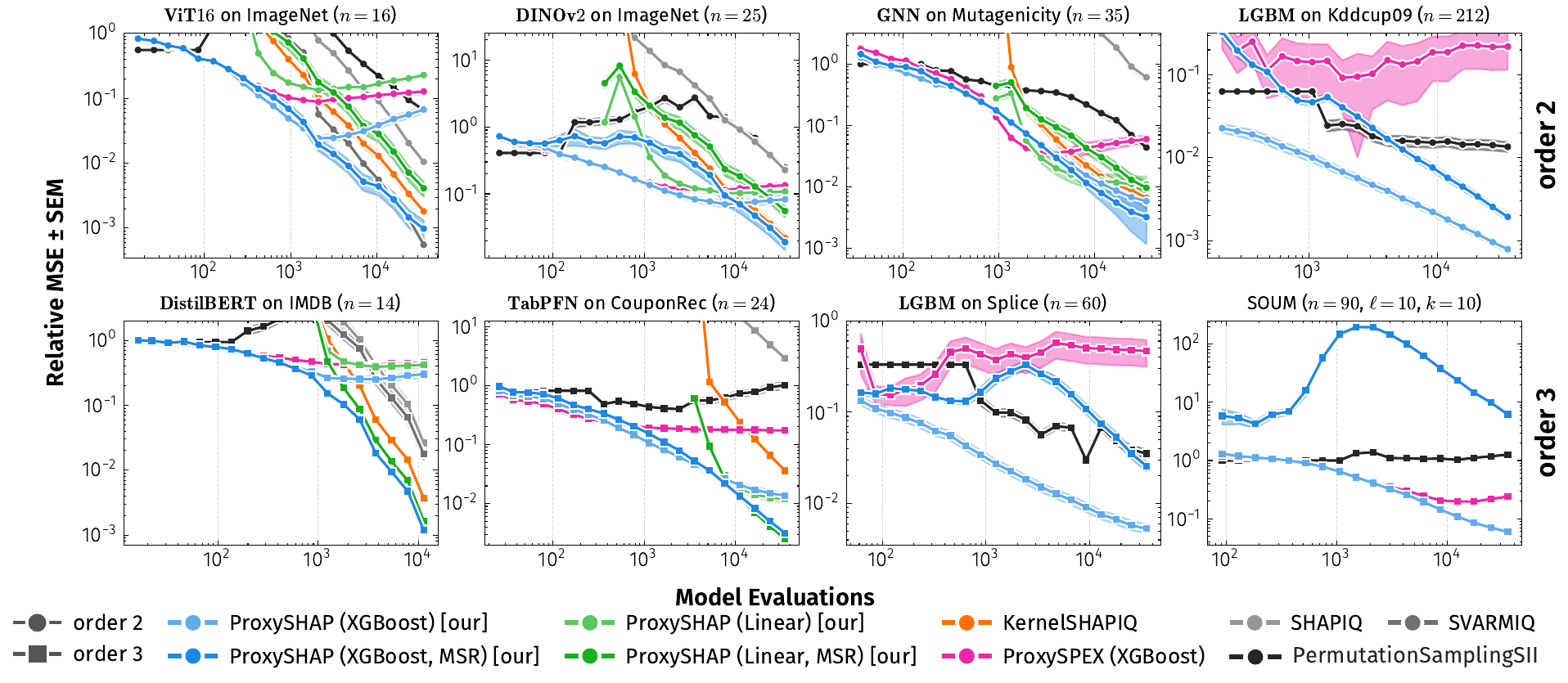}
    \caption{Approximation quality (Relative MSE) for Shapley interactions of ProxySHAP across different configurations and state-of-the-art baselines. Additional results for Shapley and Banzhaf interactions on all $54$ datasets can be found in \cref{fig:appendix:winnermap_sii} and \cref{fig:appendix:winnermap_bii}, respectively.}
    \label{fig:mse_quality}
\end{figure*}
We first compare ProxySHAP to state-of-the-art baselines in terms of approximation quality~(see \cref{fig:mse_quality}).
We evaluate two proxy classes: a linear model with interaction features, \textcolor{proxyLinearNoMSR}{ProxySHAP (Linear)}, and an XGBoost regressor with default hyperparameters, \textcolor{proxyXGBNoMSR}{ProxySHAP (XGBoost)}, each with and without MSR adjustment.

Across datasets and interaction orders, ProxySHAP consistently outperforms all baselines, often by several orders of magnitude in relative MSE.
MSR adjustment improves estimates for small games (\(n < 30\)) and for larger games when the sampling budget is sufficiently large, as in \textsc{Mutagenicity}.
However, consistent with \cref{sec:adjustment}, its benefit decreases with interaction order and player count and can hurt performance in larger games, as observed on \textsc{SPLICE}.
The linear proxy is competitive and sometimes even outperforms the tree-based proxy (see \Cref{fig:appendix:winnermap_sii}).
Yet, it becomes impractical for large feature counts and higher-order interactions due to its computational complexity (see \cref{sec:proxyshap}).
Tree-based proxies avoid this bottleneck and remain effective even when KernelSHAP-IQ is no longer feasible.
Overall, XGBoost provides a scalable and accurate proxy, while MSR adjustment is most useful when the sampling budget is sufficiently large relative to the player count and interaction order.
Additional results for further datasets and for both Shapley and Banzhaf interactions are provided in \cref{appendix:additional_experiments}.

\subsection{Practical Considerations of ProxySHAP}
\label{sec:experiments_scaling}
\textbf{HPO improves performance.}
The quality of ProxySHAP's approximation depends directly on the fitted proxy. 
We validate this via hyperparameter optimization (HPO), denoting the tuned variant as \textcolor{proxyXGBHPO}{ProxySHAP (XGBoost+HPO)} (see Appendix~\ref{appendix:hpo_method}). As shown in \cref{fig:main_scaling_figure} (left), HPO substantially improves over the default XGBoost proxy, with further examples in Appendix~\ref{appendix:xgboost_default_effect}. Since HPO is costly, we derive a cheaper \textcolor{proxyXGBCFG}{ProxySHAP (XGBoost+HPO-Informed)} variant from recurring strong HPO configurations. It performs on par with full HPO at small to medium budgets, highlighting the importance of proxy configuration.
\begin{figure}[t]
    \centering
    \includegraphics[width=\textwidth]{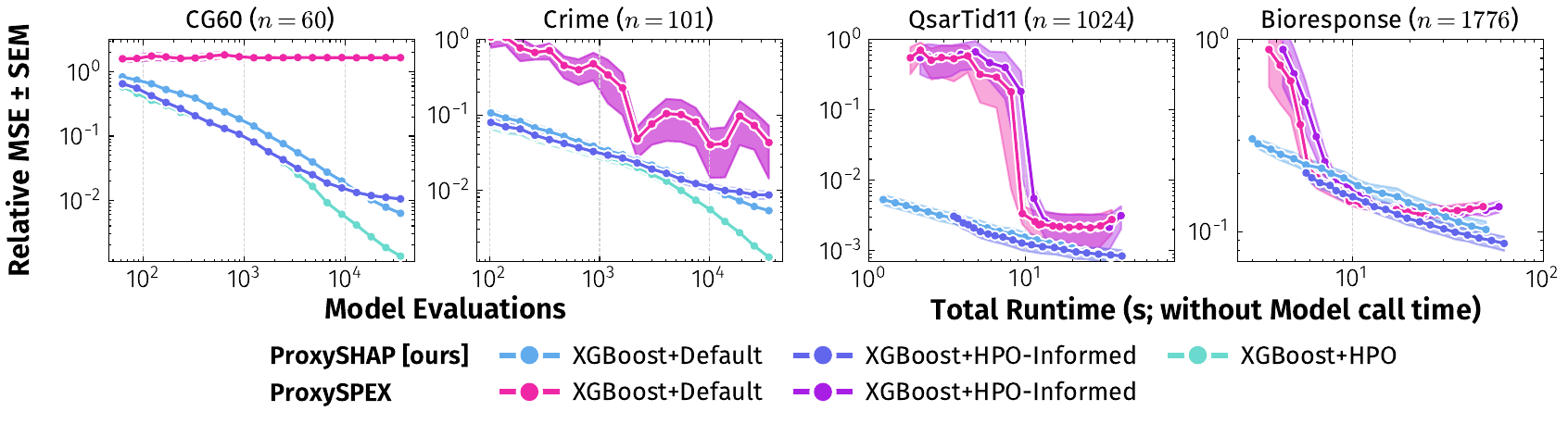}
    \caption{Relative MSE for pairwise Shapley interaction approximation of ProxySHAP with HPO (left) and for large $n$ (right). Further results in  \cref{fig:appendix:xgboost_default_comparison}.}
    \label{fig:main_scaling_figure}
\end{figure}

\textbf{Scalability of ProxySHAP.}
We demonstrate the scalability of ProxySHAP beyond $1000$ features in \cref{fig:main_scaling_figure}. Using ProxySHAP, we achieve better approximation quality at lower runtime, with the HPO-informed variant further improving it.
Unlike ProxySHAP, ProxySPEX does not improve approximation quality with the HPO-informed configuration.

\subsection{ProxySHAP Improves the Approximation of Faithful Interaction Explanations of CLIP}
\label{sec:CLIP}

We demonstrate the broader applicability of ProxySHAP to improve the approximation of faithful interaction explanations of CLIP~\citep[FIxLIP,][]{Baniecki.2025b}, a popular vision--language encoder architecture~\citep{Radford.2021}.
We follow the original experimental setup \citep{Baniecki.2025b} and compare ProxySHAP to the original linear regression-based approximation, as well as ProxySPEX, in the FIxLIP game, where the goal is to explain the interaction between image and text inputs in CLIP (see \cref{fig:intro_example}). Additional details and results are provided in Appendix~\ref{appendix:clip_experiments}.

\textbf{Setup.} Similarly to \citet{Baniecki.2025b}, we analyze the CLIP model in two vision transformer variants: ViT-32 and ViT-16.
We explain a sample of 200 image--text pairs from the MS COCO dataset~\citep{Lin.2014}, each having around 10--30 text tokens per input, yielding about 60--72 players in the final approximation game. 
As in the original work, we estimate the faithful Banzhaf interaction index of order 2 and quantify the quality of explanations with the area between the insertion/deletion curves (AID). 
Unlike the original work, we use smaller budgets ranging from $10^3$ to $10^5$ CLIP model calls, where each inference denotes that both the vision and language encoders are called on a single image--text pair.
\begin{wrapfigure}[14]{r}{0.5\textwidth}
    \centering
    \includegraphics[width=\linewidth]{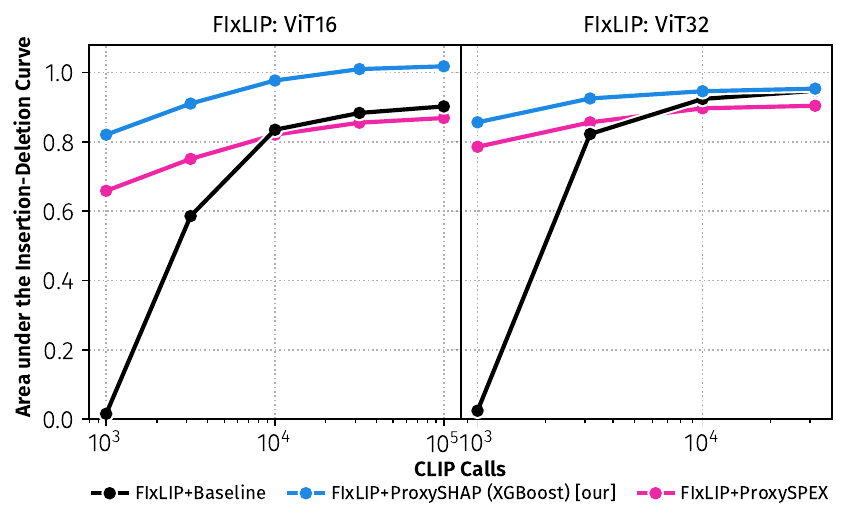}
    \caption{Area between the insertion/deletion curves (AID) for explaining two CLIP ViT variants on MS COCO with ProxySHAP, ProxySPEX, and the baseline.}
    \label{fig:clip_experiment}
\end{wrapfigure}

\textbf{Results.} 
\cref{fig:clip_experiment} effectively demonstrates that ProxySHAP Pareto-dominates the original linear regression-based approximation, as well as ProxySPEX, in the FIxLIP game for both models. 
Overall, we find that the ProxySHAP adjustment incurs substantial computational overhead despite not improving the estimate.
Note that in this setup, ProxySPEX and ProxySHAP with adjustment took about $10\times$ as long to compute as ProxySHAP without adjustment and the FIxLIP baseline.
In Appendix~\ref{appendix:clip_experiments}, we provide additional analysis measuring the $R^2$ metric, which yields a similar conclusion, and we ablate on FIxLIP's cross-modal estimator. 

\section{Conclusion}
\label{sec:conclusion}
We propose ProxySHAP, an efficient model-agnostic approximation method for the broad class of cardinal-probabilistic interaction indices. 
At its core, ProxySHAP extracts, in closed form, any cardinal-probabilistic interaction index, including Shapley and Banzhaf variants, from tree-based proxies.
Notably, we achieve orders-of-magnitude speedups over Fourier-based extraction. 
By exploiting the linearity of interaction indices, ProxySHAP admits a clean decomposition into an exact proxy term and an MSR-based residual correction.
Based on our extensive empirical evaluation and theoretical variance bound for MSR, we recommend applying the residual correction \emph{situationally}: for interactions of order $k$, it is most useful in games with $n < 30$ players or when the budget satisfies $|\mathcal{T}| \gg n^{k-1}$.
Empirically, ProxySHAP achieves strong approximation quality across an extensive benchmark, outperforming ProxySPEX and KernelSHAP-IQ in both low- and high-budget regimes, and Pareto-dominating FIxLIP and ProxySPEX when applied to CLIP.

\textbf{Limitations and future work.}
The quality of ProxySHAP's approximation depends directly on the trained proxy.
We demonstrated that HPO substantially improves the performance of the XGBoost proxy but incurs significant computational overhead. A second caveat concerns inductive bias. In the tabular games, the explained model is often itself a gradient-boosted tree, which may favor a tree-based proxy. We therefore also evaluate on games whose underlying models are not tree-based, namely TabPFN, vision-and-language transformers, graph neural networks, CLIP, and SOUM, where ProxySHAP retains its advantage. In future work, we aim to explore other proxy classes, including their efficient interaction extraction.

\textbf{Broader impact.}
We believe ProxySHAP can empower scientific researchers to quantify nonlinear dependencies among variables in complex systems, e.g., in physics~\citep{liu2026disentangling} and material sciences~\citep{cai2025improved}.
We provide a highly scalable C++ implementation of our algorithm, enabling the quantification of interactions in foundation models spanning billions of parameters across large datasets.

\begin{ack}
Fabian Fumagalli and Maximilian Muschalik acknowledge funding by the Deutsche Forschungsgemeinschaft (DFG, German Research Foundation): TRR 318/3 2026 – 438445824. Hubert Baniecki was supported by the Foundation for Polish Science (FNP), and the Polish Ministry of Education and Science within the ``Pearls of Science'' program, project number PN/01/0087/2022.
\end{ack}

\bibliographystyle{plainnat}
\bibliography{paper_bib}

\clearpage
\onecolumn
\appendix

\section*{Appendix for ``Proxy-Based Approximation of Shapley and Banzhaf Interactions''}

\startcontents[sections]
\printcontents[sections]{l}{1}{\setcounter{tocdepth}{2}}

\clearpage

\section{Proofs}
\label{appendix:proofs}

\begin{table}[t]
\centering
\small
\renewcommand{\arraystretch}{1.2}
\setlength{\tabcolsep}{12pt}
\caption{Weights of cardinal-probabilistic interaction indices.}
\label{tab_capi_weights}
\vspace{0.25em}
\begin{tabular}{@{}l c c c@{}}
\toprule
\textbf{Weight} 
& \textbf{Banzhaf} 
& \textbf{Shapley} 
& \textbf{Möbius} \\ 
\hdashline\noalign{\vskip 0.2em}
$p_t^s(n)$ 
& $(\frac{1}{2})^{n-s}$ 
& $\displaystyle \frac{1}{(n-s+1)\binom{n-s}{t}}$ 
& $\mathbbm{1}[t=0]$ \\
\addlinespace[0.3em]
$q_t^s(n)$ 
& $(\frac{1}{2})^{t-s}$ 
& $\displaystyle \frac{1}{t-s+1}$ 
& $\mathbbm{1}[t=s]$ \\ 
\bottomrule
\end{tabular}
\end{table}

A central contribution of our method is an extension of the algorithm of \citet{Zern.2023} to compute cardinal-probabilistic interactions of arbitrary order (see \cref{alg:tree_capi}). 
We establish a general extraction result in \cref{appendix:proof:general}, with specialized proofs for commonly used indices—including the Banzhaf Interaction Index (BII)\citep{Grabisch.1999}, the Chaining Interaction Index (CII) \citep{Fujimoto.2006}, the Faithful Banzhaf Interaction Index (FBII)\citep{Tsai.2022}, Faithful Shapley Interaction Index (FSII) \citep{Tsai.2022} and Shapley Taylor Interaction Index (STII) \citep{Sundararajan.2020}—provided in \cref{appendix:proofs:capi_extraction}. 
These results rely repeatedly on Lemma~\ref{lemma:interval_decomposition} and Lemma~\ref{lemma:moebius_intervalgame}.

We prove the variance growth of MSR for general interaction indices under general sampling schemes in Appendix~\ref{appendix:proof:var_msr}.
We investigate the growth of the variance of MSR for the Shapley Interaction Index and the Banzhaf Interaction Index, using sampling uniformly over sizes.
We further showcase that the derived upper bounds are tight by providing matching lower bounds for both indices and sampling schemes.

\subsection{Proof of Proposition~\ref{prop_tree_capi}}
\label{appendix:proof:general}
Before proving the main result, we first establish two lemmas that are instrumental for the proof of Proposition~\ref{prop_tree_capi}.
These will also be used to derive the closed-form expression for the Banzhaf Interaction Index, Chaining Interaction Index, Faithful Banzhaf Interaction Index, and Faithful Shapley Interaction Index in \cref{appendix:proofs:capi_extraction}.
\begin{lemma}\label{lemma:interval_decomposition}
For $A \subseteq B \subseteq N$, we have 
    $$
    \1_{[A,B]} = \sum_{T \subseteq N \setminus B} \; (-1)^{t}\1_{[T\cup A, N]}
    $$
\end{lemma}
\begin{proof}
    See proof of Lemma 1 in \citet{Zern.2023}.
\end{proof}
\begin{lemma}
\label{lemma:moebius_intervalgame}
    Let $N$ be the set of players, $A \subseteq B \subseteq N$, and the game $\1_{[A,B]}(T) = 1$ iff $A \subseteq T \subseteq B$.
    Then it holds that the Möbius value for a set $S\subseteq N$ equals
    \begin{equation*}
    m_S(\1_{[A,B]}) = (-1)^{|(N\setminus B )\cap S|} \cdot \1_{A \subseteq S, (B \setminus A) \cap S = \emptyset}
    \end{equation*}
\end{lemma}
\begin{proof}
    Let $N$ be the set of players, $A,B \subseteq N$ such that $A \subseteq B$, and the game $\1_{[A,B]}$ be defined as above, and $S \subseteq N$.
    Then 
    \begin{align*}
        m_S(\1_{[A,B]}) = \sum_{L \subseteq S} (-1)^{s-l} \1_{[A,B]}(L) \text{.}
    \end{align*}
    Note that we must have $A \subseteq S$, as otherwise the game value is always zero.
    We therefore define $F = (N \setminus B) \cap S$ and $G = (B \setminus A) \cap S$ such that $S = A \cup F \cup G$, which yields:
    \begin{align*}
        m_S(\1_{[A,B]}) = \sum_{L \subseteq G \cup F \cup A} (-1)^{f+g+a-l} \1_{[A,B]}(L) \text{.}
    \end{align*}
    Observe that only those $L$ have non-zero game values for which $A \subseteq L $ and $L \cap F = \emptyset$, as otherwise $L \subseteq B$ would not hold.
    Therefore, we can equivalently express the sum as 
    \begin{align*}
        m_S(\1_{[A,B]}) &= (-1)^{f+g} \sum_{L \subseteq G} (-1)^{-l}  \\
                &= (-1)^{f+g} \sum_{l=0}^g \binom{g}{l}(-1)^{-l}
        \text{.}
    \end{align*}
    The latter term always equals $0$ for $g \neq 0$, which yields our second condition $G = (B\setminus A) \cap S = \emptyset$.
    Finally, we have
    \begin{align*}
        m_S(\1_{[A,B]}) = (-1)^f = (-1)^{|(N\setminus B) \cap S|}
    \end{align*}
    if and only if (i) $A \subseteq S$ and $(B \setminus A) \cap S = \emptyset$, which concludes the proof.
\end{proof}

We now prove the main result of this section, which provides a closed-form expression for the cardinal-probabilistic interaction indices of tree-based proxies.

\textbf{Proposition \ref{prop_tree_capi}. }
 For $S \subseteq N$ and leaves $j \in \mathcal L$, we have
    \begin{align*}
        \phi^p_S(\hat\nu_{\text{tree}}) = \sum_{\substack{j \in \mathcal L \\ S \subseteq L_j \cup R_j}} c_j \cdot \lambda\left(\vert L_j \vert, \vert R_j\vert, \vert S \cap L_j \vert, \vert S \vert \right),
    \end{align*}
    where $\lambda(\ell,r,u,s) := \sum_{i=0}^{\ell-u} \; (-1)^{i+u} \binom{\ell-u}{i}q_{i+u+r}^s(n)$.
\begin{proof}[Proof of Proposition \ref{prop_tree_capi}]
    Let $\1_{[A,B]}: 2^N \rightarrow \{0,1\}$ be an indicator game such that $\1[A,B](T) =\1[A \subseteq T \subseteq B]$.
    We observe that $\hat{\nu}_{\text{tree}}$ can be written as a sum of indicator games :
    \begin{align*}
        \phi^p_S(\hat\nu_{\text{tree}}) 
        &= \phi^p_S\left(\sum_{j \in \mathcal L} c_j \cdot \1[R_j,  \ N \setminus L_j]\right) \\
        &= \sum_{j \in \mathcal L} c_j \cdot \phi^p_S\left( \1[R_j, \ N \setminus L_j]\right)
    \end{align*}
    Using Lemma \ref{lemma:interval_decomposition} we obtain:
    \begin{align*}
         \phi^p_S(\hat\nu_{\text{tree}}) &= \sum_{j \in \mathcal L} c_j\phi^p_S\left(\sum_{T \subseteq  L_j} \; (-1)^{t}\1_{[T\cup R_j, N]} \right) \\
                                          &= \sum_{j \in \mathcal L} c_j\sum_{T \subseteq L_j} \; (-1)^{t} \phi^p_S(\1_{[T\cup R_j, N]}) \\
                                          &= \sum_{j \in \mathcal L} c_j \sum_{T \subseteq L_j} \; (-1)^t \sum_{V \supseteq S} q_v^s(n) m_V(\1_{[T \cup R_j, N]}) \\
                                          &= \sum_{j \in \mathcal L} c_j \sum_{
                                          \substack{T \subseteq L_j \\ T \cup R_j \supseteq S}} \; (-1)^t q_{t + r_j}^s(n) 
    \end{align*}
    In the second-to-last step we use \eqref{eq_möbius_conversion}; in the last step we use $m_V(\1_{[R_j\cup T,N]}) = 1$ iff $R_j \cup T = V$ and $0$ otherwise (a direct consequence of $B = N$ in Lemma~\ref{lemma:moebius_intervalgame}).
    Notice that we can only update those interactions $S$ in leaf $j$ for which it holds $S \subseteq L_j \cup R_j$, since the condition $T \cup R_j \supseteq S$ does not hold otherwise, and as such $\phi^p_S(\1[R_j, N \setminus L_j]) = 0$. 
    We define $S_0 := S \cap L_j$ and let $B \subseteq L_j \setminus S_0$, so that $T = S_0 \cup B$, which gives rise to
    \begin{align*}
        \phi^p_S(\hat\nu_{\text{tree}}) &= \sum_{\substack{j \in \mathcal L \\ S \subseteq L_j \cup R_j}} c_j  \sum_{B \subseteq L_j \setminus S_0} \; (-1)^{b + s_0} q_{b+s_0+r_j}^s(n) \\
                                          &= \sum_{\substack{j \in \mathcal L \\ S \subseteq L_j \cup R_j}} c_j \sum_{i=0}^{\ell_j-s_0} \; (-1)^{i + s_0} \binom{\ell_j-s_0}{i}q_{i+s_0+r_j}^s(n) 
    \end{align*}
    Defining $\lambda(\ell,r,u,s):=\sum_{i=0}^{\ell-u} \; (-1)^{i+u} \binom{\ell-u}{i}q_{i+u+r}^s(n)$ we obtain:
    \begin{align*}
        \phi^p_S(\hat\nu_{\text{tree}}) &= \sum_{\substack{j \in \mathcal L \\ S \subseteq L_j \cup R_j}} c_j \lambda(|L_j|,|R_j|,|S \cap L_j|, |S|)
    \end{align*}
    which concludes the proof.
\end{proof}

\subsection{Proof of Theorem~\ref{th:var_msr}}
\label{appendix:proof:var_msr}
We first derive a general variance identity for the MSR estimator under arbitrary sampling distributions in \cref{th:msr_variance_identity}.
We then specialize this identity to proportional sampling and, most importantly, uniform over sizes for the Shapley interaction index.
The latter specialization yields Theorem~\ref{th:var_msr} from the main paper.
Corresponding lower bounds are shown in Corollary~\ref{cor:lower_bound_leverage}. The corresponding bound for the Banzhaf interaction index can be found in Corollary~\ref{cor:gamma_leverage_bii}.

\begin{theorem}[General variance identity for MSR]
\label{th:msr_variance_identity}
Let $S \subseteq N$ with $s:=|S|$, and let $T_1,\dots,T_{|\mathcal T|}$ be sampled i.i.d.\ from a sampling distribution $\Prob_{\mathrm{sampling}}$ on $2^N$ with $\Prob_{\mathrm{sampling}}(T)>0$ for all $T \subseteq N$.
Then the MSR estimator
\[
\hat\phi^{\MSR}_S(\nu;\mathcal T)
=
\frac{1}{|\mathcal T|}
\sum_{i=1}^{|\mathcal T|}
\nu(T_i)\,
\frac{(-1)^{s-|S\cap T_i|}\,p^s_{|T_i|-|S\cap T_i|}(n)}
{\Prob_{\mathrm{sampling}}(T_i)}
\]
satisfies
\[
\Var[\hat\phi^{\MSR}_S(\nu;\mathcal T)]
=
\frac{1}{|\mathcal T|}
\left(
\sum_{T \in 2^N}
\nu(T)^2
\frac{\bigl(p^s_{|T|-|S\cap T|}(n)\bigr)^2}
{\Prob_{\mathrm{sampling}}(T)}
-
\phi_S^p(\nu)^2
\right).
\]
In particular,
\[
\Var[\hat\phi^{\MSR}_S(\nu;\mathcal T)]
\le
\frac{\lVert \nu \rVert_\infty^2}{|\mathcal T|}
\Gamma_S(\Prob_{\mathrm{sampling}}),
\]
where
\[
\Gamma_S(\Prob_{\mathrm{sampling}})
:=
\sum_{T \in 2^N}
\frac{\bigl(p^s_{|T|-|S\cap T|}(n)\bigr)^2}
{\Prob_{\mathrm{sampling}}(T)}.
\]
\end{theorem}

\begin{proof}
Define the random variable
\[
X
:=
\nu(T)\,
\frac{(-1)^{s-|S\cap T|}\,p^s_{|T|-|S\cap T|}(n)}
{\Prob_{\mathrm{sampling}}(T)},
\qquad T \sim \Prob_{\mathrm{sampling}}.
\]
Since $T_1,\dots,T_{|\mathcal T|}$ are i.i.d., the corresponding variables
\[
X_i
:=
\nu(T_i)\,
\frac{(-1)^{s-|S\cap T_i|}\,p^s_{|T_i|-|S\cap T_i|}(n)}
{\Prob_{\mathrm{sampling}}(T_i)}
\]
are i.i.d.\ as well, and
\[
\hat\phi^{\MSR}_S(\nu;\mathcal T)
=
\frac{1}{|\mathcal T|}\sum_{i=1}^{|\mathcal T|} X_i.
\]
Using independence, we obtain
\begin{align}
\Var[\hat\phi^{\MSR}_S(\nu;\mathcal T)]
&=
\Var\!\left[
\frac{1}{|\mathcal T|}\sum_{i=1}^{|\mathcal T|} X_i
\right]
\notag\\
&=
\frac{1}{|\mathcal T|^2}
\sum_{i=1}^{|\mathcal T|}\Var[X_i]
\notag\\
&=
\frac{1}{|\mathcal T|}\Var[X].
\label{eq:msr_var_iid_appendix}
\end{align}
We now expand the variance of $X$:
\[
\Var[X]
=
\mathbb E[X^2]-\mathbb E[X]^2.
\]

For the second moment, we compute
\begin{align}
\mathbb E[X^2]
&=
\sum_{T \in 2^N}
\Prob_{\mathrm{sampling}}(T)\,
\nu(T)^2
\left(
\frac{p^s_{|T|-|S\cap T|}(n)}
{\Prob_{\mathrm{sampling}}(T)}
\right)^2
\notag\\
&=
\sum_{T \in 2^N}
\nu(T)^2
\frac{\bigl(p^s_{|T|-|S\cap T|}(n)\bigr)^2}
{\Prob_{\mathrm{sampling}}(T)},
\label{eq:msr_second_moment_appendix}
\end{align}
where the sign disappears since it is squared.
For the first moment,
\begin{align}
\mathbb E[X]
&=
\sum_{T \in 2^N}
\Prob_{\mathrm{sampling}}(T)\,
\nu(T)\,
\frac{(-1)^{s-|S\cap T|}\,p^s_{|T|-|S\cap T|}(n)}
{\Prob_{\mathrm{sampling}}(T)}
\notag\\
&=
\sum_{T \in 2^N}
\nu(T)\,
(-1)^{s-|S\cap T|}\,
p^s_{|T|-|S\cap T|}(n)
\notag\\
&=
\phi_S^p(\nu).
\label{eq:msr_first_moment_appendix}
\end{align}
Substituting \eqref{eq:msr_second_moment_appendix} and \eqref{eq:msr_first_moment_appendix} into the variance expansion yields
\[
\Var[X]
=
\sum_{T \in 2^N}
\nu(T)^2
\frac{\bigl(p^s_{|T|-|S\cap T|}(n)\bigr)^2}
{\Prob_{\mathrm{sampling}}(T)}
-
\phi_S^p(\nu)^2.
\]
Combining this with \eqref{eq:msr_var_iid_appendix} proves the exact variance identity.
The upper bound follows from $\phi_S^p(\nu)^2 \ge 0$ and $\nu(T)^2 \le \lVert \nu \rVert_\infty^2$ for all $T \subseteq N$.
\end{proof}

\begin{corollary}[Variance of MSR under proportional sampling]
\label{cor:var_msr_proportional}
Under the assumptions of Theorem~\ref{th:msr_variance_identity}, suppose the coalitions are sampled proportionally to the interaction weights, i.e.
\[
\Prob_{\mathrm{sampling}}(T)
\propto
p^s_{|T|-|S\cap T|}(n),
\qquad \text{for all } T \subseteq N.
\]
Then
\[
\Var[\hat\phi^{\MSR}_S(\nu;\mathcal T)]
\le
\frac{4^s}{|\mathcal T|}\,\|\nu\|_\infty^2.
\]
\end{corollary}

\begin{proof}
By Theorem~\ref{th:msr_variance_identity},
\[
\Var[\hat\phi^{\MSR}_S(\nu;\mathcal T)]
\le
\frac{\|\nu\|_\infty^2}{|\mathcal T|}\,
\Gamma_S(\Prob_{\mathrm{sampling}}).
\]
Under proportional sampling,
\[
\Prob_{\mathrm{sampling}}(T)
=
\frac{p^s_{|T|-|S\cap T|}(n)}{2^s},
\]
and hence
\begin{align*}
\Gamma_S(\Prob_{\mathrm{sampling}})
&=
\sum_{T\subseteq N}
\frac{\bigl(p^s_{|T|-|S\cap T|}(n)\bigr)^2}
{\Prob_{\mathrm{sampling}}(T)} \\
&=
\sum_{T\subseteq N}
\Prob_{\mathrm{sampling}}(T)
\left(
\frac{p^s_{|T|-|S\cap T|}(n)}
{\Prob_{\mathrm{sampling}}(T)}
\right)^2 \\
&=
\sum_{T\subseteq N}
\Prob_{\mathrm{sampling}}(T)\,(2^s)^2 \\
&=
4^s \sum_{T\subseteq N}\Prob_{\mathrm{sampling}}(T) \\
&=
4^s.
\end{align*}
Substituting this into the general bound proves the claim.
\end{proof}

\begin{corollary}[Asymptotics of uniform over sizes sampling variance factor for SII]
\label{cor:gamma_leverage_sii}
Suppose \(\Prob_{\mathrm{sampling}}\) is given by,
\[
\Prob_{\mathrm{sampling}}(T)=\frac{1}{(n+1)\binom{n}{|T|}},
\]
and let \(p_t^s(n)\) denote the SII weights,
\[
p_t^s(n)=\frac{1}{(n-s+1)\binom{n-s}{t}}.
\]
Then, for fixed interaction order \(s\),
\[
\Gamma_S(\Prob_{\mathrm{sampling}})
=
\begin{cases}
O(\log n), & s=1,\\
O(n^{s-1}), & s\ge 2.
\end{cases}
\]
\end{corollary}

\begin{proof}
We rewrite \(\Gamma_S(\Prob_{\mathrm{sampling}})\) as a double sum over the overlap size \(r=|S\cap T|\) and the outside size \(q=|T\setminus S|\), then treat the cases \(s=1\) and \(s\ge 2\) separately.

\paragraph{Step 1: Rewriting \(\Gamma_S(\Prob_{\mathrm{sampling}})\).}
Using the definition of \(\Gamma_S(\Prob_{\mathrm{sampling}})\) from Theorem~\ref{th:msr_variance_identity}, together with the SII weights and the sampling probabilities, we obtain
\begin{align*}
\Gamma_S(\Prob_{\mathrm{sampling}})
&=
\sum_{T\subseteq N}
\frac{\bigl(p^s_{|T|-|S\cap T|}(n)\bigr)^2}
{\Prob_{\mathrm{sampling}}(T)} \\
&=
\sum_{T\subseteq N}
\frac{1}{(n-s+1)^2\binom{n-s}{|T\setminus S|}^2}
\cdot
(n+1)\binom{n}{|T|}.
\end{align*}

We group coalitions \(T\subseteq N\) according to
\[
r:=|S\cap T|,
\qquad
q:=|T\setminus S|.
\]
Then \(|T|=r+q\), and for fixed \((r,q)\) there are exactly
\[
\binom{s}{r}\binom{n-s}{q}
\]
coalitions with these values. Hence
\[
\Gamma_S(\Prob_{\mathrm{sampling}})
=
\frac{n+1}{(n-s+1)^2}
\sum_{r=0}^s \binom{s}{r}
\sum_{q=0}^{n-s}
\frac{\binom{n}{r+q}}{\binom{n-s}{q}}.
\]

Next, we rewrite the binomial ratio:
\begin{align*}
\frac{\binom{n}{r+q}}{\binom{n-s}{q}}
&=
\frac{n!}{(r+q)!(n-r-q)!}\cdot\frac{q!(n-s-q)!}{(n-s)!} \\
&=
\frac{n!}{(n-s)!}\cdot\frac{q!}{(r+q)!}\cdot\frac{(n-s-q)!}{(n-r-q)!}.
\end{align*}
Using the rising factorial
\[
(a)_k := a(a+1)\cdots(a+k-1),
\]
this becomes
\[
\frac{\binom{n}{r+q}}{\binom{n-s}{q}}
=
\frac{(n-s+1)_s}{(q+1)_r\,(n-s-q+1)_{s-r}}.
\]
Therefore,
\begin{align}
\Gamma_S(\Prob_{\mathrm{sampling}})
=
\frac{(n+1)(n-s+1)_s}{(n-s+1)^2}
\sum_{r=0}^s \binom{s}{r}
\sum_{q=0}^{n-s}
\frac{1}{(q+1)_r\,(n-s-q+1)_{s-r}}.
\label{eq:gamma_upper_reduced}
\end{align}

\paragraph{Step 2: The case \(s=1\).}
Assume first that \(s=1\). Then \(r\in\{0,1\}\), so \eqref{eq:gamma_upper_reduced} becomes
\[
\Gamma_S(\Prob_{\mathrm{sampling}})
=
\frac{n+1}{n}
\sum_{r=0}^1 \binom{1}{r}
\sum_{q=0}^{n-1}
\frac{1}{(q+1)_r\,(n-q)_{1-r}}.
\]
We used \((n-s+1)_s=(n)_1=n\) and \((n-s+1)^2=n^2\) when \(s=1\).
For \(r=0\),
\[
\sum_{q=0}^{n-1}\frac{1}{n-q}
=
\sum_{j=1}^{n}\frac{1}{j}
=
H_n,
\]
and similarly for \(r=1\),
\[
\sum_{q=0}^{n-1}\frac{1}{q+1}
=
\sum_{j=1}^{n}\frac{1}{j}
=
H_n.
\]
Hence
\[
\Gamma_S(\Prob_{\mathrm{sampling}})
=
\frac{2(n+1)}{n}H_n
=
O(\log n).
\]

\paragraph{Step 3: The case \(s\ge 2\).}
Now let \(s\ge 2\) be fixed. In \eqref{eq:gamma_upper_reduced}, we bound the inner sum via
\[
(q+1)_r \ge (q+1)^r,
\qquad
(n-s-q+1)_{s-r}\ge (n-s-q+1)^{s-r},
\]
which yields
\[
\sum_{q=0}^{n-s}
\frac{1}{(q+1)_r\,(n-s-q+1)_{s-r}}
\le
\sum_{q=0}^{n-s}
\frac{1}{(q+1)^r\,(n-s-q+1)^{s-r}}.
\]

We distinguish two cases.

\emph{Case 1: \(r=0\) or \(r=s\).}
If \(r=0\), then
\[
\sum_{q=0}^{n-s}
\frac{1}{(n-s-q+1)^s}
=
\sum_{j=1}^{n-s+1}\frac{1}{j^s}
=
O(1),
\]
since for every $s \ge 2$ and every $j \ge 1$ we have $j^{-s} \le j^{-2}$, so the sum is bounded by $\sum_{j \ge 1} j^{-2} = \pi^2/6$, a constant that holds uniformly in both $n$ and $s$.
The same argument applies to \(r=s\), giving
\[
\sum_{q=0}^{n-s}
\frac{1}{(q+1)^s}
=
O(1).
\]

\emph{Case 2: \(1\le r\le s-1\).}
In this range, both exponents are at least \(1\), and therefore
\[
\frac{1}{(q+1)^r\,(n-s-q+1)^{s-r}}
\le
\frac{1}{(q+1)(n-s-q+1)}.
\]
Writing \(A:=n-s\), we get
\begin{align*}
\sum_{q=0}^{A}\frac{1}{(q+1)(A-q+1)}
&=
\frac{1}{A+2}\sum_{q=0}^{A}
\left(\frac{1}{q+1}+\frac{1}{A-q+1}\right) \\
&=
\frac{2H_{A+1}}{A+2}
=
O\!\left(\frac{\log n}{n}\right).
\end{align*}
Here we used
\[
\frac{1}{(q+1)(A-q+1)}
=
\frac{1}{A+2}\left(\frac{1}{q+1}+\frac{1}{A-q+1}\right).
\]

Combining the two cases, we find that for every fixed \(r\in\{0,\dots,s\}\),
\[
\sum_{q=0}^{n-s}
\frac{1}{(q+1)_r\,(n-s-q+1)_{s-r}}
=
\begin{cases}
O(1), & r=0 \text{ or } r=s,\\[4pt]
O\!\left(\dfrac{\log n}{n}\right), & 1\le r\le s-1.
\end{cases}
\]
Thus, the dominant contributions come from the boundary cases \(r=0\) and \(r=s\), while all interior contributions are asymptotically smaller. Therefore
\[
\sum_{r=0}^s \binom{s}{r}
\sum_{q=0}^{n-s}
\frac{1}{(q+1)_r\,(n-s-q+1)_{s-r}}
=
O(1).
\]
By \eqref{eq:gamma_upper_reduced},
\[
\Gamma_S(\Prob_{\mathrm{sampling}})
=
O\!\left(
\frac{(n+1)(n-s+1)_s}{(n-s+1)^2}
\right).
\]
Finally, for fixed \(s\),
\[
\frac{(n+1)(n-s+1)_s}{(n-s+1)^2}
=
O(n^{s-1}),
\]
and therefore
\[
\Gamma_S(\Prob_{\mathrm{sampling}})
=
O(n^{s-1}).
\]
\end{proof}

\begin{proof}[Proof of Theorem~\ref{th:var_msr}]
By Theorem~\ref{th:msr_variance_identity},
\[
\Var[\hat\phi^{\MSR}_S(\nu;\mathcal T)]
\le
\frac{\|\nu\|_\infty^2}{|\mathcal T|}\,
\Gamma_S(\Prob_{\mathrm{sampling}}).
\]
Under uniform sampling over sizes and SII weights, Corollary~\ref{cor:gamma_leverage_sii} yields
\[
\Gamma_S(\Prob_{\mathrm{sampling}})
=
\begin{cases}
O(\log n), & |S|=1,\\
O(n^{|S|-1}), & |S|\ge 2.
\end{cases}
\]
Substituting this into the general bound gives
\[
\Var[\hat\phi^{\MSR}_S(\nu;\mathcal T)]
\le
\begin{cases}
\mathcal{O}\!\left(\dfrac{\|\nu\|_\infty^2 \log n}{|\mathcal T|}\right), & |S|=1,\\[8pt]
\mathcal{O}\!\left(\dfrac{\|\nu\|_\infty^2 n^{|S|-1}}{|\mathcal T|}\right), & |S|\geq 2,
\end{cases}
\]
where the constants hidden in the \(\mathcal O(\cdot)\) notation depend only on \(|S|\).
This proves Theorem~\ref{th:var_msr}.
\end{proof}

\begin{lemma}[Lower bound on $\Gamma_S$ for SII under uniform over sizes]\label{lemma:lower_bound_gamma_sii}
Suppose $\mathbb{P}_{\text{sampling}}$ is uniform over sizes and $p^s_t(n)$ are the SII weights. Then, for fixed interaction order $s$,
\[
\Gamma_S(\mathbb{P}_{\text{sampling}}) =
\begin{cases}
\Theta(\log n), & s = 1, \\
\Theta(n^{s-1}), & s \ge 2.
\end{cases}
\]
\end{lemma}

\begin{proof}
Recall from the proof of Corollary~\ref{cor:gamma_leverage_sii} that
\[
\Gamma_S(\mathbb{P}_{\text{sampling}})
=
\frac{(n+1)(n-s+1)_s}{(n-s+1)^2}
\sum_{r=0}^{s} \binom{s}{r}
\sum_{q=0}^{n-s} \frac{1}{(q+1)_r\,(n-s-q+1)_{s-r}}.
\]
Since every term of the double sum is non-negative, we may lower-bound it by retaining only the boundary term $r = 0$:
\[
\Gamma_S(\mathbb{P}_{\text{sampling}})
\ge
\frac{(n+1)(n-s+1)_s}{(n-s+1)^2}
\sum_{q=0}^{n-s} \frac{1}{(n-s-q+1)_s}
=
\frac{(n+1)(n-s+1)_s}{(n-s+1)^2}
\sum_{j=1}^{n-s+1} \frac{1}{(j)_s}.
\]

\textbf{Case $s=1$.} Then $(j)_1 = j$ and
\[
\sum_{j=1}^{n} \frac{1}{j} = H_n = \Theta(\log n),
\]
so $\Gamma_S(\mathbb{P}_{\text{sampling}}) \ge \tfrac{n+1}{n} H_n = \Theta(\log n)$.

\textbf{Case $s \ge 2$.} The sum $\sum_{j=1}^{n-s+1} 1/(j)_s$ converges to a positive constant as $n \to \infty$ (its $j=1$ term alone equals $1/s!$), so it is $\Theta(1)$. Combined with $(n-s+1)_s = \Theta(n^s)$ and $(n-s+1)^2 = \Theta(n^2)$ for fixed $s$, we obtain
\[
\Gamma_S(\mathbb{P}_{\text{sampling}})
=
\Theta\!\left(\frac{n \cdot n^s}{n^2}\right)
=
\Theta(n^{s-1}).
\]
\end{proof}

\begin{corollary}[Matching lower bound on MSR variance for SII uniform over sizes]\label{cor:lower_bound_leverage}
Suppose coalitions are sampled i.i.d.\ according to:
\[
\mathbb{P}_{\text{sampling}}(T) = \frac{1}{(n+1)\binom{n}{|T|}}.
\]
Then there exists a value function $\nu$ such that
\[
\mathbb{V}\!\left[\hat{\phi}^{\mathrm{MSR}}_S(\nu;\mathcal{T})\right]
=
\begin{cases}
\Theta\!\left(\dfrac{\|\nu\|_\infty^2 \log n}{|\mathcal{T}|}\right), & |S| = 1, \\[4pt]
\Theta\!\left(\dfrac{\|\nu\|_\infty^2\, n^{|S|-1}}{|\mathcal{T}|}\right), & |S| \ge 2.
\end{cases}
\]
\end{corollary}

\begin{proof}
By Theorem~\ref{th:msr_variance_identity},
\[
\mathbb{V}\!\left[\hat{\phi}^{\mathrm{MSR}}_S(\nu;\mathcal{T})\right]
=
\frac{1}{|\mathcal{T}|}\!\left(
\sum_{T \in 2^N} \nu(T)^2\,\frac{\bigl(p^s_{|T|-|S\cap T|}(n)\bigr)^2}{\mathbb{P}_{\text{sampling}}(T)}
\;-\; \phi_S^p(\nu)^2
\right).
\]
Take $\nu \equiv 1$, so $\phi_S^p(\nu) = 0$ for any $|S| \ge 1$ by the dummy axiom. The variance then equals $\Gamma_S(\mathbb{P}_{\text{sampling}})/|\mathcal{T}| = \|\nu\|_\infty^2\,\Gamma_S(\mathbb{P}_{\text{sampling}})/|\mathcal{T}|$, and the claim follows from the Lemma~\ref{lemma:lower_bound_gamma_sii} above.
\end{proof}

\begin{corollary}[Variance of MSR under uniform over sizes for BII]
\label{cor:gamma_leverage_bii}
Suppose \(\Prob_{\mathrm{sampling}}\) is given by uniform over sizes,
\[
\Prob_{\mathrm{sampling}}(T)=\frac{1}{(n+1)\binom{n}{|T|}},
\]
and let \(p_t^s(n)\) denote the Banzhaf Interaction Index weights,
\[
p_t^s(n)=\frac{1}{2^{\,n-s}}.
\]
Then
\[
\Gamma_S(\Prob_{\mathrm{sampling}})
=
\frac{n+1}{4^{\,n-s}}\binom{2n}{n}.
\]
In particular,
\[
\Gamma_S(\Prob_{\mathrm{sampling}})
=
\Theta(4^s \sqrt{n}),
\]
and therefore, by Theorem~\ref{th:msr_variance_identity},
\[
\Var[\hat\phi^{\MSR}_S(\nu;\mathcal T)]
\le
\frac{\|\nu\|_\infty^2}{|\mathcal T|}\,
\Gamma_S(\Prob_{\mathrm{sampling}})
=
O\!\left(\frac{\|\nu\|_\infty^2 4^s \sqrt n}{|\mathcal T|}\right).
\]
In particular, for fixed interaction order \(s\),
\[
\Var[\hat\phi^{\MSR}_S(\nu;\mathcal T)]
=
O\!\left(\frac{\|\nu\|_\infty^2 \sqrt n}{|\mathcal T|}\right).
\]
\end{corollary}

\begin{proof}
By Theorem~\ref{th:msr_variance_identity},
\[
\Var[\hat\phi^{\MSR}_S(\nu;\mathcal T)]
\le
\frac{\|\nu\|_\infty^2}{|\mathcal T|}\,
\Gamma_S(\Prob_{\mathrm{sampling}}),
\]
where
\[
\Gamma_S(\Prob_{\mathrm{sampling}})
:=
\sum_{T\subseteq N}
\frac{\bigl(p^s_{|T|-|S\cap T|}(n)\bigr)^2}
{\Prob_{\mathrm{sampling}}(T)}.
\]

For the Banzhaf Interaction Index, the weights are constant in the coalition size, namely
\[
p_t^s(n)=\frac{1}{2^{\,n-s}}.
\]
Hence
\begin{align*}
\Gamma_S(\Prob_{\mathrm{sampling}})
&=
\sum_{T\subseteq N}
\frac{1}{4^{\,n-s}}
\frac{1}{\Prob_{\mathrm{sampling}}(T)} \\
&=
\frac{1}{4^{\,n-s}}
\sum_{T\subseteq N}
(n+1)\binom{n}{|T|}.
\end{align*}

We now group coalitions \(T\subseteq N\) by their size \(t:=|T|\).
For each \(t\in\{0,\dots,n\}\), there are exactly \(\binom{n}{t}\) coalitions of size \(t\).
Therefore
\begin{align*}
\sum_{T\subseteq N}\binom{n}{|T|}
&=
\sum_{t=0}^n \binom{n}{t}\binom{n}{t} \\
&=
\sum_{t=0}^n \binom{n}{t}^2.
\end{align*}
Using Vandermonde's identity,
\[
\sum_{t=0}^n \binom{n}{t}^2
=
\binom{2n}{n}.
\]
Substituting this into the previous display yields the exact formula
\[
\Gamma_S(\Prob_{\mathrm{sampling}})
=
\frac{n+1}{4^{\,n-s}}\binom{2n}{n}.
\]

To obtain the asymptotic form, we use the central binomial coefficient estimate
\[
\binom{2n}{n}
=
\Theta\!\left(\frac{4^n}{\sqrt n}\right).
\]
Hence
\[
\Gamma_S(\Prob_{\mathrm{sampling}})
=
\frac{n+1}{4^{\,n-s}}\binom{2n}{n}
=
\Theta\!\left(
\frac{n+1}{4^{\,n-s}}\cdot \frac{4^n}{\sqrt n}
\right)
=
\Theta(4^s \sqrt n).
\]

Finally, substituting this into the general variance bound from Theorem~\ref{th:msr_variance_identity} gives
\[
\Var[\hat\phi^{\MSR}_S(\nu;\mathcal T)]
\le
\frac{\|\nu\|_\infty^2}{|\mathcal T|}\,
\Gamma_S(\Prob_{\mathrm{sampling}})
=
O\!\left(\frac{\|\nu\|_\infty^2 4^s \sqrt n}{|\mathcal T|}\right).
\]
For fixed interaction order \(s\), the factor \(4^s\) is absorbed into the constant, yielding
\[
\Var[\hat\phi^{\MSR}_S(\nu;\mathcal T)]
=
O\!\left(\frac{\|\nu\|_\infty^2 \sqrt n}{|\mathcal T|}\right).
\]
This proves the claim.
\end{proof}

\clearpage

\subsection{Closed Forms for Proposition \ref{prop_tree_capi}}
\label{appendix:proofs:capi_extraction}

\subsubsection{Shapley Value and Shapley Interaction Index}

\begin{proposition}\label{prop:shapley}
The Shapley Value and the Shapley Interaction Index yield the closed form
\[
\lambda(|L_j|,|R_j|,|L_j\cap S|,|S|)
=
(-1)^u \frac{1}{(a+b+1)\binom{a+b}{a}},
\]
where
\[
u := |L_j\cap S|,
\qquad
a := |R_j|-|R_j\cap S|,
\qquad
b := |L_j|-|L_j\cap S|.
\]
\end{proposition}

\begin{proof}
This is exactly Proposition~1 in \citet{Zern.2023}, specialized to our notation.
\end{proof}

\subsubsection{Banzhaf Value and Banzhaf Interaction Index}

\begin{proposition}\label{prop:banzhaf}
The Banzhaf Value and the Banzhaf Interaction Index yield the closed form
\[
\lambda(\ell,r,u,s)
=
(-1)^u \frac{1}{2^{\ell+r-s}}.
\]
\end{proposition}

\begin{proof}
Starting from Proposition~\ref{prop_tree_capi} and using the Möbius weights for the Banzhaf interaction index from \cref{tab_capi_weights}, we obtain
\[
\lambda(\ell,r,u,s)
=
\sum_{i=0}^{\ell-u}
(-1)^{i+u}\binom{\ell-u}{i}\frac{1}{2^{i+u+r-s}}.
\]
Define
\[
k_0 := u+r-s,
\qquad
m := \ell-u.
\]
Then
\begin{align*}
\lambda(\ell,r,u,s)
&=
\sum_{i=0}^{m}
(-1)^{i+u}\binom{m}{i}\frac{1}{2^{i+k_0}} \\
&=
(-1)^u\frac{1}{2^{k_0}}
\sum_{i=0}^{m}
(-1)^i\binom{m}{i}\frac{1}{2^i} \\
&=
(-1)^u\frac{1}{2^{k_0}}
\sum_{i=0}^{m}
\binom{m}{i}\left(-\frac{1}{2}\right)^i \\
&=
(-1)^u\frac{1}{2^{k_0}}
\sum_{i=0}^{m}
\binom{m}{i}\left(-\frac{1}{2}\right)^i 1^{m-i}.
\end{align*}
By the binomial theorem,
\[
\sum_{i=0}^{m}
\binom{m}{i}\left(-\frac{1}{2}\right)^i 1^{m-i}
=
\left(1-\frac{1}{2}\right)^m
=
\left(\frac{1}{2}\right)^m.
\]
Hence
\[
\lambda(\ell,r,u,s)
=
(-1)^u\frac{1}{2^{k_0}}\left(\frac{1}{2}\right)^m
=
(-1)^u\frac{1}{2^{k_0+m}}.
\]
Substituting the definitions of \(k_0\) and \(m\) gives
\[
\lambda(\ell,r,u,s)
=
(-1)^u\frac{1}{2^{\ell+r-s}}.
\]
\end{proof}
%
%
\subsubsection{Chaining Value and Chaining Interaction Index}

\begin{proposition}\label{prop:chaining}
The Chaining Value and the Chaining Interaction Index yield the closed form
\[
\lambda(\ell,r,u,s)
=
s(-1)^u B(u+r,\ell-u+1),
\]
where
\[
B(z_1,z_2) := \int_0^1 x^{z_1-1}(1-x)^{z_2-1}\,dx
\]
denotes the Beta function.
\end{proposition}

\begin{proof}
For the Chaining Interaction Index, the Möbius weights satisfy \(q_t^s(n)=\frac{s}{t}\) \citep{Fujimoto.2006}. Therefore, Proposition~\ref{prop_tree_capi} yields
\[
\lambda(\ell,r,u,s)
=
\sum_{i=0}^{\ell-u}
(-1)^{i+u}\binom{\ell-u}{i}\frac{s}{i+u+r}.
\]
Factor out the constant terms:
\[
\lambda(\ell,r,u,s)
=
(-1)^u s
\sum_{i=0}^{\ell-u}
(-1)^i\binom{\ell-u}{i}\frac{1}{i+u+r}.
\]
Now define
\[
k_0 := u+r,
\qquad
m := \ell-u,
\]
and use the identity
\[
\frac{1}{i+k_0}=\int_0^1 x^{i+k_0-1}\,dx.
\]
Then
\begin{align*}
\lambda(\ell,r,u,s)
&=
(-1)^u s
\sum_{i=0}^{m}
(-1)^i\binom{m}{i}\int_0^1 x^{i+k_0-1}\,dx \\
&=
(-1)^u s
\int_0^1
\sum_{i=0}^{m}
(-1)^i\binom{m}{i}x^i x^{k_0-1}\,dx \\
&=
(-1)^u s
\int_0^1
\left(
\sum_{i=0}^{m}
\binom{m}{i}(-x)^i 1^{m-i}
\right)x^{k_0-1}\,dx \\
&=
(-1)^u s
\int_0^1
(1-x)^m x^{k_0-1}\,dx \\
&=
(-1)^u s\, B(k_0,m+1).
\end{align*}
Substituting back \(k_0=u+r\) and \(m=\ell-u\) gives
\[
\lambda(\ell,r,u,s)
=
(-1)^u s\, B(u+r,\ell-u+1).
\]
\end{proof}
\subsubsection{Faithful Banzhaf Interaction Index (FBII)}

\begin{proposition}\label{prob:fbii_closed_form}
For the Faithful Banzhaf Interaction Index (FBII), Proposition~\ref{prop_tree_capi} yields the closed form
\[
\phi_S^{\mathrm{FBII}}(\hat{\nu}_{\textnormal{tree}})
=
\sum_{\substack{j \in \mathcal L \\ S \subseteq L_j \cup R_j}}
c_j
\lambda\bigl(|L_j|,|R_j|,|S\cap L_j|,|S|\bigr)
,
\]
where
\begin{align*}
\lambda(\ell,r,u,s)
={}& (-1)^u \1[r = s - u] \\
&+ \sum_{i=\max(0,k-r-u+1)}^{\ell-u}
(-1)^{u+i+k-s}
2^{-(r+i+u-s)}
\binom{\ell-u}{i}
\binom{r+i+u-s-1}{k-s}.
\end{align*}
and \(k\) denotes the maximal interaction order to be computed.
\end{proposition}

\begin{proof}
We begin with the Möbius representation of FBII \citep{Tsai.2022}:
\[
\phi_S^{\mathrm{FBII}}(\nu)
=
m_S(\nu)
+
(-1)^{k-s}
\sum_{\substack{T\supseteq S \\ |T|>k}}
\left(\frac12\right)^{|T|-s}
\binom{|T|-s-1}{k-s}
m_T(\nu).
\]
Applying Lemma~\ref{lemma:interval_decomposition} to the tree representation yields
\begin{align*}
\phi_S^{\mathrm{FBII}}(\hat\nu_{\text{tree}})
&=
\sum_{j\in\mathcal L}
c_j\,
\phi_S^{\mathrm{FBII}}
\left(
\sum_{T\subseteq L_j}(-1)^{|T|}\1_{[T\cup R_j,N]}
\right) \\
&=
\sum_{j\in\mathcal L}
c_j
\sum_{T\subseteq L_j}
(-1)^{|T|}
\phi_S^{\mathrm{FBII}}(\1_{[T\cup R_j,N]}).
\end{align*}
We now analyze a fixed leaf \(j\). For notational brevity, write
\[
r := |R_j|,
\qquad
\ell := |L_j|,
\qquad
u := |S\cap L_j|,
\qquad
s := |S|.
\]
Using the Möbius representation above, we split the expression into two parts:
\begin{align*}
&\sum_{T\subseteq L_j}(-1)^{|T|}
\phi_S^{\mathrm{FBII}}(\1_{[T\cup R_j,N]}) \\
&\qquad=
\sum_{T\subseteq L_j}(-1)^{|T|}
m_S(\1_{[R_j\cup T,N]}) \\
&\qquad\quad+
(-1)^{k-s}
\sum_{T\subseteq L_j}(-1)^{|T|}
\sum_{\substack{V\supseteq S \\ |V|>k}}
\left(\frac12\right)^{|V|-s}
\binom{|V|-s-1}{k-s}
m_V(\1_{[R_j\cup T,N]}).
\end{align*}

\paragraph{Step 1: The Möbius term.}
Consider
\[
\sum_{T\subseteq L_j}(-1)^{|T|}
m_S(\1_{[R_j\cup T,N]}).
\]
By Lemma~\ref{lemma:moebius_intervalgame} with \(B=N\), we have
\[
m_S(\1_{[R_j\cup T,N]})=1
\quad\Longleftrightarrow\quad
R_j\cup T = S,
\]
and otherwise it is zero.
Moreover, only interactions with \(S\subseteq L_j\cup R_j\) can contribute.
Under this condition, we must have \(R_j\subseteq S\), hence
\begin{align*}
\sum_{T\subseteq L_j}(-1)^{|T|}m_S(\1_{[R_j\cup T,N]})
&=
\sum_{T\subseteq L_j}(-1)^{|T|}\1[R_j\cup T=S] \\
&=
(-1)^u\1[R_j \subseteq S].
\end{align*}

\paragraph{Step 2: The faithful tail term.}
Now consider
\begin{align*}
&(-1)^{k-s}
\sum_{T\subseteq L_j}(-1)^{|T|}
\sum_{\substack{V\supseteq S \\ |V|>k}}
\left(\frac12\right)^{|V|-s}
\binom{|V|-s-1}{k-s}
m_V(\1_{[R_j\cup T,N]}) \\
&\qquad=
(-1)^{k-s}
\sum_{\substack{V\supseteq S \\ |V|>k}}
\left(\frac12\right)^{|V|-s}
\binom{|V|-s-1}{k-s}
\sum_{T\subseteq L_j}(-1)^{|T|}
m_V(\1_{[R_j\cup T,N]}).
\end{align*}
Again by Lemma~\ref{lemma:moebius_intervalgame} with \(B=N\),
\[
m_V(\1_{[R_j\cup T,N]})=1
\quad\Longleftrightarrow\quad
R_j\cup T=V.
\]
Since the outer sum only contains \(V\) with \(|V|>k\), only subsets \(T\) satisfying \(|R_j\cup T|>k\) contribute, i.e.
\[
|T|>k-r.
\]
Furthermore, we must have \(S\subseteq R_j\cup T\), which is equivalent to \(S\cap L_j \subseteq T\).
Therefore, the tail term becomes
\[
\sum_{\substack{S \cap L_j \subseteq T\subseteq L_j \\ |T|>k-r}}
(-1)^{|T|+k-s}
\left(\frac12\right)^{r+|T|-s}
\binom{r+|T|-s-1}{k-s}.
\]

Write
\[
S_L := L_j\cap S,
\qquad
S_R \subseteq L_j\setminus S_L,
\qquad
T=S_L\cup S_R.
\]
Then \(|T|=|S_L|+|S_R|\), and summing over all possible \(S_R\) gives
\[
\sum_{\substack{S_R\subseteq L_j\setminus S_L \\ |S_R|>k-r-u}}
(-1)^{|S_R|+u+k-s}
\left(\frac12\right)^{r+|S_R|+u-s}
\binom{r+|S_R|+u-s-1}{k-s}.
\]
Grouping terms by \(i:=|S_R|\), there are \(\binom{\ell-u}{i}\) such subsets. Hence, the tail term simplifies to
\[
\sum_{i=k-r-u+1}^{\ell-u}
(-1)^{u+i+k-s}
\left(\frac12\right)^{r+i+u-s}
\binom{r+i+u-s-1}{k-s}
\binom{\ell-u}{i}.
\]

\paragraph{Step 3: Combine both parts.}
When $k - r - u + 1 \le 0$, the constraint $|T| > k - r$ is satisfied by every $S \cap L_j \subseteq T \subseteq L_j$, so the lower summation index collapses to $0$; as such, the lower index can be written as $\max(0,\,k-r-u+1)$ as in the proposition statement.
Combining the Möbius term and the faithful tail term, summing over all leaves, and incorporating the condition \(S\subseteq L_j\cup R_j\), we obtain
\[
\phi_S^{\mathrm{FBII}}(\hat{\nu}_{\textnormal{tree}})
=
\sum_{\substack{j \in \mathcal L \\ S \subseteq L_j \cup R_j}}
c_j\cdot\lambda\bigl(|L_j|,|R_j|,|S\cap L_j|,|S|\bigr),
\]
where 
\begin{align*}
    \lambda(|L_j|,|R_j|,|S\cap L_j|,|S|)
    =
    &(-1)^{|S\cap L_j|}\1[R_j \subseteq S] \\ + &\sum_{i=\max(0,\,k-r-u+1)}^{\ell-u}
    (-1)^{u+i+k-s}
    \left(\frac12\right)^{r+i+u-s}
    \binom{r+i+u-s-1}{k-s}
    \binom{\ell-u}{i}
\end{align*}
yields the claimed formula.
\end{proof}

\subsubsection{Faithful Shapley Interaction Index (FSII)}

\begin{proposition}\label{prob:fsii_closed_form}
For the Faithful Shapley Interaction Index (FSII), Proposition~\ref{prop_tree_capi} yields the closed form
\[
\phi_S^{\mathrm{FSII}}(\hat{\nu}_{\textnormal{tree}})
=
\sum_{\substack{j \in \mathcal L \\ S \subseteq L_j \cup R_j}}
c_j\lambda\bigl(|L_j|,|R_j|,|S\cap L_j|,|S|\bigr),
\]
where $u:=|S \cap L_j|$ in
\begin{align*}
    \lambda(\ell,r,u,s)
=
&(-1)^{u}\1[r=s-u] \\ + &\sum_{i=\max(0,k-r-u+1)}^{\ell-u}
(-1)^{u+i+k-s}
\frac{s}{k+s}
\binom{k}{s}
\binom{\ell-u}{i}
\frac{\binom{r+i+u-1}{k}}
{\binom{r+u+i+k-1}{k+s}} 
\end{align*}

\end{proposition}

\begin{proof}
We use the Möbius representation of FSII \citep{Tsai.2022}:
\[
\phi_S^{\mathrm{FSII}}(\nu)
=
m_S(\nu)
+
(-1)^{k-s}\frac{s}{k+s}\binom{k}{s}
\sum_{\substack{T\supseteq S \\ |T|>k}}
\frac{\binom{|T|-1}{k}}
{\binom{|T|+k-1}{k+s}}
m_T(\nu).
\]
Applying Lemma~\ref{lemma:interval_decomposition} to the tree representation gives
\begin{align*}
\phi_S^{\mathrm{FSII}}(\hat\nu_{\text{tree}})
&=
\sum_{j\in\mathcal L}
c_j
\sum_{T\subseteq L_j}
(-1)^{|T|}
\phi_S^{\mathrm{FSII}}(\1_{[T\cup R_j,N]}).
\end{align*}
Fix a leaf \(j\), and write
\[
r := |R_j|,
\qquad
\ell := |L_j|,
\qquad
u := |S\cap L_j|,
\qquad
s := |S|.
\]

The first term is identical to the corresponding term in the proof of Proposition~\ref{prob:fbii_closed_form}, since it only depends on the Möbius term \(m_S\). Hence it contributes
\[
(-1)^u\1[R_j \subseteq S].
\]

For the second term, the same Möbius support argument as in the FBII proof shows that only subsets \(T\subseteq L_j\) with \(|T|>k-r\) contribute, and the outer sum reduces to the case \(V=R_j\cup T\).
Again, we must have \(S\subseteq R_j\cup T\), which is equivalent to \(S\cap L_j \subseteq T\).
Therefore, the faithful tail term equals
\[
(-1)^{k-s}\frac{s}{k+s}\binom{k}{s}
\sum_{\substack{S \cap L_j \subseteq  T\subseteq L_j \\ |T|>k-r}}
(-1)^{|T|}
\frac{\binom{r+|T|-1}{k}}
{\binom{r+|T|+k-1}{k+s}}.
\]

Now write
\[
S_L := L_j\cap S,
\qquad
S_R \subseteq L_j\setminus S_L,
\qquad
T=S_L\cup S_R.
\]
Then \(|T|=|S_L|+|S_R|\), and grouping subsets by \(i:=|S_R|\) yields
\[
(-1)^{k-s}\frac{s}{k+s}\binom{k}{s}
\sum_{i=k-r-u+1}^{\ell-u}
(-1)^{u+i}
\binom{\ell-u}{i}
\frac{\binom{r+i+u-1}{k}}
{\binom{r+u+i+k-1}{k+s}}.
\]

When $k - r - u + 1 \le 0$, the constraint $|T| > k - r$ is satisfied by every $S \cap L_j \subseteq T \subseteq L_j$, so the lower summation index collapses to $0$; as such, the lower index can be written as $\max(0,\,k-r-u+1)$ as in the proposition statement.
Combining this faithful tail term with the Möbius term and summing over all leaves gives
\[
\phi_S^{\mathrm{FSII}}(\hat{\nu}_{\textnormal{tree}})
=
\sum_{\substack{j \in \mathcal L \\ S \subseteq L_j \cup R_j}}
c_j\cdot\lambda\bigl(|L_j|,|R_j|,|S\cap L_j|,|S|\bigr)
\]
where
\begin{align*}
\lambda(|L_j|,|R_j|,|S\cap L_j|,|S|)
= &(-1)^{|S\cap L_j|}\1[R_j \subseteq S] \\ + &\sum_{i=\max(0, k-r-u+1)}^{\ell-u}
(-1)^{u+i+k-s}
\frac{s}{k+s}\binom{k}{s}
\binom{\ell-u}{i}
\frac{\binom{r+i+u-1}{k}}
{\binom{r+u+i+k-1}{k+s}}
\end{align*}
which proves the claim.
\end{proof}

\begin{proposition}\label{prop:stii_closed_form}
For the Shapley Taylor Interaction Index (STII), Proposition~\ref{prop_tree_capi} yields the closed form 
\[
\phi_S^{\mathrm{STII}}(\hat{\nu}_{\textnormal{tree}})
=
\sum_{\substack{j \in \mathcal L \\ S \subseteq L_j \cup R_j}}
c_j 
\lambda\bigl(|L_j|,|R_j|,|S\cap L_j|,|S|\bigr),
\]
where $u:= |S \cap L_j|$ in 
\begin{align*}
    \lambda(|L_j|,|R_j|,|S\cap L_j|,|S|)
=
& \1_{s < k}\1_{R_j \subseteq S}(-1)^{u} \\ + &\1_{s=k}(-1)^{u}\, s\, B(u+r_j - s + 1, \ell_j-u+s)
\end{align*}
and $B(a,b)$ is the Beta function.
\end{proposition}
\begin{proof}
    Firstly, we observe that the STII is defined to match the Möbius value for $s < k$, as such, Lemma~\ref{lemma:moebius_intervalgame} applies for that case, yielding $m_S(\1_{[R_j, N\setminus L_j]}) = (-1)^{u}\1_{R_j \subseteq S}$ for $S \subseteq L_j \cup R_j$.
    We next show that for $s =k$, a closed form is obtainable. We begin with using Proposition~\ref{prop_tree_capi} and obtain
    \begin{align*}
        \lambda(l,r,u,s) = \sum_{i=0}^{l-u} (-1)^{i+u} \binom{l-u}{i} \frac{1}{\binom{i+u+r}{s}}
    \end{align*}
    by directly using the STII Möbius weights $q_t^s(n)=\frac{1}{\binom{t}{s}}$ \citep{Sundararajan.2020}, where $i+u+r \geq s$ since $S \subseteq L_j \cup R_j$.
    Next, we begin writing out the latter fraction and obtain
    \begin{align*}
        \lambda(l,r,u,s) &= \sum_{i=0}^{l-u} (-1)^{i+u} \binom{l-u}{i} \frac{s!(i+u+r-s)!}{(i+u+r)!} \\
        &= \sum_{i=0}^{n} (-1)^{i+u} \binom{n}{i} \frac{s!(i+m-s)!}{(i+m)!}
    \end{align*}
    where we rewrite $n:= l-u$ and $m:= u+r$.
    By using the formulation of the Beta function $B(a,b)=\frac{(a-1)!(b-1)!}{(a+b-1)!}$, we obtain
    \begin{align*}
        \lambda(l,r,u,s) &= \sum_{i=0}^{n} (-1)^{i+u} \binom{n}{i} s B(i+m-s+1, s) \\
        &=(-1)^u s\sum_{i=0}^n(-1)^i\binom{n}{i}B(i+m-s+1,s)
    \end{align*}
    Now we use the different representation $B(a,b)= \int_0^1t^{a-1}(1-t)^{b-1}dt$ to remove the sum by
    \begin{align*}
        \lambda(l,r,u,s) &= (-1)^u s\sum_{i=0}^n(-1)^i\binom{n}{i}B(i+m-s+1,s) \\
                        &=(-1)^u s\sum_{i=0}^n(-1)^i\binom{n}{i} \int_0^1(1-t)^{s-1}t^{i+m-s}dt \\
                        &= (-1)^u s \int_0^1 (1-t)^{s-1} \bigl[\sum_{i=0}^n(-1)^i\binom{n}{i} t^i\bigr]t^{m-s}dt \\
                        &= (-1)^u s \int_0^1 (1-t)^{s-1} (1-t)^nt^{m-s}dt \\
                        &= (-1)^u s \int_0^1 (1-t)^{n+s-1} t^{m-s}dt \\
                        &= (-1)^u s B(m-s+1,n+s)
    \end{align*}
    Then the final formula is obtained by substituting $m=u+r$ and $n=l-u$, which concludes the proof.
\end{proof}
\clearpage

\section{Generalization of Interventional TreeSHAP}
\label{appendix:generalization_tree_shap}

\begin{wrapfigure}[13]{r}{0.55\textwidth}
\vspace{-2.2em}
\begin{minipage}{\linewidth}
\begin{algorithm}[H]
\caption{\texttt{Tree Interaction Extraction}}
\label{alg:tree_capi}
\begin{algorithmic}[1]
\REQUIRE Tree-based model $\hat{\nu}_{\text{tree}}$, interaction $S \subseteq N$
\ENSURE Extracted interaction $\phi^p_S$
\STATE $\phi^p_S \gets 0$ \hfill {\color{gray!90} $\triangleright$ initialize}
\FOR{each leaf $j \in \mathcal L$}
    \STATE $(L_j, R_j) \gets \textsc{Path}(j)$
    \STATE $c_j \gets \textsc{LeafValue}(j)$
    \IF{$S \subseteq L_j \cup R_j$}
        \STATE $\phi^p_S \gets \phi^p_S
        + \lambda\!\left(|L_j|, |R_j|, |S \cap L_j|, |S|\right)c_j$
    \ENDIF
\ENDFOR
\STATE \textbf{return} $\phi^p_S$
\end{algorithmic}
\end{algorithm}
\end{minipage}
\vspace{-1.2em}
\end{wrapfigure}

\begin{table}[t]
\centering
\caption{Interventional Tree-extraction algorithms. Parenthesized check marks indicate pairwise-only.}
\label{tab:extraction_algorithms}
\vspace{0.2em}
\resizebox{0.9\textwidth}{!}{%
\begin{tabular}{@{}lcccccccccc@{}}
\toprule
\textbf{Algorithm}
& \multicolumn{3}{c}{\textbf{Values}}
& \multicolumn{5}{c}{\textbf{Interactions}}
& \multicolumn{2}{c@{}}{\textbf{Faithful}} \\
\cmidrule(lr){2-4}\cmidrule(lr){5-7}\cmidrule(l){8-9}
& \textbf{SV} & \textbf{BV} & \textbf{CV}
& \textbf{SII} & \textbf{BII} & \textbf{CII} & \textbf{STII} & \textbf{k-SII}
& \textbf{FSII} & \textbf{FBII} \\
\midrule
\rowcolor[gray]{0.94}
\multicolumn{11}{@{}l}{\emph{Interventional extraction}} \\
TreeSHAP (\citet{Lundberg.2020})
    & \greencheck & \redcross & \redcross
    & \redcross & \redcross & \redcross
    & \redcross & \redcross  & \redcross & \redcross \\
TreeSHAP (\citet{Zern.2023})
    & \greencheck & \redcross & \redcross
    & \greencheck & \redcross & \redcross
    & \redcross & \redcross & \redcross & \redcross \\
Woodelf~\citep{Nadel.2026}
    & \greencheck & \greencheck & \redcross
    & $(\greencheck)$ & $(\greencheck)$ & \redcross
    & \redcross & \redcross & \redcross & \redcross  \\
\rowcolor[gray]{0.97}
\textbf{Ours (\cref{alg:tree_capi})}
    & \greencheck & \greencheck & \greencheck
    & \greencheck & \greencheck & \greencheck
    & \greencheck & \greencheck & \greencheck & \greencheck \\
\bottomrule
\end{tabular}%
}
\end{table}
Central to ProxySHAP is the ability to efficiently extract exact cardinal-probabilistic interaction indices from tree-based proxies.
Building on interventional TreeSHAP~\citep{Zern.2023}, we extract the indices of $\hat\nu$ for a target set $\mathcal S \subseteq 2^N$ in $\mathcal O(n_{\text{nodes}}\vert\mathcal S\vert)$ time.
\cref{tab:extraction_algorithms} summarizes the supported indices and compares our generalized interventional TreeSHAP with existing tree-based extraction algorithms. Note that $k-$SII can be computed by first extracting SII and then using the aggregation rule of \citep{Bord.2023}.

First, consider a single decision tree with inputs as binary-encoded subsets $T \subseteq N$.
Each leaf $j \in \mathcal L$ corresponds to a path of node splits, where each split on feature $i \in N$ determines whether $i \in T$ (right branch) or $i \notin T$ (left branch).
An input $T$ reaches leaf $j$ if and only if along the path to $j$: (i) $T \supseteq R_j$, meaning it contains all features that split to the right ($R_j$), and (ii) $T \subseteq N \setminus L_j$, meaning it contains none of the features that split to the left ($L_j$).
The tree-based proxy is then given by the piecewise constant leaf predictions $c_j \in \mathbb{R}$ as
\begin{align}
    \hat\nu_{\text{tree}}(T) := \sum_{j \in \mathcal L} c_j \cdot \1[R_j \subseteq T \subseteq N \setminus L_j].
\end{align}

\citet{Zern.2023} exploited linearity in \cref{eq_tree_proxy} and computed the Shapley interactions for each $\1[R_j \subseteq T \subseteq N \setminus L_j]$, based on the Möbius representation of $\1[R_j \subseteq T \subseteq N \setminus L_j]$.
In a similar spirit, we can compute the general cardinal-probabilistic interaction indices for $\1[R_j \subseteq T \subseteq N \setminus L_j]$ via the Möbius representation, and then use linearity to obtain the interaction indices for $\hat\nu_{\text{tree}}$.
This yields Proposition~\ref{prop_tree_capi}, which gives a closed-form solution for any cardinal-probabilistic interaction index of $\1[R_j \subseteq T \subseteq N \setminus L_j]$.
Then, by using the closed forms of the Möbius weights $q_t^s$ for specific cardinal-probabilistic interaction indices, we can obtain closed-form solutions for the corresponding interventional TreeSHAP interactions, as shown in Propositions~ \ref{prop:banzhaf}, \ref{prop:chaining}, \ref{prob:fbii_closed_form}, \ref{prob:fsii_closed_form} and \ref{prop:stii_closed_form}.
The overall procedure is summarized in \cref{alg:tree_capi}.
Note that, due to Corollary~1 in~\citet{Zern.2023}, Proposition~\ref{prop_tree_capi} also enables exact computation of interactions for piecewise linear regression trees.

Given an interaction $S$, the algorithm iterates over all leaves $j$ and extracts the corresponding path information $(L_j, R_j)$ and leaf value $c_j$.
Then, it updates the interaction value $\phi^p_S$ by adding the contribution from leaf $j$, which is computed using the closed-form solution $\lambda(\vert  L_j \vert, \vert R_j \vert, \vert S \cap L_j\vert, |S|)$ from Proposition~\ref{prop_tree_capi} multiplied by the leaf value $c_j$.
Finally, the algorithm returns the extracted interaction $\phi^p_S$.
The lookups \textsc{Path} and \textsc{LeafValue} can be pre-computed for all leaves and stored in a hash map for efficient access, resulting in an overall time complexity of $\mathcal O(n_{\text{nodes}})$ for extracting a single interaction $S$.

 \clearpage

\section{Experimental Details}
\label{appendix:experiment_details}

\subsection{Datasets}
\begin{table}[ht]
\centering
\small
\renewcommand{\arraystretch}{1.12}
\setlength{\tabcolsep}{4pt}
\caption{Summary of datasets used in the experiments. ``Players'' denotes the number of features or components considered in the corresponding game.}
\label{tab:dataset_summary}
\vspace{0.25em}
\resizebox{\textwidth}{!}{%
\begin{tabular}{@{}l l l c l l@{}}
\toprule
\textbf{Dataset} & \textbf{Task} & \textbf{Model(s)} & \textbf{Players} & \textbf{Source} & \textbf{License} \\
\hdashline\noalign{\vskip 0.25em}

\rowcolor[gray]{0.92}[0pt][0pt]
\multicolumn{6}{@{}l@{}}{\textbf{Vision and language datasets} \citep{Muschalik.2024a}} \\
\addlinespace[0.15em]
Language (IMDB) & Classification & DistilBERT & 14 & shapiq / HuggingFace & Public Domain \\
ResNet18-SP & Image classification & ResNet18 & 14 & shapiq &  ImageNet Terms of Access \\
ViT4by4 & Image classification & ViT-B & 16 & shapiq &  ImageNet Terms of Access \\
DINOv2-16 & Image classification & DINOv2-B + linear head & 16 & ImageNet-1k / HuggingFace & ImageNet Terms of Access \\
DINOv2-20 & Image classification & DINOv2-B + linear head & 20 & ImageNet-1k / HuggingFace & ImageNet Terms of Access \\
DINOv2-25 & Image classification & DINOv2-B + linear head & 25 & ImageNet-1k / HuggingFace & ImageNet Terms of Access \\

\addlinespace[0.4em]
\rowcolor[gray]{0.92}[0pt][0pt]
\multicolumn{6}{@{}l@{}}{\textbf{Tabular datasets} \citep{Strumbelj.2010,Muschalik.2024a}} \\
\addlinespace[0.15em]
Bike & Regression & TabPFN & 12 & OpenML (42712) & Public Domain \\
Forest Fires & Regression & TabPFN & 13 & UCI & CC-BY 4.0 \\
AdultCensus & Classification & TabPFN & 14 & OpenML (1590) & CC-BY 4.0 \\
Estate & Regression & TabPFN & 15 & UCI & CC-BY 4.0 \\
Hepatitis & Classification & TabPFN & 19 & UCI (46) & CC-BY 4.0 \\
Thyroid & Classification & TabPFN & 21 & UCI (102) & CC-BY 4.0 \\
Mushroom & Classification & TabPFN & 22 & UCI (73) & CC-BY 4.0 \\
Cancer & Classification & XGBoost & 30 & scikit-learn & CC-BY 4.0 \\
Ionosphere & Classification & LightGBM & 33 & UCI (52) & CC-BY 4.0 \\
Soybean & Classification & LightGBM & 35 & UCI (90) & CC-BY 4.0 \\
CG60 & Regression (synthetic) & XGBoost & 60 & SHAP & MIT \\
IL60 & Regression (synthetic) & XGBoost & 60 & SHAP & MIT \\
NHANES & Survival & XGBoost & 79 & SHAP & Public Domain \\
Crime & Regression & XGBoost & 101 & SHAP & CC-BY 4.0 \\

\addlinespace[0.4em]
\rowcolor[gray]{0.92}[0pt][0pt]
\multicolumn{6}{@{}l@{}}{\textbf{TabArena datasets}~\citep{Erickson.2025}} \\
\addlinespace[0.15em]
Wine Quality & Regression & TabPFN & 12 & OpenML (46964) & CC-BY 4.0 \\
HR Analytics & Classification & TabPFN & 12 & OpenML (46935) & CC0 \\
Bank Marketing & Classification & TabPFN & 13 & OpenML (46910) & CC-BY 4.0 \\
Good Customer & Classification & TabPFN & 13 & OpenML (46938) & CC0 \\
Miami Housing & Regression & TabPFN & 15 & OpenML (46942) & CC-BY-NC-SA \\
Seismic Bumps & Classification & TabPFN & 15 & OpenML (46956) & CC-BY 4.0 \\
Online Shoppers & Classification & TabPFN & 17 & OpenML (46947) & CC-BY 4.0 \\
Churn & Classification & TabPFN & 19 & OpenML (46915) & MIT \\
Credit-G & Classification & TabPFN & 20 & OpenML (46918) & CC-BY 4.0 \\
Airline Satisfaction & Classification & TabPFN & 21 & OpenML (46920) & CC0 \\
JM1 & Classification & TabPFN & 21 & OpenML (46979) & Public Domain \\
Credit Card Default & Classification & TabPFN & 23 & OpenML (46919) & CC-BY 4.0 \\
HELOC & Classification & TabPFN & 23 & OpenML (46932) & Public Domain \\
Coupon Recommendation & Classification & TabPFN & 24 & OpenML (46937) & CC-BY 4.0 \\
Marketing Campaign & Classification & TabPFN & 25 & OpenML (46940) & Public Domain \\
Hazelnut & Classification & LightGBM & 30 & OpenML (46930) & CC-BY-SA \\
Students Dropout & Classification & LightGBM & 36 & OpenML (46960) & CC-BY 4.0 \\
Anneal & Classification & LightGBM & 38 & OpenML (46906) & CC-BY 4.0 \\
QSAR Biodeg & Classification & LightGBM & 41 & OpenML (46952) & CC-BY 4.0 \\
Diabetes 130-US & Classification & LightGBM & 47 & OpenML (46922) & CC-BY 4.0 \\
Splice & Classification & LightGBM & 60 & OpenML (46958) & CC-BY 4.0 \\
Bankruptcy & Classification & LightGBM & 64 & OpenML (46950) & CC-BY 4.0 \\
Superconductivity & Regression & LightGBM & 81 & OpenML (46961) & CC-BY 4.0 \\
COIL 2000 & Classification & LightGBM & 85 & OpenML (46916) & CC-BY 4.0 \\
NATICUSdroid & Classification & LightGBM & 86 & OpenML (46969) & CC-BY 4.0 \\
Taiwanese Bankruptcy & Classification & LightGBM & 94 & OpenML (46962) & CC-BY 4.0 \\
MIC & Classification & LightGBM & 111 & OpenML (46980) & CC-BY 4.0 \\
APS Failure & Classification & LightGBM & 170 & OpenML (46908) & CC-BY 4.0 \\
KDD Cup 09 & Classification & LightGBM & 212 & OpenML (46939) & Public Domain \\
QSAR TID11 & Classification & LightGBM & 1024 & OpenML (46953) & Public Domain \\
HivaAgnostic & Classification & LightGBM & 1617 & OpenML (46933) & Public Domain \\
Bioresponse & Classification & LightGBM & 1776 & OpenML (46912) & Public Domain \\

\addlinespace[0.4em]
\rowcolor[gray]{0.92}[0pt][0pt]
\multicolumn{6}{@{}l@{}}{\textbf{Graph datasets}~\citep{Muschalik.2025}} \\
\addlinespace[0.15em]
Benzene & Classification & Graph Neural Network & 25 & \citet{Sanchez.2020} & CC0 1.0 Universal \\
Mutagenicity & Classification & Graph Neural Network & 35 & \citet{Kazius.2005} & CC0 1.0 Universal \\

\bottomrule
\end{tabular}%
}
\end{table}

\label{appendix:datasets}
We build an interventional game for each tabular dataset with more than $25$ features, based on the underlying trained XGBoost \citep{Chen.2016} or LightGBM \citep{Ke.2017}. 
The interventional game is based on a point to explain $x^e$ and background samples $\mathcal{B} = \{b_1, \dots, b_{n_{bg}}\}$. Given the coalition $S$, we construct new points 
$$
    z_i[j] = \begin{cases}
        x^e[j] \; \text{if} \;j \in S \\
        b_i[j] \; \text{otherwise}
    \end{cases}
$$
where we access the feature $j$ via $x[j]$.
The game value of $T$ is then obtained by averaging each value from $z_1, \dots, z_{n_{bg}}$, i.e. 
$$
\nu(T) = \frac{1}{n_{bg}}\sum_{i=1}^{n_{bg}} f(z_i)
$$
where we denote the underlying tree-based models as $f$. 
In total, we use $n_{bg}=50$ background samples drawn from the training dataset.
The exact values are calculated using \cref{alg:tree_capi}.
In the following, we provide concise descriptions of each domain; an overview of all datasets is presented in \cref{tab:dataset_summary}.

\paragraph{Vision and Language}
For Vision and Language tasks, we use the precomputed games provided by \texttt{shapiq} \citep{Muschalik.2024a}, which are based on pretrained models and standard datasets for their respective domains.
The language games are based on DistilBERT \citep{Sanh.2019} fine-tuned on the IMDB movie reviews dataset \citep{Maas.2011,Lhoest.2021}, while the image classification games use Vision Transformer and ResNet18 pretrained on ImageNet.
The language games aim to predict the sentiment of movie reviews, while the image classification games aim to predict the class of input images.

\paragraph{Tabular Datasets}
We consider a variety of tabular datasets \citep{Strumbelj.2010, Witter.2025, Muschalik.2024a, Erickson.2025}.
For all tabular datasets with up to 25 features, we train TabPFN \citep{Grinsztajn.2025,Hollmann.2025}, since this feature count still allows exact computation of the probabilistic indices.
For all other tabular datasets, we either train an XGBoost model \citep{Chen.2016} or a LightGBM model \citep{Ke.2017} using default hyperparameters to obtain the interventional game. 
This makes our evaluation more comprehensive and enables us to investigate how the underlying model influences the interactions we obtain.

\paragraph{Graph Datasets}
For graph datasets, we follow the setup of \citet{Muschalik.2025} and use their proposed \texttt{GraphSHAP-IQ} method to compute exact Shapley interactions for graph neural networks. The \textsc{Mutagenicity} dataset \citep{Kazius.2005} consists of $1{,}768$ molecular graphs, categorized into two classes based on their mutagenic properties, specifically their effects on the Gram-negative bacterium \emph{S. typhimurium} \citep{Muschalik.2025}. The \textsc{Benzene} dataset \citep{Sanchez.2020} consists of $12{,}000$ molecular graphs, each labeled according to whether it contains a benzene ring. We use the pretrained models from \citet{Muschalik.2025}, which have $2$ GNN layers and a hidden dimension of $32$.

\paragraph{Sum Of Unanimity (SOUM)}
For the synthetic games, we follow the procedure described in \citep{Tsai.2022}.
A SOUM game $\nu$ is expressed as:
$$
\nu(T) = \sum_{i=1}^l \alpha_i \mathbf{1}_{S_i \subseteq T}
$$
where we denote $l$ as the number of unanimity basis games and $S_i$ as uniformly sampled subsets with a maximal size of $k$. Based on the experimental setup in [1], we build two games: $70$ players with $l=30$ and $k=5$ and $90$ players with $l=10$ and $k=10$. Each $\alpha_i$ is sampled uniformly over $[-1,+1]$. The exact Shapley interactions can be obtained via the Möbius transform, as $\nu$ is in its Möbius form.

\subsection{Computational Resources}
\label{appendix:computation_resource}
Approximation-quality experiments and hyperparameter optimization were conducted on a compute cluster with 96 Intel Xeon Platinum 8480+ (Sapphire Rapids) CPUs and 512\, GB of RAM, totaling 3 weeks of runtime.
TabPFN-based games were precomputed using NVIDIA V100 GPUs in parallel, enabling exhaustive evaluation of the value function for datasets with $n \leq 25$ features. 
CLIP experiments were run on a cluster consisting of NVIDIA A100 GPUs.

\subsection{Hyperparameter Optimization}
\label{appendix:hpo_method}
To obtain high-quality proxy models, we tune the XGBoost hyperparameters using Bayesian optimization (BO), since the approximation quality of ProxySHAP depends directly on the fitted proxy.
Compared with grid or random search, BO explores the predefined search space more adaptively and can therefore identify strong configurations with fewer evaluations.
We use the BO implementation of \citet{Lindauer.2022}.

The integer-valued hyperparameters are the number of estimators ($[500, 2000]$), maximum tree depth ($[2, 8]$), and minimum child weight ($[1, 20]$).
The continuous hyperparameters are the subsampling rate ($[0.6, 1.0]$), column subsampling rate ($[0.6, 1.0]$), learning rate ($[0.001, 0.3]$), and $\ell_2$ regularization strength ($[0.001, 50]$).
We run BO for 200 iterations and evaluate each configuration using 5-fold cross-validation.
After optimization, the final proxy model is retrained on the full set of sampled coalitions.

HPO runtime varies substantially across datasets, depending on the number of sampled coalitions and the number of features.
Expanding the search space may further improve approximation quality, but would also increase the computational cost.

\subsection{Baselines}
\label{appendix:baselines}
We briefly describe the baseline methods used for comparison in our experiments. 
All baselines are based on the implementation of the \texttt{shapiq} library \citep{Muschalik.2024a}.

\paragraph{PermutationSamplingSII.}
PermutationSamplingSII is a Monte Carlo estimator for the Shapley Interaction Index (SII), extending classical permutation sampling for Shapley values \citep{Castro.2009} to higher-order interactions \citep{Tsai.2022,Fumagalli.2023}. 
The method estimates interactions by averaging marginal contributions over random orderings of the random features. 
It is applicable exclusively to SII.

\paragraph{SVARM-IQ.}
SVARM-IQ (Stratified Variable Approximation for Interaction Quantification) \citep{Kolpaczki.2024b} is an efficient estimator for cardinal probabilistic interaction indices based on stratified sampling of marginal contributions. 
It builds on the MSR framework \citep{Wang.2023} and additionally uses stratified sampling, further reducing the variance of the estimate  \citep{Castro.2017}.

\paragraph{SHAP-IQ.}
SHAP-IQ \citep{Fumagalli.2023} provides a unified framework for approximating Shapley interactions of arbitrary order by extending unbiased KernelSHAP to interaction indices. 
It supports all cardinal probabilistic interaction indices.

\paragraph{KernelSHAP-IQ.}
KernelSHAP-IQ \citep{Fumagalli.2024} generalizes KernelSHAP to higher-order interactions by solving a weighted least-squares optimization problem with interaction-specific kernels. 
The method supports estimating both Shapley and Banzhaf interactions but suffers from scalability limitations due to the combinatorial growth of interaction terms.

\paragraph{ProxySPEX.}
ProxySPEX \citep{Butler.2025} is an inference-efficient method for estimating sparse Shapley and Banzhaf interactions, improving upon SPEX \citep{Kang.2025}. 
By learning a sparse proxy model, ProxySPEX significantly reduces the number of required model evaluations and scales to settings with thousands of features.

\clearpage

\section{Additional Approximation Results}

\subsection{Additional Experiments with FIxLIP}
\label{appendix:clip_experiments}

Following \citet{Baniecki.2025b}, we estimate model explanations with the faithful Banzhaf interaction index (FBII) of second-order and assess approximation quality using the area between the insertion/deletion curves (AID) and $R^2$ metrics.
We rely on the default experimental setup, restricting pairwise interactions to the top 72 clique for ViT-B/16, selected via a greedy increase-in-value strategy.
Given estimated interactions $\phi^{\text{FBII}}$, we sample $m=1000$ coalitions and define $\hat{\nu}(T)=\sum_{S\subseteq T}\phi^{\text{FBII}}_S$. 
The $R^2$ score is computed as
\[
R^2 = 1 - \frac{\|\hat{\nu}-\nu\|^2}{\|\nu\|^2},
\]
where $\nu$ denotes the sampled game values obtained from the model outputs.
AID score measures the average increase/decrease in the model's prediction when sequentially inserting/deleting important features based on their ranking as measured with an explanation~\citep{Hama.2023,Zhao.2024}.

\paragraph{Evaluating faithfulness of the game.}
We measure $R^2$ on explanations for 30 image--text inputs from MS COCO, with budgets ranging from $10^2$ to $10^4$.
\cref{fig:appendix:clip_proxyspex} shows that ProxySHAP improves the original approximation in low-budget regimes, improves the $R^2$ metric more rapidly, and requires substantially fewer model evaluations than ProxySPEX.
It effectively solves the challenge of approximating the $m\ll p$ regression matrix using a linear-based model, as reported in the original work.

\paragraph{Ablation on approximating the cross-modal FIxLIP estimator.}
We run additional ablations with another \emph{cross-modal} estimator proposed in~\citep{Baniecki.2025b}, which is more sample-efficient.
The standard estimator samples image--text coalitions jointly, i.e., each CLIP call yields a single game evaluation.
The cross-modal estimator samples image and text coalitions independently and evaluates the model on all pairwise combinations, yielding more game evaluations for the same CLIP-call budget. 
\cref{tab:appendix:clip_calls_to_budget} summarizes the rough relationship between CLIP calls and the resulting number of game evaluations, which varies with vision model size and input text length. 
\cref{fig:appendix:clip_crossmodal} reports $R^2$ approximation performance for ProxySHAP on the FIxLIP game, which remains useful in the low-budget regimes, especially for the larger ViT-16 game.

\begin{figure}[ht]
    \centering
    \includegraphics[width=0.9\textwidth]{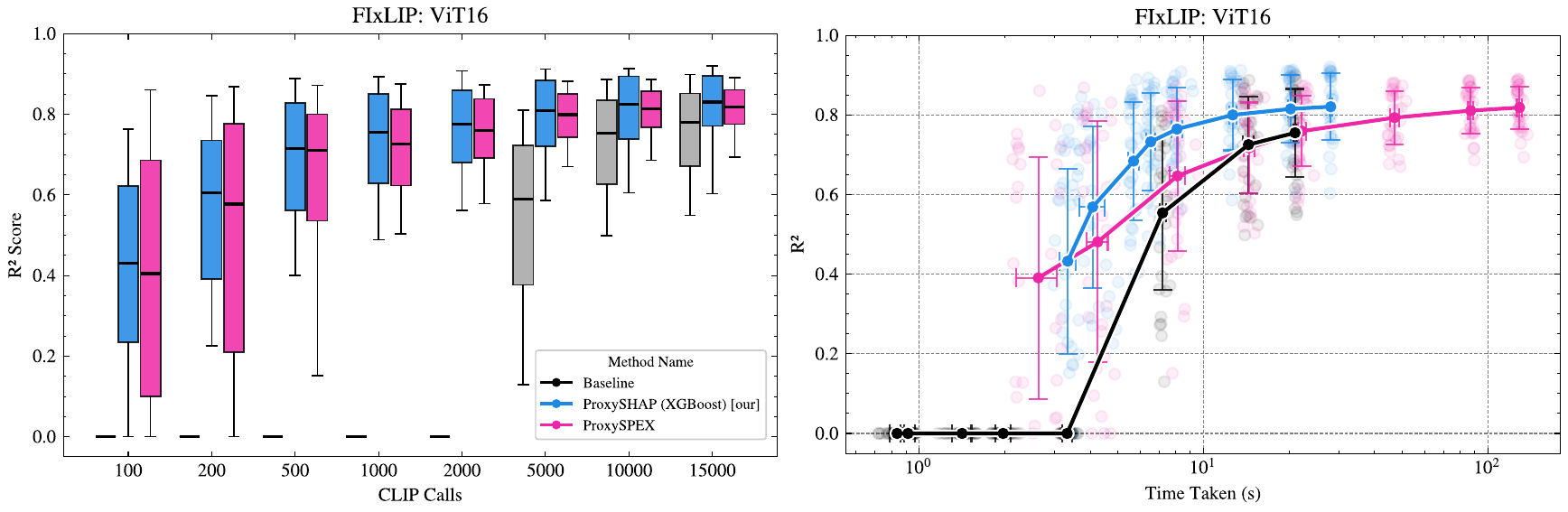}
    \caption{Faithfulness $R^2$ for explaining CLIP (ViT-16) on the MS COCO dataset with ProxySHAP, ProxySPEX, and the FIxLIP baseline.}
    \label{fig:appendix:clip_proxyspex}
\end{figure}

\begin{table}[t]
\centering
    \captionof{table}{Relationship between CLIP calls and the resulting number of game evaluations under the standard and cross-modal estimators, averaged over the explained MS COCO images.}
    \label{tab:appendix:clip_calls_to_budget}
    \resizebox{0.5\textwidth}{!}{
    \begin{tabular}{l r r r}
    \toprule
    \textbf{CLIP model} & \textbf{CLIP calls} & \textbf{Approximate Game Evaluations} & \textbf{Crossmodal Game Evaluations} \\
    \midrule
    \multirow{7}{*}{ViT-B/16}
     & 100    & 100    & 790 \\
     & 200    & 200    & 1\,713 \\
     & 500    & 500    & 6\,834 \\
     & 1\,000 & 1\,000 & 26\,088 \\
     & 2\,000 & 2\,000 & 102\,107 \\
     & 5\,000 & 5\,000 & 630\,766 \\
     & 10\,000 & 10\,000 & 2\,645\,308 \\
    \addlinespace
    \multirow{7}{*}{ViT-B/32}
     & 100    & 100    & 2\,870 \\
     & 200    & 200    & 11\,280 \\
     & 500    & 500    & 69\,654 \\
     & 1\,000 & 1\,000 & 278\,522 \\
     & 2\,000 & 2\,000 & 1\,112\,185 \\
     & 5\,000 & 5\,000 & 6\,945\,324 \\
     & 10\,000 & 10\,000 & 24\,062\,728 \\
    \bottomrule
    \end{tabular}
    }
\end{table}

\begin{figure}[ht]
    \centering
    \includegraphics[width=0.7\textwidth]{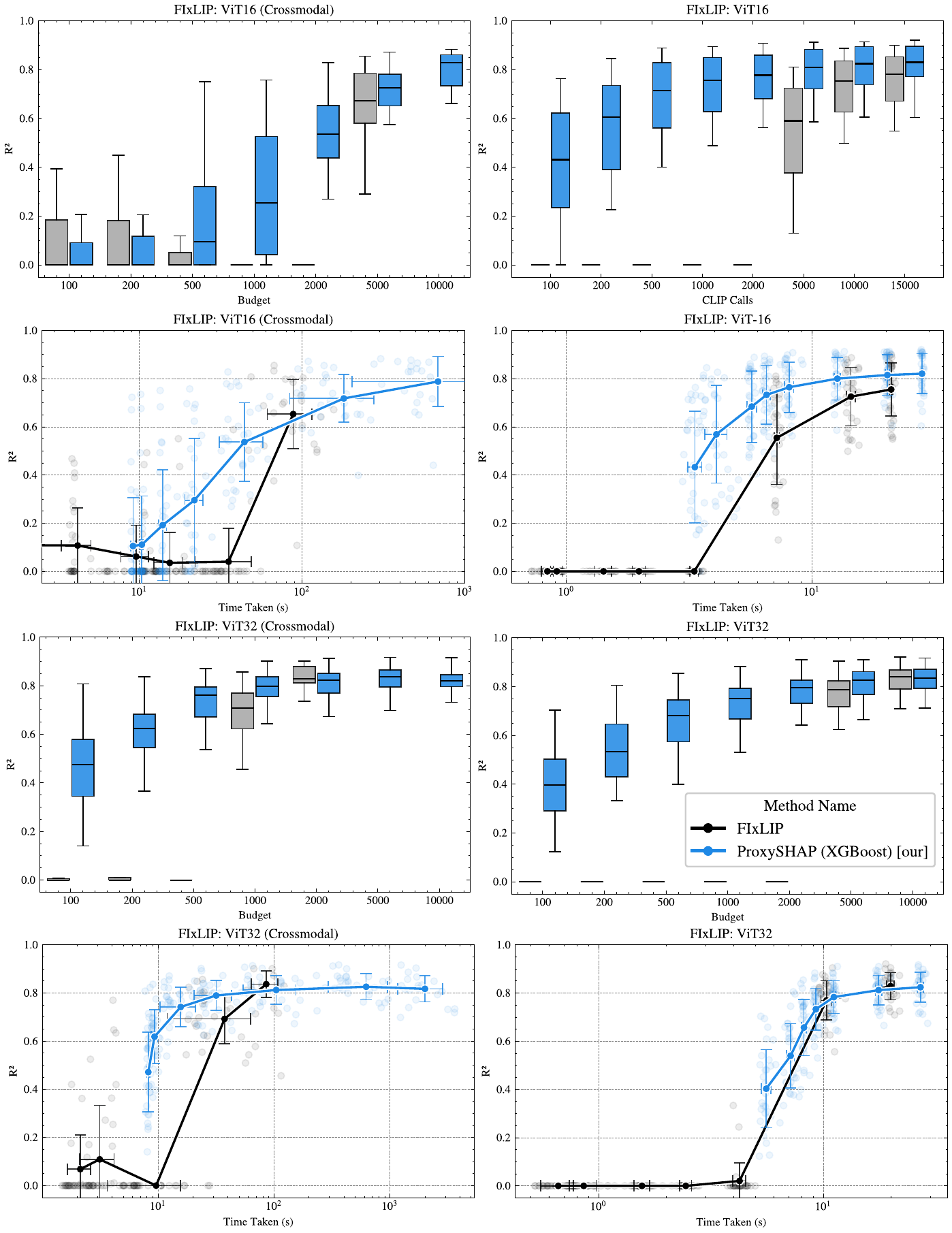}
    \caption{
    Ablation on approximating the cross-modal FIxLIP estimator. 
    Faithfulness $R^2$ for explaining two CLIP variants on MS COCO with ProxySHAP and the FIxLIP baseline.
    }
    \label{fig:appendix:clip_crossmodal}
\end{figure}

\subsection{XGBoost Default for Large Player Counts}
\label{appendix:xgboost_default_effect}

\begin{figure}[ht]
    \centering
    \includegraphics[width=\textwidth]{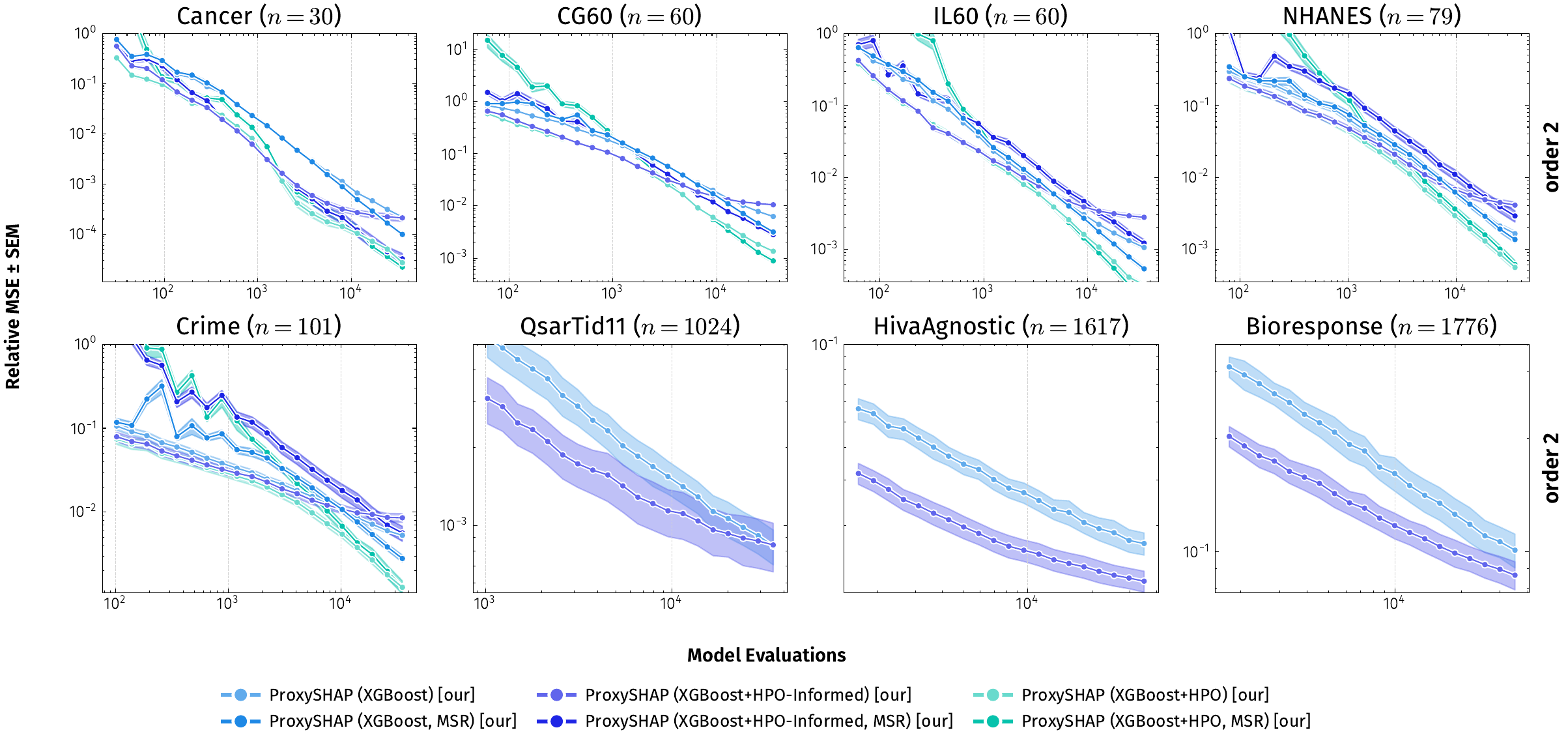}
    \caption{Approximation quality of two different XGBoost defaults. 
    We show that using $2000$ trees with a maximum depth of $3$ improves estimation quality in low- to medium-budget regimes.}
    \label{fig:appendix:xgboost_default_comparison}
\end{figure}

\begin{wraptable}[8]{r}{0.38\textwidth}
\vspace{-0.6cm}
\centering
\caption{XGBoost proxy configurations.}
\label{tab:comparison_configs}
\small
\setlength{\tabcolsep}{4pt}
\renewcommand{\arraystretch}{1.08}
\begin{tabular}{@{}lcc@{}}
\toprule
\textbf{Hyperparameter} 
& \textbf{Default} 
& \textbf{HPO-informed} \\
\midrule
$n_{\text{estimators}}$   & $100$  & $2{,}000$ \\
\texttt{max\_depth}       & $6$    & $3$ \\
\texttt{learning\_rate}   & $0.3$  & $0.05$ \\
\texttt{reg\_lambda}      & $1$    & $5$ \\
\bottomrule
\end{tabular}
\end{wraptable}

When games involve many features, hyperparameter optimization (HPO) often selects XGBoost proxies with many shallow trees. 
In particular, the selected configurations often use approximately 2{,}000 trees together with a small maximum depth. 
We use this observation to define an HPO-informed default configuration, denoted by \emph{ProxySHAP (XGBoost+HPO-Informed)}. 
\cref{tab:comparison_configs} summarizes the differences between the standard XGBoost configuration used in ProxySHAP and the HPO-informed configuration.
Motivated by this observation, we evaluate \emph{ProxySHAP (XGBoost+HPO-Informed)} as an alternative default proxy for large-scale games. 
The results show that this configuration improves approximation quality in low-budget regimes and for games with many players. 
As illustrated in \cref{fig:appendix:xgboost_default_comparison}, the HPO-informed configuration consistently outperforms the standard XGBoost default at low budgets. 
For moderate feature counts, however, its advantage decreases as the budget increases, and the standard configuration eventually becomes competitive or superior.

For datasets with more than 1{,}000 features, \emph{ProxySHAP (XGBoost+HPO-Informed)} significantly outperforms the standard default across all considered budgets. 
This is the case, for example, on \textsc{HivaAgnostic} and \textsc{Bioresponse}, which contain 1{,}617 and 1{,}776 features, respectively. 
These results further highlight the importance of proxy hyperparameters for ProxySHAP, especially in low-budget regimes and for games with many players, where the standard XGBoost default may be insufficient to capture the relevant interaction structure.

\subsection{Runtime}
\label{appendix:runtime}
\begin{figure}[t]
    \centering
    \begin{subfigure}{0.9\textwidth}
        \includegraphics[width=\textwidth]{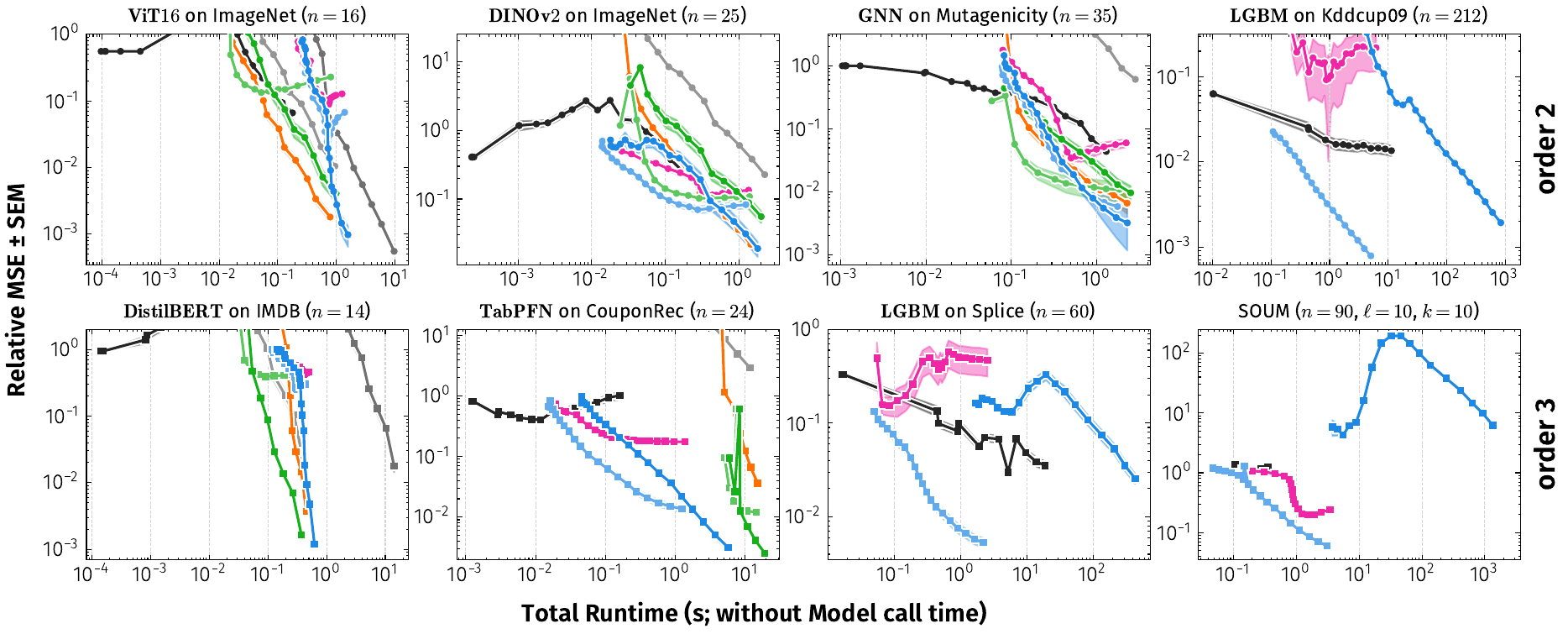}
    \end{subfigure}\\
    \begin{subfigure}{0.9\textwidth}
        \includegraphics[width=\textwidth]{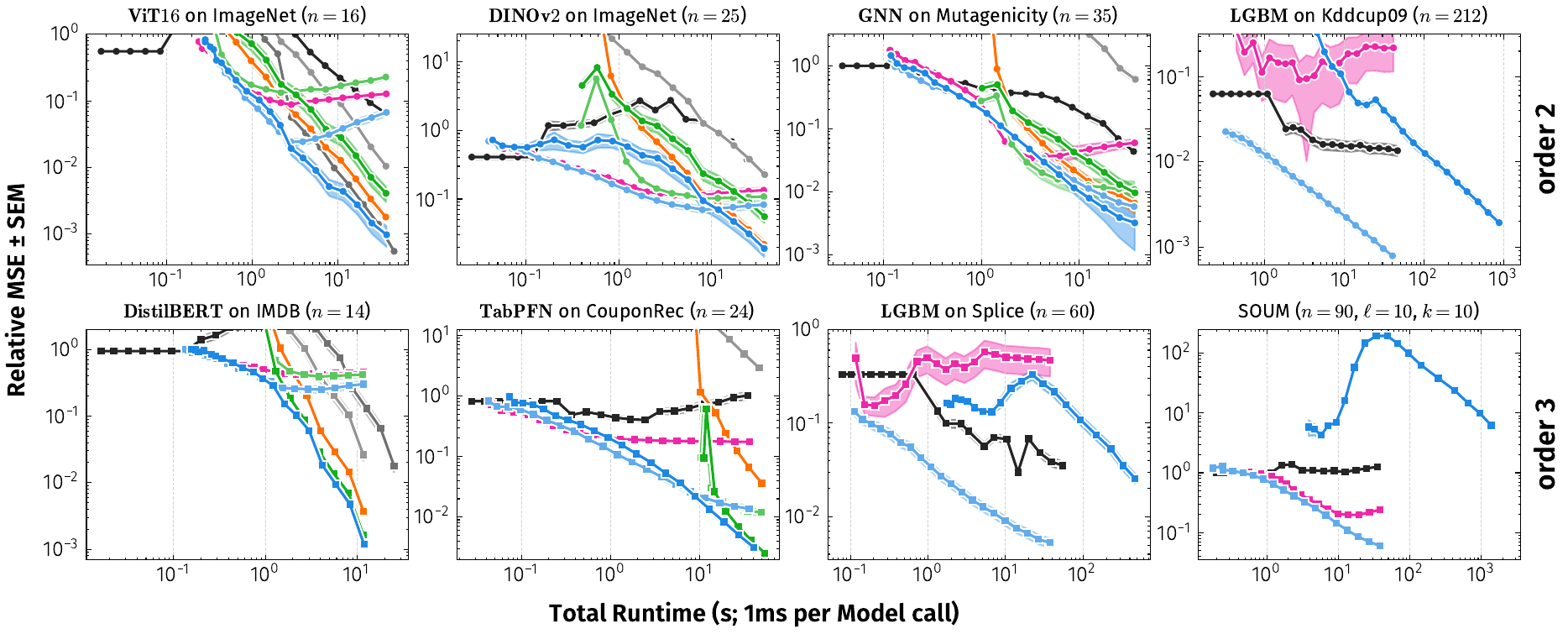}
    \end{subfigure}\\
    \begin{subfigure}{0.9\textwidth}
        \includegraphics[width=\textwidth]{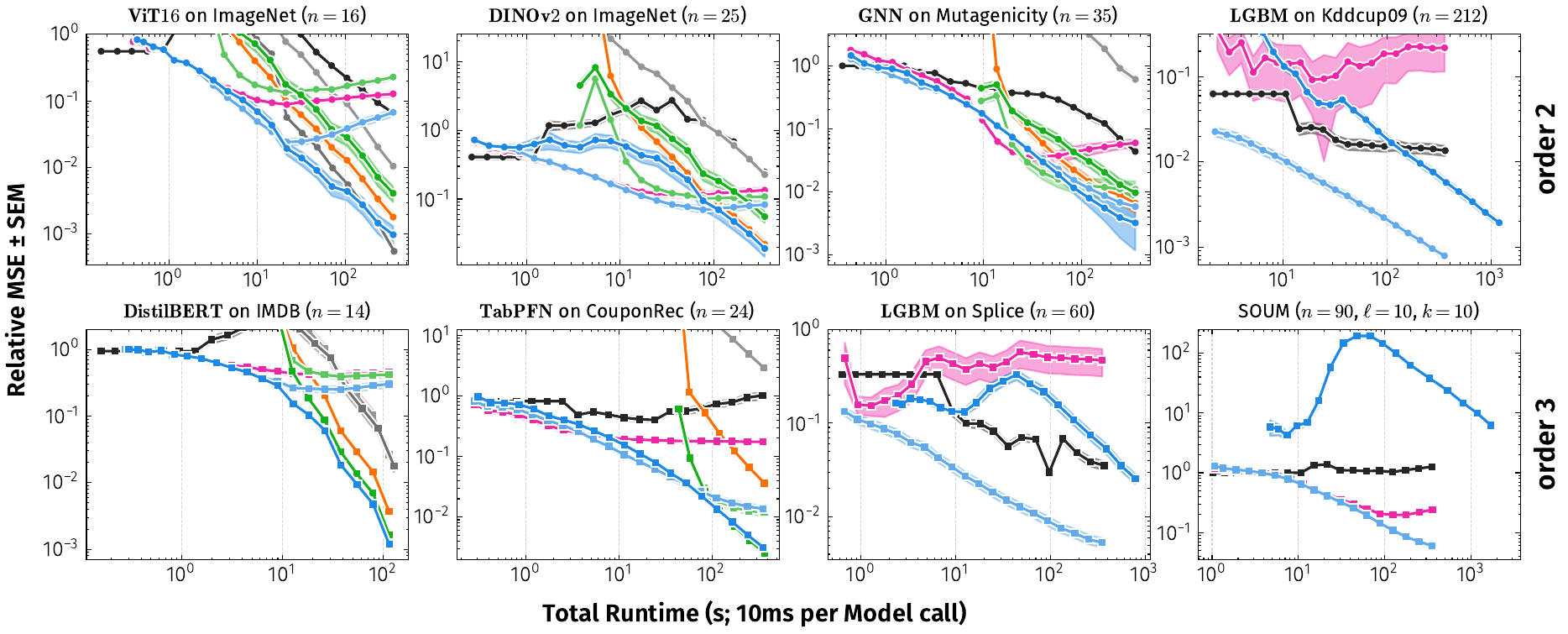}
    \end{subfigure}\\
    \begin{subfigure}{0.9\textwidth}
        \includegraphics[width=\textwidth]{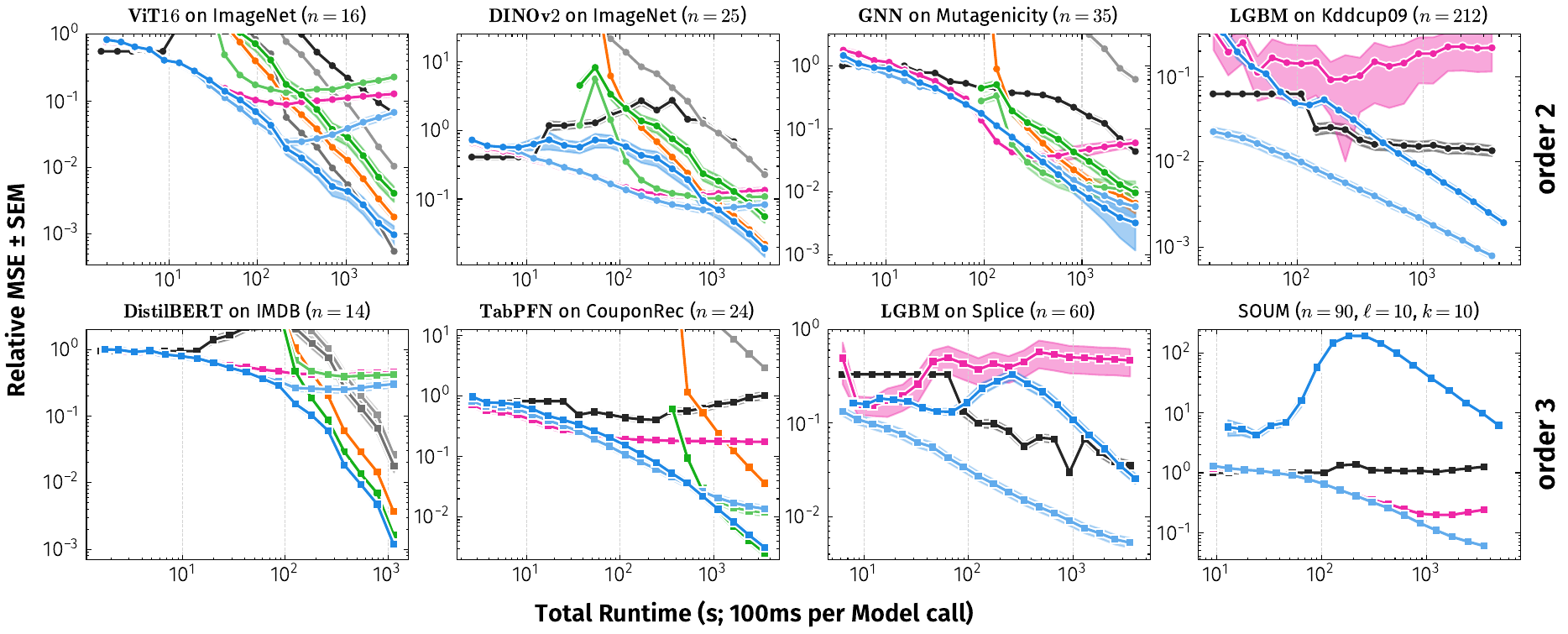}
    \end{subfigure}\\
    \includegraphics[width=\textwidth]{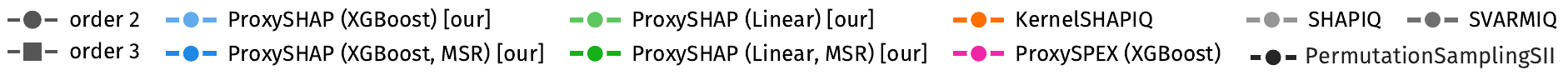}
    \caption{
Approximation quality as a function of runtime for second- and third-order interaction estimation across different per-evaluation cost regimes.}
    \label{fig:appendix:runtime}
\end{figure}
We evaluate runtime by translating model evaluations into wall-clock time using fixed per-evaluation costs of $0\,\mathrm{ms}$, $1\,\mathrm{ms}$, $10\,\mathrm{ms}$, and $100\,\mathrm{ms}$. These regimes capture increasingly expensive inference settings, from highly optimized models to large-scale neural networks. \cref{fig:appendix:runtime} reports the resulting runtimes for second- and third-order interaction estimation.

\cref{fig:appendix:runtime} reports approximation quality together with runtime, ranging from zero model-evaluation cost (top row) to $100\,\mathrm{ms}$ per model call. We find that the linear proxy incurs higher computational overhead than the tree-based proxy, especially for higher-order interactions, due to its polynomial scaling in the number of interaction terms. It only outperforms the tree-based proxy on a few games, such as GNN on Mutagenicity.

ProxySHAP is faster than ProxySPEX when model-inference costs are ignored. As model evaluations increasingly dominate the runtime, this difference becomes less pronounced. We observe a similar pattern for MSR adjustment: its overhead grows with the interaction order and number of players, but becomes less visible as the cost of model calls increases.

Overall, ProxySHAP remains an efficient estimator for Shapley and Banzhaf interactions, requiring less computation than current proxy-based baselines, most notably ProxySPEX.

\clearpage
\subsection{Ablation Details}
\label{appendix:abalation}
\begin{wrapfigure}[12]{r}{0.5\textwidth}
    \centering
    \begin{subfigure}{0.22\textwidth}
        \includegraphics[width=\textwidth]{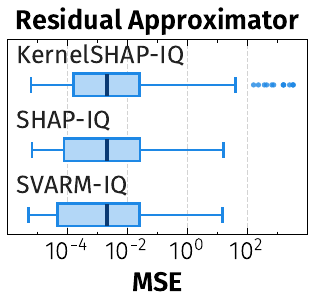}
    \end{subfigure}
    \begin{subfigure}{0.22\textwidth}
        \includegraphics[width=\textwidth]{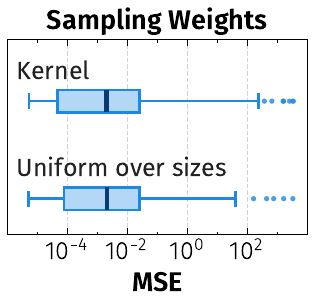}
    \end{subfigure}
    \caption{Ablations of sampling weights and residual approximators for ProxySHAP.}
    \label{fig:abalation_figure}
\end{wrapfigure}

We further investigate the effect of the residual approximator and sampling weights used in the adjustment step.
Specifically, we compare SHAP-IQ \citep{Fumagalli.2023}, KernelSHAP-IQ \citep{Fumagalli.2024} and SVARM-IQ \citep{Kolpaczki.2024b} as model-agnostic residual approximators.
We also compare uniform over sizes, as used in LeverageSHAP \citep{Musco.2025}, with KernelSHAP-IQ weights \citep{Fumagalli.2024}.
As underlying games, we use \textsc{ViT4by4Patches}, \textsc{BikeSharingLocalXAI}, \textsc{CaliforniaHousingLocalXAI}, \textsc{Corrgroups60LocalXAI}, and \textsc{CommunitiesAndCrimeLocalXAI}; details on these datasets are provided in \cref{appendix:datasets}.
For each game, we approximate second-order Shapley interaction indices so that the comparison directly reflects the effect on interaction estimation.
As shown in \cref{fig:abalation_figure}, neither the choice of residual approximator nor the choice of sampling weights has a clear systematic effect on approximation quality for these games.

\subsection{Fourier Extraction vs. Interventional Extraction}
\label{appendix:fourier_vs_interventional}

\begin{figure}[ht]
    \centering
    \begin{subfigure}{\linewidth}
      \includegraphics[width=\linewidth]{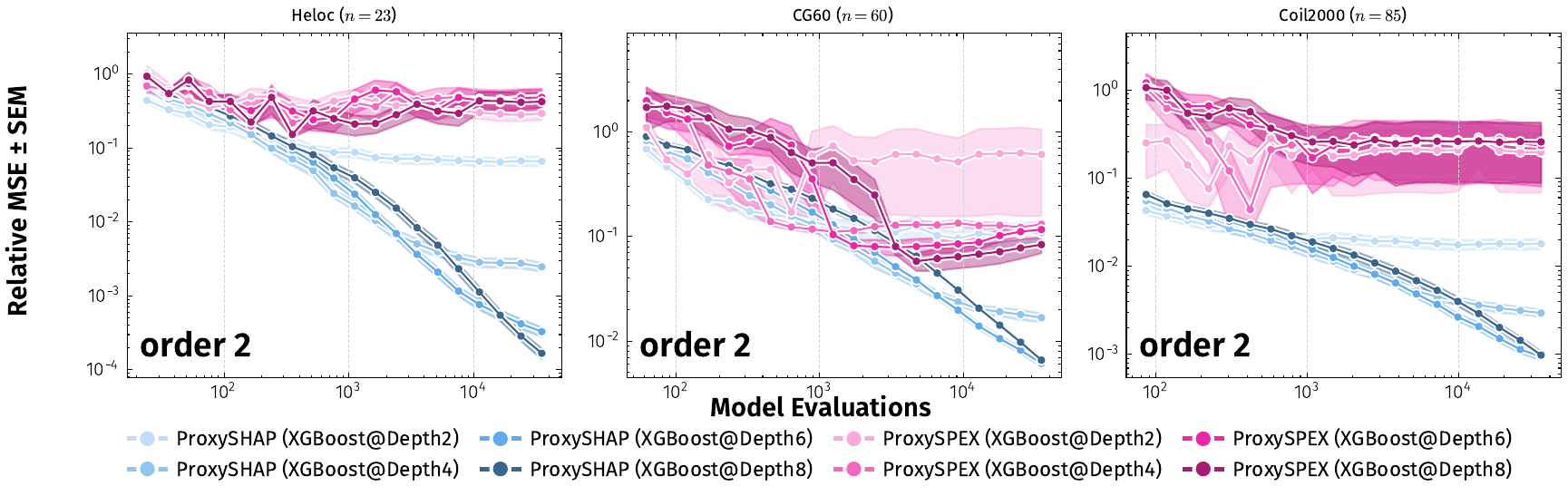}  
    \caption{Approximation quality (Relative MSE) of ProxySHAP and ProxySPEX using different maximum tree depth options across small, medium, and large player domains.}
    \label{fig:tree_depth_on_approximation_quality}
    \end{subfigure}
    \begin{subfigure}{\linewidth}
            \includegraphics[width=\linewidth]{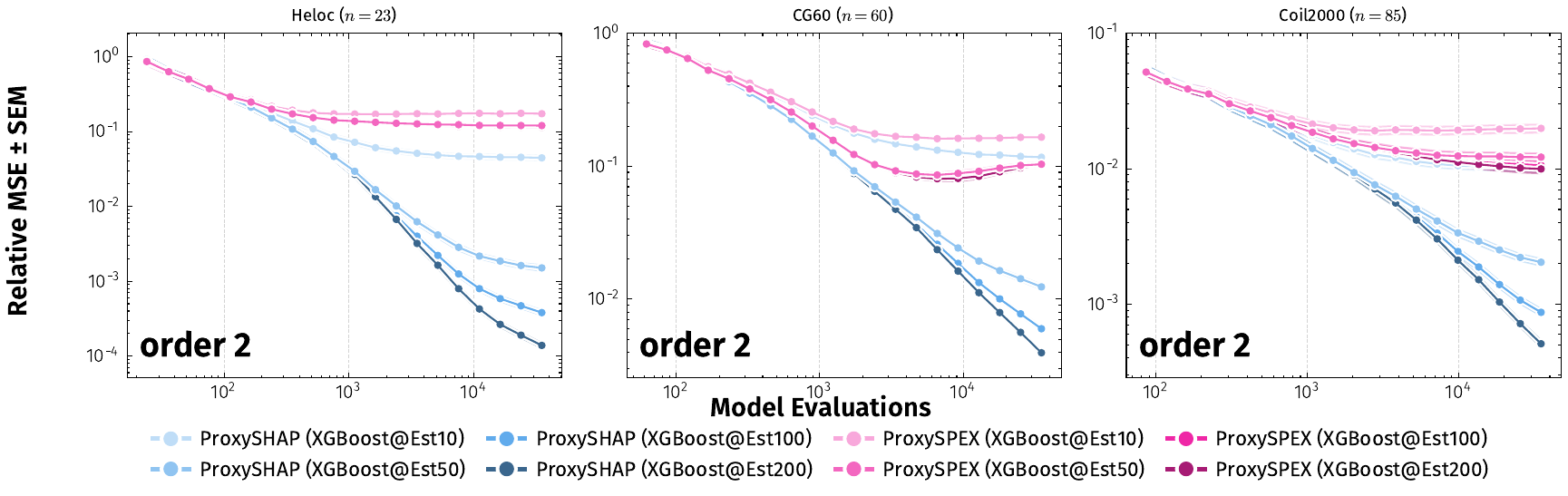}
    \caption{Approximation quality (Relative MSE) of ProxySHAP and ProxySPEX using different numbers of trees across small, medium, and large player domains.}
    \label{fig:tree_capacity_on_approximation_quality}
    \end{subfigure}
    \label{fig:placeholder}
\end{figure}

Our method relies on the ability to efficiently extract exact cardinal-probabilistic interaction indices from the underlying tree-based model.
We extend interventional TreeSHAP by \citet{Zern.2023} to extract the \emph{exact} cardinal-probabilistic interaction indices of $\hat\nu$ for a target set $\mathcal S \subseteq 2^N$ in $\mathcal{O}(n_{\text{nodes}} \cdot |\mathcal S|)$ time, where $n_{\text{nodes}}$ denotes the number of tree nodes.
An alternative is Fourier extraction, as proposed by \citet{Gorji.2025, Butler.2025}, whose cost grows exponentially with the maximum tree depth $d$.
In the worst case, a tree of depth $d$ can induce up to $\mathcal{O}(4^d)$ Fourier coefficients \citep{Gorji.2025}, which must then be extracted and converted to the desired interaction indices.

We compare the two extraction approaches on nine datasets: \textsc{Bike}, \textsc{AdultCensus}, \textsc{Crime}, \textsc{ForestFires}, \textsc{IL60}, \textsc{CG60}, \textsc{NHANES}, \textsc{Cancer}, and \textsc{Estate}.
For each method, we measure the runtime required to extract all target interactions.
We exclude preprocessing time, such as storing the tree structure in a hash map, since this is a one-time cost that can be amortized across multiple extraction runs.
For Fourier extraction, we measure the time required to extract the Fourier coefficients and convert them to the target interaction indices.
For interventional extraction, we measure the runtime of our extraction algorithm directly.
We use the Fourier extraction implementation provided by \texttt{shapiq} \citep{Muschalik.2024a} and our interventional extraction implementation based on \cref{alg:tree_capi}.
We report the speedup as the ratio between the runtime of Fourier extraction and the runtime of interventional extraction in \cref{fig:appendix:interventional_speedup}.
Since a tree can only contain interactions of order $k$ if its depth is at least $k$, we report order-$k$ speedups only for trees with depth at least $k$.

The results show that interventional extraction consistently outperforms Fourier extraction across datasets, with speedups ranging from $2\times$ to more than $10^4\times$, depending on the dataset and maximum tree depth.
The only exception occurs for shallow trees of depth $3$ when extracting third-order interactions, where Fourier extraction can be slightly faster on large datasets such as \textsc{Crime}.

We additionally investigate the effect of varying the tree depth and the number of estimators of the XGBoost proxy on the approximation quality of ProxySPEX and ProxySHAP in \cref{fig:tree_depth_on_approximation_quality} and \ref{fig:tree_capacity_on_approximation_quality}.

\paragraph{Tree Depth.}
To this end, we set the maximum depth of the individual XGBoost trees \citep{Chen.2016} to $2$, $4$, $6$, and $8$.
We observe that shallow trees perform particularly well for small coalition budgets, while medium-depth trees are effective across most of the budget range considered.
For large coalition budgets, deeper trees become increasingly beneficial, suggesting that the default depth of $6$ may no longer capture all relevant interaction structure once sufficient training data are available.
This trend is evident in both ProxySHAP and ProxySPEX, underscoring the importance of adjusting the tree depth to the available coalition budget.

\paragraph{Number of Estimators.}
For the number of estimators, we choose the values $10$, $50$, $100$, and $200$.
We observe that the effect of more trees only becomes visible when given a sufficient budget.
Notably, the more trees we have, the more interaction structure it can capture, and thus the better it approximates the underlying value function.
This becomes most prominent as we increase the budget, with more of the underlying game's structure revealed.
Interestingly, ProxySPEX does not gain much from an increased number of trees, whereas ProxySHAP clearly does as the budget increases.

\begin{figure}
    \centering
    \includegraphics[width=\textwidth]{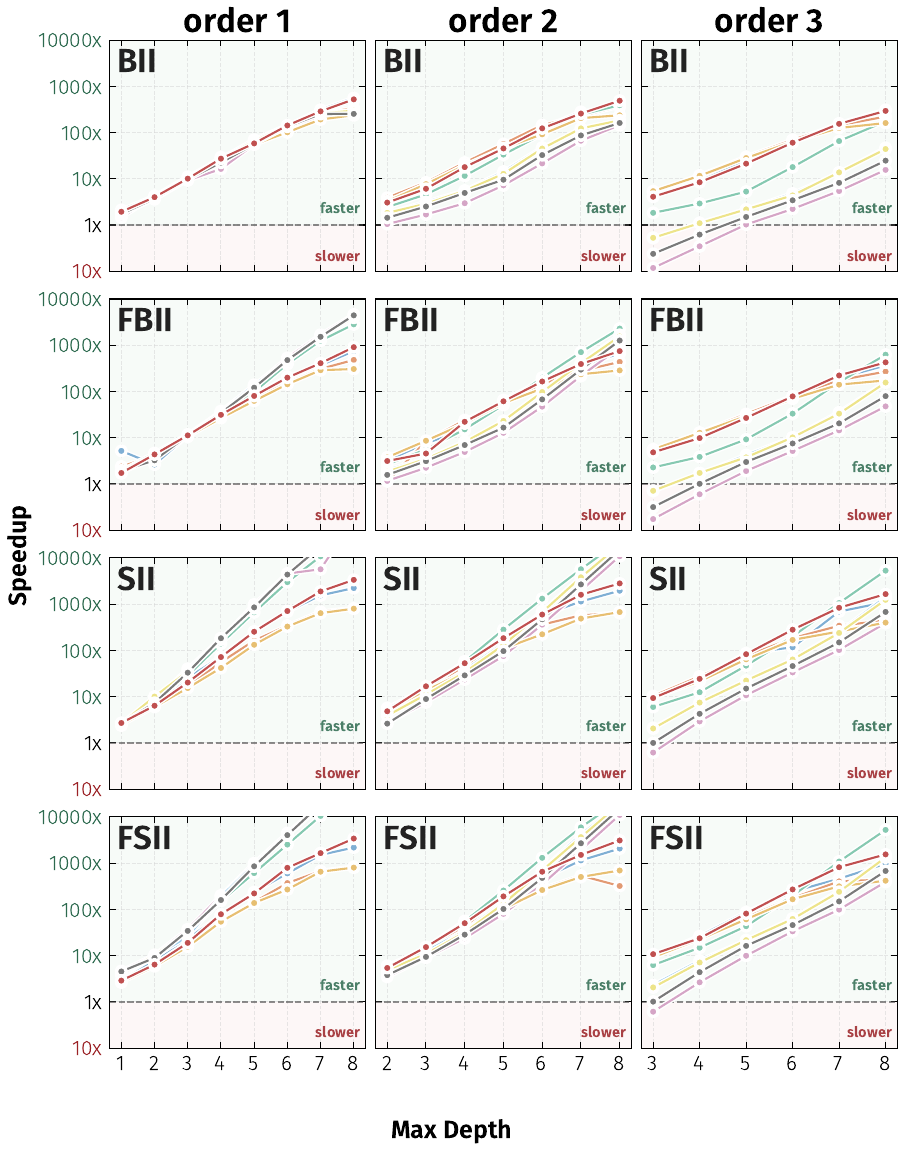}
    \includegraphics[width=\textwidth]{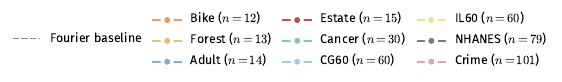}
    \caption{Speedup of interventional extraction compared to Fourier extraction for extracting all interactions of order 1, 2, and 3 across different datasets.}
    \label{fig:appendix:interventional_speedup}
\end{figure}

\subsection{Detailed Comparison of ProxySHAP and ProxySPEX}
\label{appendix:detailed_comparison_proxyshap_proxyspex}

ProxySPEX \citep{Butler.2025} is a model-agnostic approximation method for computing any-order cardinal-probabilistic interaction indices, including in settings with large feature counts.
It consists of three main steps with an optional fourth step:
\begin{enumerate}
    \item \textbf{Sampling and evaluation.}
    Coalitions $\mathcal{T} \subseteq 2^N$ are sampled and evaluated, yielding the dataset
    \[
        \mathcal{D} = \{(T, \nu(T))\}_{T \in \mathcal{T}}.
    \]

    \item \textbf{Proxy fitting.}
    A gradient-boosted tree model is fitted on $\mathcal{D}$ by minimizing the mean squared error.
    By default, \texttt{shapiq} \citep{Muschalik.2024a} uses a LightGBM with an additional hyperparameter sweep.
    For all experiments, we restrict ourselves to the default XGBoost parameters to ensure the same tree across all proxy-based methods.

    \item \textbf{Fourier extraction and truncation.} \label{truncation_proxyspex}
    Fourier coefficients are extracted from the fitted tree proxy.
    ProxySPEX then keeps a minimal subset $C^\star \subseteq \mathcal{F}$ of coefficients that explains at least $95\%$ of the total squared Fourier mass,
    \[
        C^\star
        =
        \arg\min_{C \subseteq \mathcal{F}} |C|
        \quad
        \text{s.t.}
        \quad
        \frac{\sum_{F \in C} F^2}{\sum_{F \in \mathcal{F}} F^2}
        \geq 0.95,
    \]
    where $\mathcal{F}$ denotes the set of Fourier coefficients extracted from the tree.

    \item \textbf{Adjustment (optional).}
    Given the truncated coefficient set $C^\star$, ProxySPEX applies a refinement step to improve the extracted Fourier coefficients.
    It constructs a design matrix $X \in \{-1,+1\}^{|\mathcal{T}| \times |C^\star|}$ with entries
    \[
        X_{i,j} = (-1)^{|T_i \cap C_j|},
    \]
    and solves the regularized regression problem
    \[
        F^\star
        =
        \arg\min_{F \in \mathbb{R}^{|C^\star|}}
        \lVert \nu - XF \rVert_2^2
        +
        \lambda \lVert F \rVert_2^2 .
    \]
\end{enumerate}

The truncation step is essential for making the refinement step computationally feasible, since the number of Fourier coefficients of a tree can grow as $\mathcal{O}(4^d)$ with the tree depth $d$.
Based on the refined Fourier coefficients $F^\star$, ProxySPEX can compute any interaction index having a Möbius transform as in \cref{eq_möbius_conversion}, since the Fourier coefficients can be translated into Möbius coefficients \footnote{As the Fourier coefficients also form a basis of the games, it is sometimes computationally better to directly use those instead of transforming to Möbius coefficients.}.
For the exact transformations, we refer to Appendix~A of \citet{Butler.2025}.

ProxySHAP follows a closely related proxy-based strategy, but differs in two important aspects.
First, instead of extracting Fourier coefficients and converting them into the desired index, ProxySHAP directly extracts the target cardinal-probabilistic interaction index from the tree proxy.
Second, ProxySHAP applies a case-by-case MSR adjustment directly to the residual game.
This yields a consistent estimator whenever the residual correction is sufficiently well covered by the sampled coalitions, and it leads to improved performance in regimes where sufficient budget is available (see \cref{fig:mse_quality,fig:appendix:winnermap_sii}).

\Cref{fig:proxyspex_converges_proxyshap} illustrates how the approximation quality of ProxySPEX changes as the cutoff threshold is increased and, consequently, more Fourier energy is retained in $C^*$. 
As the cutoff approaches $1$, ProxySPEX increasingly recovers the behavior of ProxySHAP, but at substantially higher computational cost, since up to $\mathcal{O}(4^d)$ Fourier coefficients may need to be converted to the desired interaction index. 
In contrast, ProxySHAP consistently exhibits near-diagonal behavior across budgets and datasets, and is only outperformed by KernelSHAP-IQ once the latter receives a sufficiently large evaluation budget. 
For \textsc{Coil2000}, even a cutoff of $0.999$ does not lead to substantial improvements, since more than $48\%$ of the Fourier coefficients are still discarded. We refer to \cref{fig:appendix:winnermap_sii} for a complementary overview of the best-performing methods on these datasets.

\begin{figure}
    \centering
    \includegraphics[width=\linewidth]{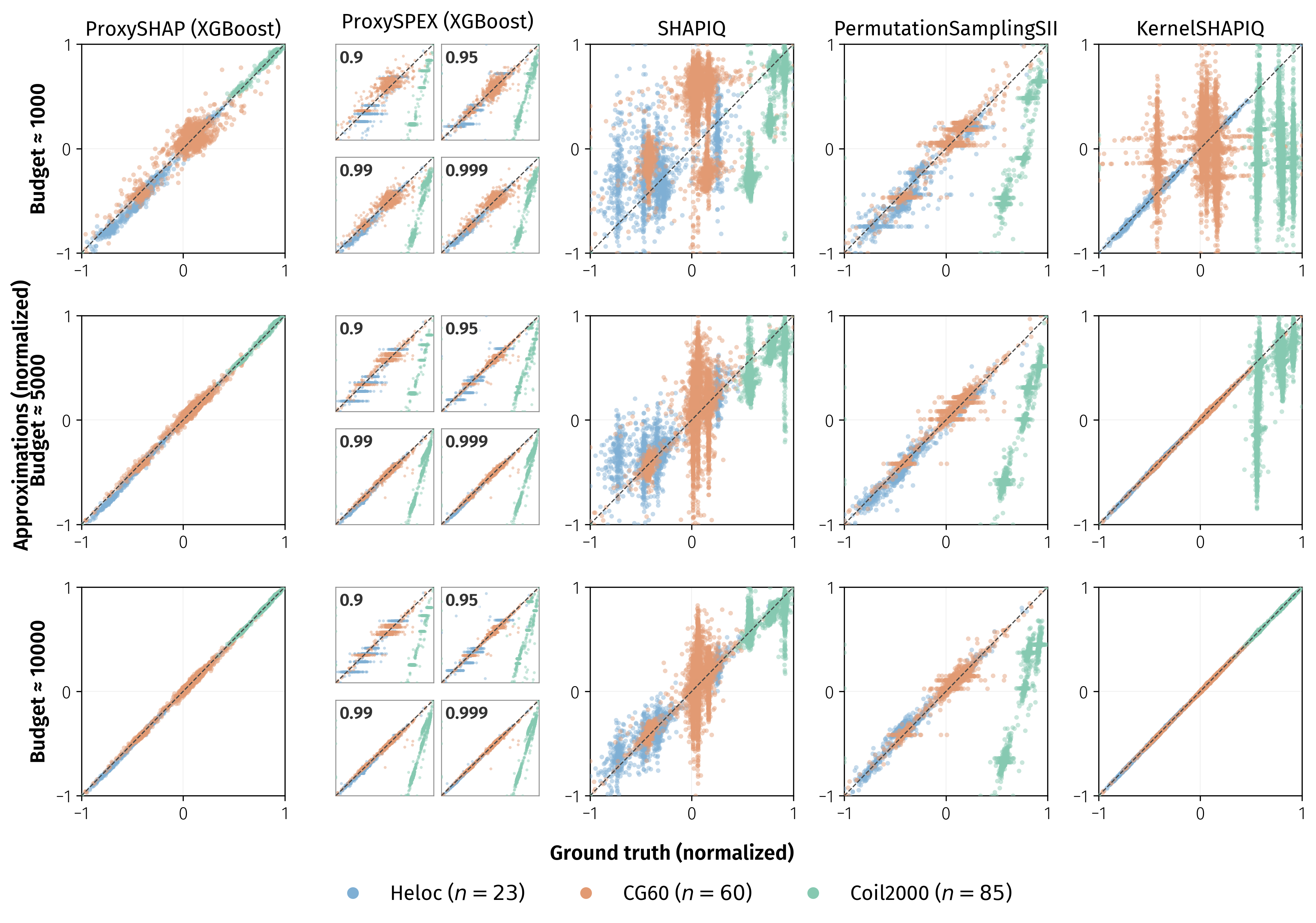}
    \caption{
Predicted versus ground-truth normalized interaction values for different approximation methods and sampling budgets.
Each point represents an interaction value from a dataset and a benchmark run; points closer to the diagonal indicate better agreement with the exact interaction values.
Columns compare ProxySHAP, ProxySPEX, SHAPIQ, PermutationSamplingSII, and KernelSHAPIQ, while rows correspond to increasing evaluation budgets.
In ProxySPEX, the central column shows four cutoff thresholds, with larger cutoffs retaining more Fourier energy in the surrogate approximation.
As the ProxySPEX cutoff increases, the scatter increasingly aligns with ProxySHAP, indicating that ProxySPEX becomes closer to ProxySHAP as more Fourier energy is retained.
    }
    \label{fig:proxyspex_converges_proxyshap}
\end{figure}

\subsection{When to use Adjustment?}
\label{sec:appendix:adjustment}
\begin{figure}[t]
    \centering
    \includegraphics[width=\textwidth]{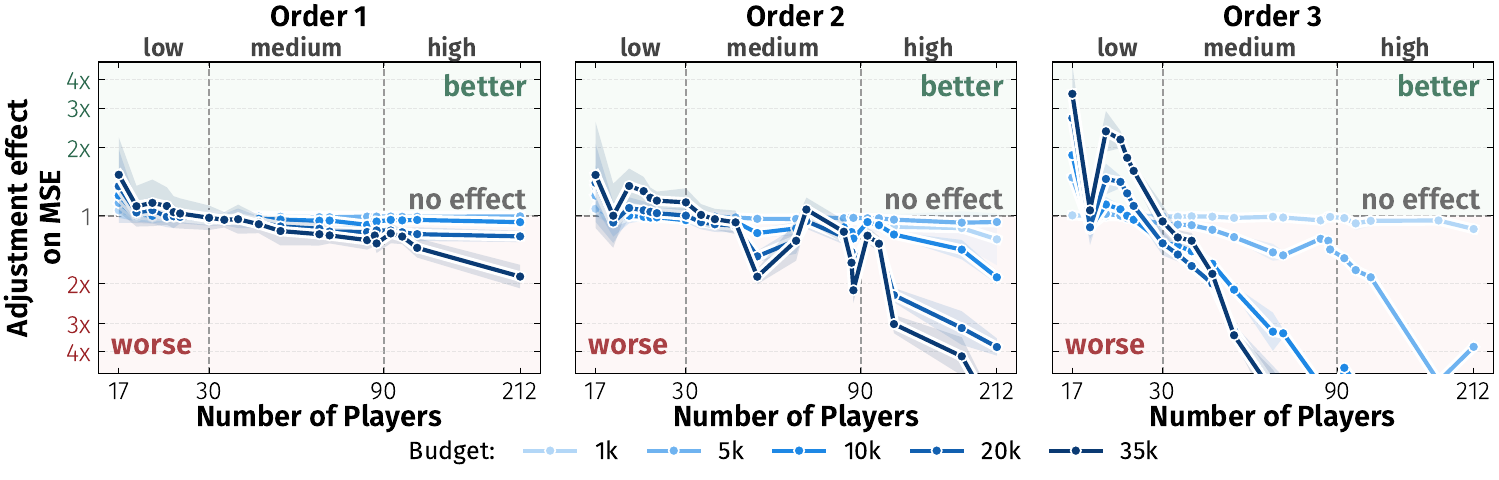}
    \caption{MSR adjustment effect for \emph{Banzhaf} values and interactions of orders one to three, in the format of \cref{fig:main_adjustment_benefit}. Values above one indicate that adjustment improves approximation quality. Even for Banzhaf values (order one), the adjustment helps only mildly and only for small player counts; for higher orders, the regime in which adjustment helps shrinks faster with the number of players than for Shapley interactions, consistent with the $4^s$-dominated regime of Corollary~\ref{cor:gamma_leverage_bii}.}
    \label{fig:appendix:adjustment_benefit_bii}
\end{figure}
The adjustment step is designed to make the resulting interaction estimates consistent. To assess its practical relevance, we complement this theoretical motivation with an empirical analysis across all $32$ considered TabArena datasets \citep{Erickson.2025}, budgets, and interaction orders. Since each experiment is repeated over $30$ explained instances per dataset, we measure the instance-wise relative change in approximation quality as
\[
    r_i = \frac{\text{MSE without adjustment}}{\text{MSE with adjustment}},
\]
and report the geometric mean
\[
    \bar r = \exp\!\left(\frac{1}{N}\sum_{i=1}^N \log r_i\right),
\]
pooled across all instances with the same number of players. The shaded bands show the geometric standard deviation factor,
\[
    s = \exp\!\left(\operatorname{std}(\log r_i)\right),
\]
drawn as the multiplicative interval $[\bar r / s,\, \bar r \cdot s]$, corresponding to one standard deviation in log-space. 
Hence, values above one indicate that adjustment improves approximation quality, whereas values below one indicate that it deteriorates.

For interaction indices, however, the effect of adjustment is more nuanced.
While adjustment can reduce proxy bias, it can also introduce substantial variance, especially for higher-order interactions.
This is consistent with our theoretical analysis in \cref{sec:proxyshap}, where we show that the dominant variance term of the adjustment estimator scales as \(n^{k-1}\), with \(n\) denoting the number of features and \(k\) the maximal interaction order.
Consequently, adjustment can deteriorate approximation quality when the feature dimension or interaction order is large.
Empirically, \cref{fig:main_adjustment_benefit} shows that, for second-order interactions, adjustment remains beneficial up to datasets with 90 features, provided that the budget is sufficiently large.
In contrast, for third-order interactions, the adjustment already degrades approximation quality in games with more than 30 features, even at large budgets.

We therefore recommend applying MSR adjustment primarily when the available coalition budget is substantially larger than the dominant variance term, i.e., when \(|\mathcal T| \gg n^{k-1}\).
In practice, this supports using adjustment for second-order interactions at sufficiently large budgets, while omitting it for third-order interactions and beyond unless the budget is exceptionally large.

\paragraph{Banzhaf interactions.}
For the Banzhaf interaction index, Corollary~\ref{cor:gamma_leverage_bii} gives a variance factor of $\Theta(4^s \sqrt n)$ with uniform over sizes, which grows more slowly in $n$ than the $\Theta(n^{s-1})$ factor for Shapley interactions with $s \ge 2$. This asymptotic advantage, however, only takes effect once $n > 4^{2s/(2s-3)}$ when comparing the leading terms. Empirically, at a budget of $35{,}000$ coalitions, adjustment for second-order Banzhaf interactions stops being beneficial from roughly $n \approx 35$, whereas for Shapley interactions it remains beneficial up to $n=90$ (see \cref{fig:appendix:adjustment_benefit_bii} for the Banzhaf counterpart of \cref{fig:main_adjustment_benefit}). The $n < 30$ recommendation therefore applies to Banzhaf interactions as well.

\subsection{Shared vs. Disjoint Subsets in ProxySHAP}
\label{appendix:shared_disjoint_subsets}
\begin{figure}[t]
  \centering
  \includegraphics[width=\textwidth]{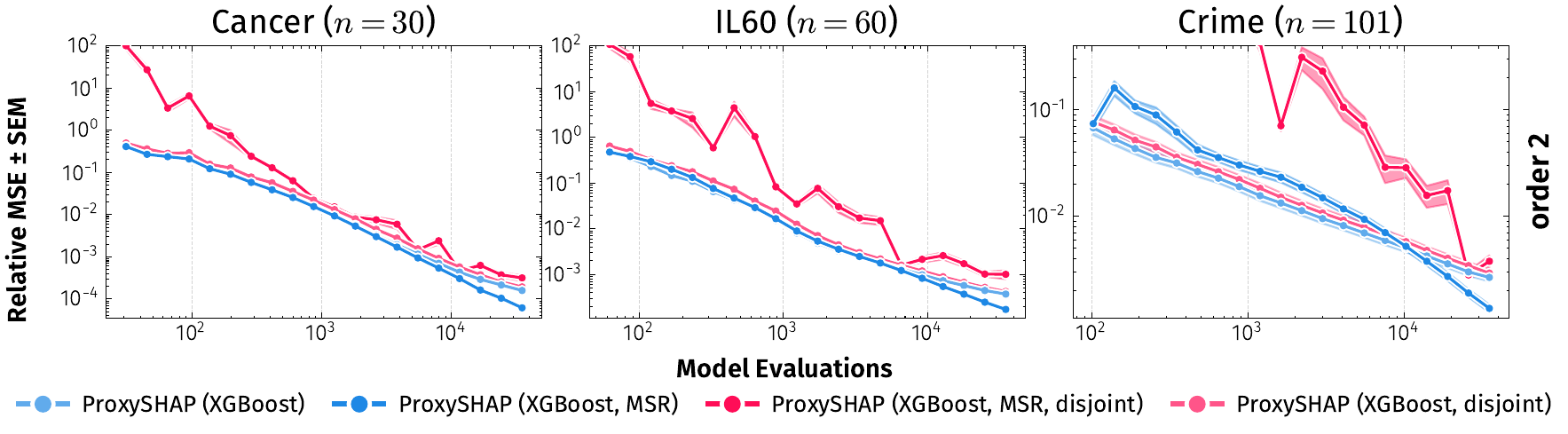}
  \caption{ProxySHAP with disjoint coalition sets for proxy fitting and residual adjustment.}
  \label{fig:proxyshap_disjoint_sets}
\end{figure}

We compare two strategies for using the sampled coalitions: reusing the same set for both proxy fitting and residual adjustment, or splitting it into disjoint sets.
For the disjoint variant, we first sample coalitions $\mathcal{T}$ and then split them into $\mathcal{T}^{\text{Proxy}}$ and $\mathcal{T}^{\text{Adjustment}}$, which are used for fitting the proxy and estimating the residual correction, respectively.
As shown in \cref{fig:proxyshap_disjoint_sets}, enforcing disjoint sets degrades approximation quality for cardinal-probabilistic interactions.
This provides empirical evidence that the observation of \citet{Witter.2025} also holds for interaction indices: reusing the same sampled coalitions for proxy fitting and adjustment is preferable in practice.
Based on this finding, all experiments in the main paper use the same sampled coalitions for both steps, as shown in \cref{alg:proxyshap}.

\subsection{Approximation Quality}
\label{appendix:additional_experiments}

We provide additional approximation-quality results across datasets and budgets for second- and third-order Shapley interaction indices and Banzhaf interaction indices (\cref{fig:appendix:additional_curves}).
We also report winner maps for SII (\cref{fig:appendix:winnermap_sii}) and BII (\cref{fig:appendix:winnermap_bii}), where the winner is the method with the lowest average relative MSE across the 30 explained instances for each dataset and budget.
A complete collection of approximation curves is available at \githubrepo, alongside additional metrics such as Pearson's correlation.

\paragraph{Banzhaf Interaction Index.}
As shown in \cref{fig:appendix:winnermap_bii}, ProxySHAP with an XGBoost proxy almost always outperforms all baselines across datasets and budgets for both second- and third-order interactions.
For second-order interactions, the linear proxy can be stronger on some datasets, such as \textsc{Crime} and \textsc{Soybean}.
For third-order interactions, however, the tree-based proxy tends to perform better.
As the number of players and target interactions grows, the linear proxy becomes less effective because fitting the interaction basis requires larger budgets.
Overall, ProxySHAP with an XGBoost proxy is the strongest method for approximating Banzhaf interaction indices in our experiments.

\paragraph{Shapley Interaction Index.}
\cref{fig:appendix:winnermap_sii} shows that ProxySHAP with an XGBoost proxy consistently outperforms all baselines across budget and player regimes. For smaller datasets, ProxySHAP outperforms for mid-sized budget counts, but shifts to the lower-budget regime as the number of players increases.
For pairwise interactions on datasets with roughly 30 features and sufficiently large budgets, KernelSHAP-IQ can outperform both XGBoost and linear ProxySHAP variants, e.g., on \textsc{Soybean}, \textsc{Ionosphere}, and \textsc{Students Dropout}.
However, as the number of features increases, KernelSHAP-IQ becomes applicable in fewer regimes, and ProxySHAP with XGBoost becomes the best-performing method.
The winner maps also show that the regime in which MSR adjustment is beneficial shrinks with increasing feature count for both second- and third-order interactions.
For third-order interactions, ProxySHAP with XGBoost outperforms all baselines across datasets and budgets, whereas the linear proxy is preferable in only a few cases, such as \textsc{Cancer}.

\begin{figure}
    \centering
    \includegraphics[width=\linewidth]{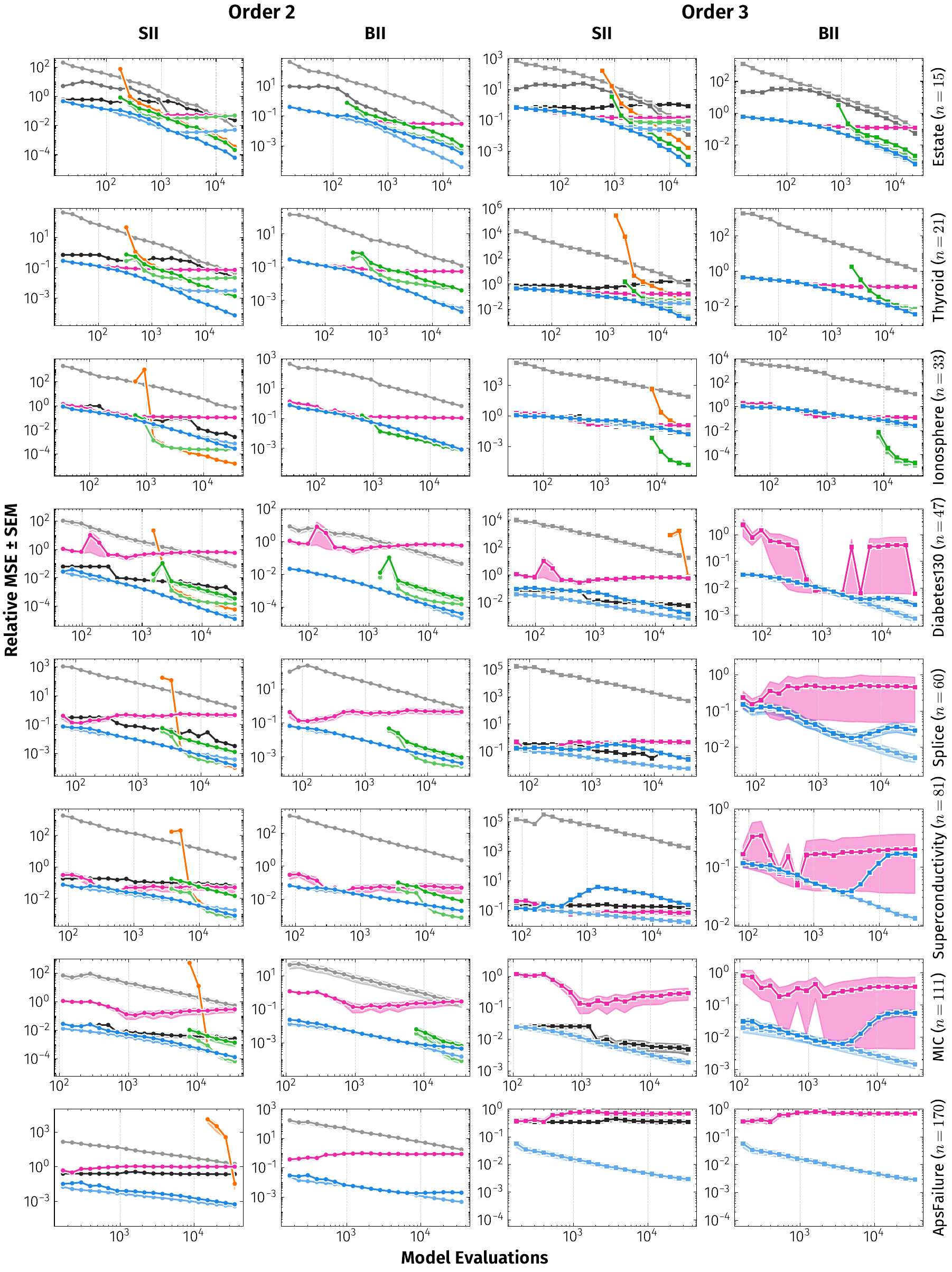}
    \caption{Selection of representative approximation curves for SII and BII at second- and third-order interactions.}
    \label{fig:appendix:additional_curves}
\end{figure}
\begin{figure}
    \centering
    \includegraphics[width=\textwidth]{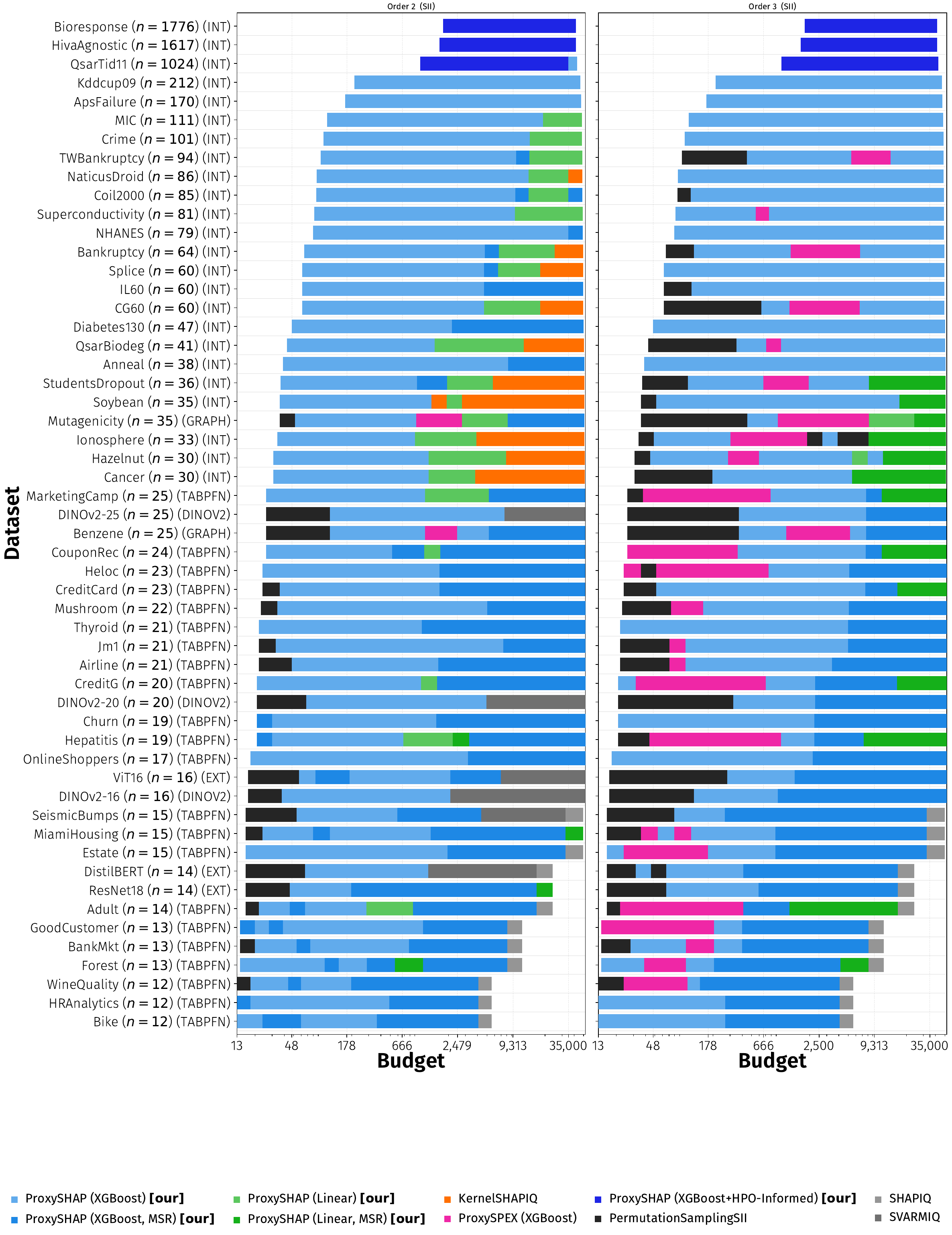}
    \caption{Winnermap comparing the best-performing method for each dataset and budget for SII orders 2 and 3. Note that the HPO-Informed variants are considered only in this overview for datasets with more than $1000$ features.}
    \label{fig:appendix:winnermap_sii}
\end{figure}

\begin{figure}
    \centering
    \includegraphics[width=\textwidth]{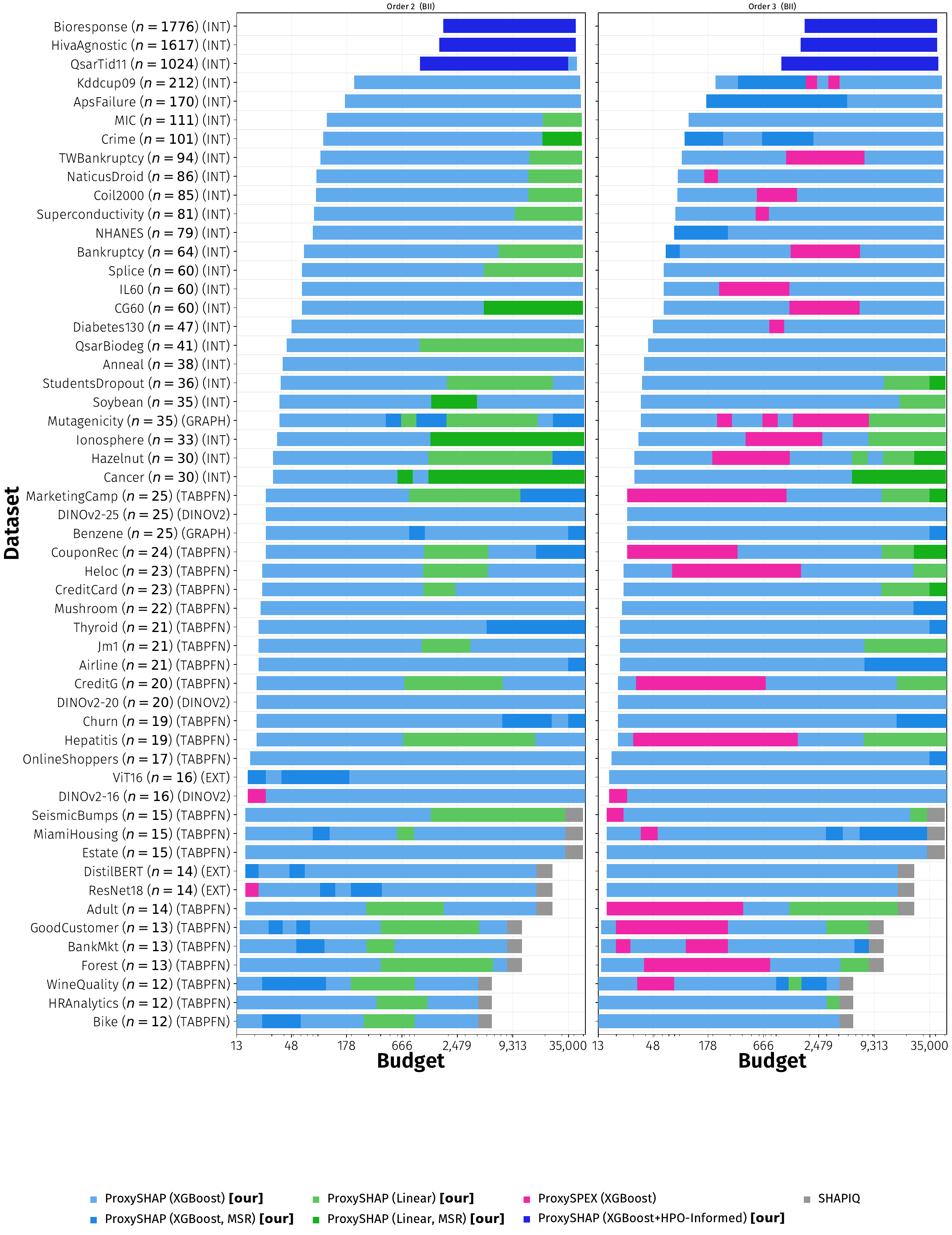}
    \caption{Winnermap comparing the best-performing method for each dataset and budget for BII orders 2 and 3. Note that the HPO-Informed variants are considered only in this overview for datasets with more than $1000$ features.}
    \label{fig:appendix:winnermap_bii}
\end{figure}

\section{Additional Related Work}\label{app:relatedwork}
Beyond proxy-based methods, research has focused on variance reduction for Monte Carlo estimates of probabilistic values~\citep{Strumbelj.2014,Wang.2023,Kolpaczki.2024a} and interactions~\citep{Tsai.2022,Kolpaczki.2024b}. 
Parallel to our work is \emph{amortization}, where an explanation model is trained to directly predict attributions~\citep{Jethani.2022,Covert.2024} or interactions~\citep{Enouen.2025}, rather than explaining the proxy post-hoc.
Further related work includes ways to improve estimation of the value function for explanations via marginal distribution compression~\citep{Baniecki.2025} and conditional distribution generation~\citep{Olsen.2022}.
Specialized polynomial-time algorithms exist for computing values and interactions for k-nearest neighbours~\citep{Wang.2024}, SVMs~\citep{Mohammadi.2025b}, and graph neural networks~\citep{Muschalik.2025}. 
Within tree-based methods, \citet{Muschalik.2024} derived exact algorithms for \emph{path-dependent} interactions; however, unlike our proposed interventional extension, their method is incompatible with proxy modeling.

\clearpage
\input{checklist}
\end{document}

%% file: checklist.tex
\section*{NeurIPS Paper Checklist}

\begin{enumerate}

\item {\bf Claims}
    \item[] Question: Do the main claims made in the abstract and introduction accurately reflect the paper's contributions and scope?
    \item[] Answer: \answerYes{}
    \item[] Justification: Please see \cref{sec:proxyshap}, \cref{sec:experiments} and for the proofs see Appendix~\ref{appendix:proofs}.

\item {\bf Limitations}
    \item[] Question: Does the paper discuss the limitations of the work performed by the authors?
    \item[] Answer: \answerYes{} 
    \item[] Justification: Please see the end of \cref{sec:conclusion}.

\item {\bf Theory assumptions and proofs}
    \item[] Question: For each theoretical result, does the paper provide the full set of assumptions and a complete (and correct) proof?
    \item[] Answer: \answerYes{} 
    \item[] Justification: Please see Section \ref{sec:adjustment} and Appendix~\ref{appendix:proofs}.

    \item {\bf Experimental result reproducibility}
    \item[] Question: Does the paper fully disclose all the information needed to reproduce the main experimental results of the paper to the extent that it affects the main claims and/or conclusions of the paper (regardless of whether the code and data are provided or not)?
    \item[] Answer: \answerYes{} 
    \item[] Justification: Please see \cref{sec:experiments} as well as Appendix~\ref{appendix:experiment_details}.

\item {\bf Open access to data and code}
    \item[] Question: Does the paper provide open access to the data and code, with sufficient instructions to faithfully reproduce the main experimental results, as described in supplemental material?
    \item[] Answer: \answerYes{} 
    \item[] Justification: Please see the code: \githubrepo

\item {\bf Experimental setting/details}
    \item[] Question: Does the paper specify all the training and test details (e.g., data splits, hyperparameters, how they were chosen, type of optimizer) necessary to understand the results?
    \item[] Answer: \answerYes{} 
    \item[] Justification: Please see \cref{sec:experiments} as well as Appendix~\ref{appendix:experiment_details}.

\item {\bf Experiment statistical significance}
    \item[] Question: Does the paper report error bars suitably and correctly defined or other appropriate information about the statistical significance of the experiments?
    \item[] Answer: \answerYes{} 
    \item[] Justification: Please see \cref{sec:experiments} as well as Appendix~\ref{appendix:experiment_details}.

\item {\bf Experiments compute resources}
    \item[] Question: For each experiment, does the paper provide sufficient information on the computer resources (type of compute workers, memory, time of execution) needed to reproduce the experiments?
    \item[] Answer: \answerYes{} 
    \item[] Justification: Please see Appendix~\ref{appendix:experiment_details}.
    
\item {\bf Code of ethics}
    \item[] Question: Does the research conducted in the paper conform, in every respect, with the NeurIPS Code of Ethics \url{https://neurips.cc/public/EthicsGuidelines}?
    \item[] Answer: \answerYes{} 
    \item[] Justification: We have reviewed the NeurIPS Code of Ethics and confirm that our research conforms to it in every respect.

\item {\bf Broader impacts}
    \item[] Question: Does the paper discuss both potential positive societal impacts and negative societal impacts of the work performed?
    \item[] Answer: \answerYes{} 
    \item[] Justification: Please consult the Broader Impact statement in \cref{sec:conclusion}.
    
\item {\bf Safeguards}
    \item[] Question: Does the paper describe safeguards that have been put in place for responsible release of data or models that have a high risk for misuse (e.g., pre-trained language models, image generators, or scraped datasets)?
    \item[] Answer: \answerNA{} 
    \item[] Justification: We do not release any data or models with a high risk for misuse, so we have not put in place any safeguards.

\item {\bf Licenses for existing assets}
    \item[] Question: Are the creators or original owners of assets (e.g., code, data, models), used in the paper, properly credited and are the license and terms of use explicitly mentioned and properly respected?
    \item[] Answer: \answerYes{} 
    \item[] Justification: Please see Appendix~\ref{appendix:experiment_details}.

\item {\bf New assets}
    \item[] Question: Are new assets introduced in the paper well documented and is the documentation provided alongside the assets?
    \item[] Answer: \answerYes{} 
    \item[] Justification: Please see \githubrepo

\item {\bf Crowdsourcing and research with human subjects}
    \item[] Question: For crowdsourcing experiments and research with human subjects, does the paper include the full text of instructions given to participants and screenshots, if applicable, as well as details about compensation (if any)? 
    \item[] Answer: \answerNA{} 
    \item[] Justification: Our paper does not involve crowdsourcing nor research with human subjects.

\item {\bf Institutional review board (IRB) approvals or equivalent for research with human subjects}
    \item[] Question: Does the paper describe potential risks incurred by study participants, whether such risks were disclosed to the subjects, and whether Institutional Review Board (IRB) approvals (or an equivalent approval/review based on the requirements of your country or institution) were obtained?
    \item[] Answer: \answerNA{} 
    \item[] Justification: Our paper does not involve crowdsourcing nor research with human subjects.

\item {\bf Declaration of LLM usage}
    \item[] Question: Does the paper describe the usage of LLMs if it is an important, original, or non-standard component of the core methods in this research? Note that if the LLM is used only for writing, editing, or formatting purposes and does \emph{not} impact the core methodology, scientific rigor, or originality of the research, declaration is not required.
    \item[] Answer: \answerNA{} 
    \item[] Justification: The core method development does not involve LLMs as any important, original, or non-standard components.

\end{enumerate}